\documentclass{article}
\usepackage[T1]{fontenc}
\usepackage{iclr2027_conference,times}
\usepackage{float}
\usepackage{amsmath,amssymb,amsthm,booktabs,graphicx,microtype,tabularx,longtable}
\usepackage{hyperref,url,xcolor,colortbl}
\usepackage{wrapfig,needspace}
\definecolor{scorebetter}{RGB}{222,241,226}
\definecolor{scoreworse}{RGB}{250,224,224}
\usepackage{placeins}
\title{What Makes High-Magnification Knowledge Transferable?
A Study of Cross-Resolution Distillation in Whole-Slide Imaging}
\author{Zhiyuan Yang \& Jiahao Cheng \\
Department of Computer Science and Software Engineering (CSSE)\\
Concordia University\\
Montreal, Canada\\
\texttt{\{zhiyuan.yang,jiahao.cheng\}@mail.concordia.ca}\\
\AND
Mahdi S. Hosseini\thanks{Corresponding author.}\\
Department of Computer Science and Software Engineering (CSSE)\\
Concordia University\\
Montreal, Canada\\
Mila -- Quebec AI Institute\\
Montreal, Canada\\
\texttt{mahdi.hosseini@concordia.ca}
}
\iclrfinalcopy 

\newcommand{\five}{$5\times$}
\newcommand{\twenty}{$20\times$}

\newcommand{\R}{\mathbb{R}}
\begin{document}
\maketitle
\lhead{Under review as a conference paper at ICLR 2027}
\addtocontents{toc}{\protect\setcounter{tocdepth}{-1}}
\begin{abstract}
Cross-resolution knowledge distillation aims to improve low-magnification whole-slide analysis by transferring high-magnification representations, yet the conditions for useful transfer remain unclear. We develop a decomposition-based analysis of teacher access, representation loss, and model excess, motivating three questions: whether \textbf{(a) teacher targets help the task, (b) low-magnification students can predict them,} and  \textbf{(c) slide models benefit from those predictions.} We investigate them through controlled experiments across ten pathology cohorts spanning classification, grading, and survival prediction. In the main comparison, providing teacher regional means alongside native low-magnification features improves downstream performance in all ten cohorts. Direct prediction achieves lower reconstruction error than residual prediction, yet the predicted features underrepresent variation in the teacher targets. Moreover, better reconstruction does not consistently improve downstream scores, and retaining native features changes performance even when the predicted teacher features are held fixed. Together, these findings expose a gap between reconstructing teacher representations and realizing their downstream value. They challenge the sufficiency of reconstruction error as a measure of cross-resolution transfer and provide a diagnostic framework for examining where that transfer breaks down. Future distillation designs must account for both what students can predict and how slide models use those predictions.
\end{abstract}
\section{Introduction}\label{sec:introduction}
It is believed that knowledge transfer across resolutions can improve low-magnification task models in computational pathology \citep{XMAG,LRMIL}, but the conditions that make high-magnification supervision transferable remain poorly understood. Slide models aggregate large bags of patch features for diagnosis and prognosis \citep{ilse2018attention,xu2024gigapath,chief2024,cpath,ding2025titan,vorontsov2026prism2}. For the same tissue area and patch dimensions, \twenty{} processing requires sixteen times as many patches as \five{}. Cross-resolution transfer seeks to use this detailed supervision during training while retaining low-magnification inputs at deployment, as illustrated in Figure~\ref{fig:magnification_gap_intro}a. Choosing what to transfer is part of the problem. Adapting teacher guidance to student capacity can improve distillation \citep{qian2025adapt}, while stronger teachers may not lead to comparable student gains \citep{lee2025customkd}. A regional mean averages high-magnification features within a low-magnification field, while spatial targets retain variation within that field. The low-magnification input may not determine all of this detail, reducing the prediction diversity (Figure \ref{fig:magnification_gap_intro}b). Even an accurate prediction may provide little task benefit if it replaces useful native features or is difficult for the slide model to use.

\begin{figure}[H]
\centering
\includegraphics[width=\linewidth]{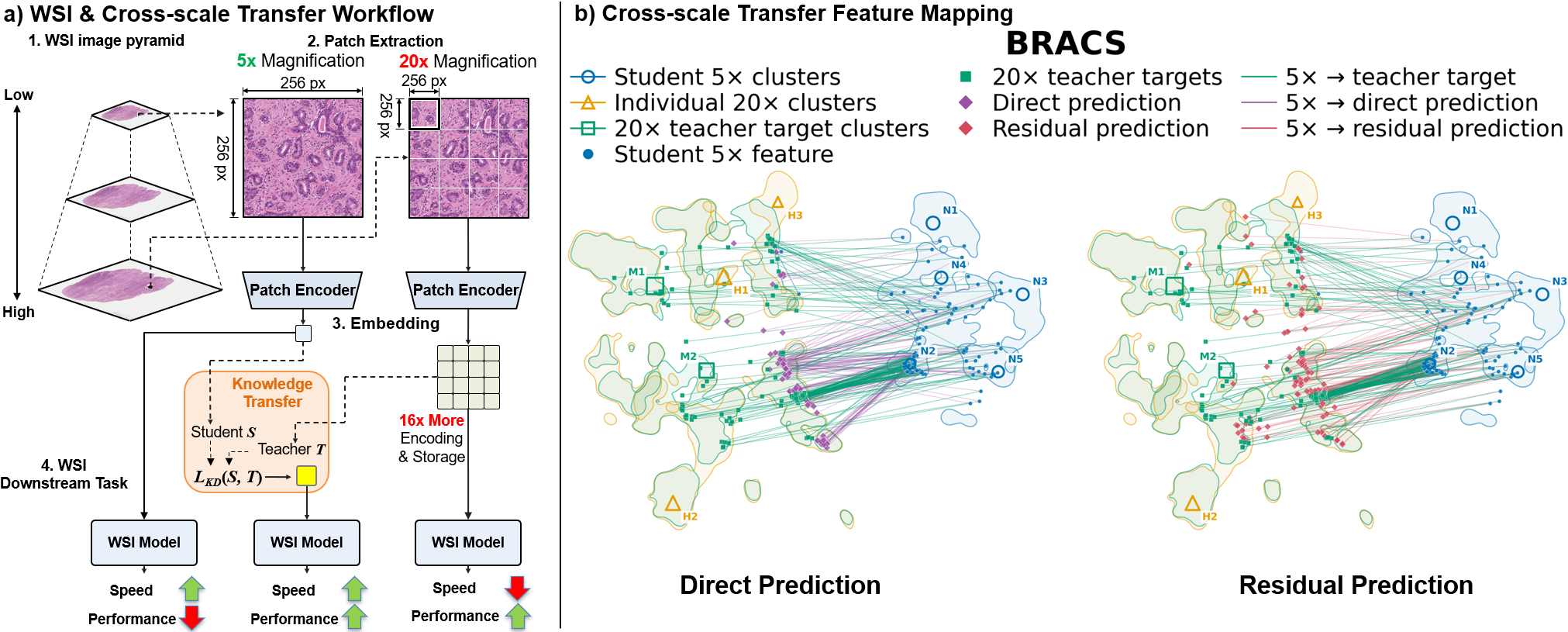}
\caption{
\textbf{Cross-resolution transfer in WSI representation learning.}
(a) General WSI workflow for improving slide-level prediction through high-magnification supervision while retaining a low-magnification token budget. (b) Direct transfers predict the teacher target from student input directly, while residual transfers predict the difference between the student input and teacher target. In either case, the predictions fail to cover the full diversity of their high-magnification teacher targets.
}
\label{fig:magnification_gap_intro}
\end{figure}

Recent systems align representations across magnifications \citep{MagAlign}, combine regional and spatial teacher guidance \citep{XMAG}, or transfer patch and slide supervision to a low-magnification student \citep{LRMIL}. These studies evaluate complete transfer systems and their components. Task gains alone do not establish whether (a) the teacher target is useful, (b) low-magnification input can predict it, and (c) the slide model benefits from that prediction. In our study, we start from mathematics to careful experimental designs to answer these questions. We study these questions for high-resolution to low-resolution transfer across five classification and grading and five survival cohorts. Formal analysis and matched experiments connect observed teacher targets, student predictions and their use by the slide model. Our contributions are as follows.

\textbf{Task risk decomposition.} We formulate cross-resolution transfer through a decomposition of task risk. The formulation separates the value of teacher information from information lost in the student representation and the task model's ability to use what remains.

\textbf{Controlled experiments across ten cohorts.} We study these distinctions across ten classification, grading and survival cohorts with five realizations. Our controls let us examine teacher content, prediction and delivery separately for clean answers to each question.

\textbf{Empirical findings on useful transfer.} We find that observed regional teacher means can complement native features, yet better reconstruction does not consistently improve task scores. Predictions often fail to cover full teacher target diversity. Further, retaining native features changes performance. Therefore, the way predictions are used becomes an essential part of the transfer evaluation.

\section{Related Work}\label{sec:related}
\textbf{Knowledge transfer across representations.} Knowledge distillation transfers teacher predictions, features or attention \citep{hinton2015distilling, romero2015fitnets, zagoruyko2017attention}. Recent methods adapt guidance to student capacity \citep{qian2025adapt} or preserve useful student features during teacher matching \citep{zhang2025rsd}. Residual error distillation trains an assistant to predict differences between teacher and student features \citep{ResErrKD}. Our direct and residual predictors share the same frozen native input and regional target. Cross-resolution transfer additionally changes the information available at deployment, connecting it to learning with privileged information \citep{vapnik2009new,lopezpaz2016unifying}.

\textbf{Cross-resolution learning.} Earlier histology studies transfer high-resolution feature, output and class-specific attention guidance to lower-resolution classifiers \citep{RDPath,InterRKD}. Recent approaches align magnification-specific representations \citep{MagAlign}, distill regional means and local features \citep{XMAG}, or combine aligned patch supervision with slide predictions and attention \citep{LRMIL}. Their component ablations assess particular systems. Our controls evaluate observed teacher targets separately from student predictions, prediction beyond a training-derived constant, and native retention with the prediction unchanged. 

\textbf{Reconstruction and useful information.} Closer agreement with teacher predictions need not improve generalization \citep{stanton2021does,tiapkin2025teacherhacking}. Usable information and conditional probing distinguish what a representation contains from what a restricted model can exploit, including beyond a baseline \citep{xu2020usable, hewitt2021conditional}. Work on spatial token compression within tiles evaluates both reconstruction and slide-task performance \citep{li2025pathvq}. These distinctions motivate our separate tests of teacher prediction and task benefit when native features are retained.

\newtheorem{crdtheorem}{Theorem}
\newtheorem{crdproposition}{Proposition}
\newtheorem{crdlemma}{Lemma}

\section{Problem Formulation and Analytical Rationale}\label{sec:method}
Cross-resolution knowledge transfer involves choosing a teacher target, learning from low-magnification input and delivering the resulting representation to a task-specific model. We distinguish these choices through a decomposition of population risk. It separates the benefit of teacher access from information lost in the representation and the task model's ability to use what remains.

\textbf{Information access and task risk.} We denote task outcome as $Y$, complete low magnification input as $B$, including raw images or feature embeddings. For each region $i$, a summary function $\tau$ maps the aligned high-magnification features $H_i$ to a target $t_i$, and $T=(t_i)_i$ collects the teacher targets. A student $f_\theta$ with parameters $\theta$ produces $Q$ from $B$, either a teacher prediction or an adapted encoder output. To describe how outputs reach the task model, we use a delivery function $D$ that combines or replaces native features to form representation $Z$. The task model $g$ maps $Z$ to its prediction $\widehat Y$,
\begin{equation}
t_i=\tau(H_i),\qquad Q=f_\theta(B),\qquad Z=D(B,Q)\text{ or }Z=D(Q),\qquad \widehat Y=g(Z).
\label{eq:teacher_targets}
\end{equation}
The analysis concerns an independent evaluation population after training is complete. Training data, transformations and parameters remain fixed, so $Z$ is a deterministic function of $B$. Let $U$ denote the input supplied to a task model. It can be the native input $B$, the student representation $Z$ or the joint input $(B,T)$. A nonnegative loss $\ell$ compares a task prediction with outcome $Y$, and $\mathbb E$ averages over the evaluation population. We write the population risk of a trained model $g$ using $U$ as $\mathcal R(g;U)=\mathbb E[\ell(g(U),Y)]$. To distinguish the information in an input from a particular model's ability to use it, let $a$ range over all measurable prediction rules using $U$ with the allowed task outputs. The infimum risk $\rho(U)=\inf_a\mathbb E[\ell(a(U),Y)]$ is the lower limit permitted by $U$ without restricting network architecture. We write $\rho(B,T)$ for the joint input $(B,T)$. For a trained model $g$ using $Z$, the excess risk relative to $\rho(Z)$ is $\varepsilon_g=\mathcal R(g;Z)-\rho(Z)$. This gap includes the effects of architecture, finite training data, optimization and model selection.

\begin{crdtheorem}
\label{thm:transfer-risk}
Suppose the defined risks are finite and all functions are measurable. Then
\begin{equation}
\mathcal R(g;Z)-\rho(B,T)
=\underbrace{\rho(B)-\rho(B,T)}_{\text{teacher access}}
+\underbrace{\rho(Z)-\rho(B)}_{\text{representation loss}}
+\underbrace{\varepsilon_g}_{\text{model excess}}.
\label{eq:transfer-risk}
\end{equation}
Each term is nonnegative. If $B$ can be recovered from $Z$, representation loss is zero. Let $Z_P$ and $Z_I$ be two deterministic augmentations retaining the same complete $B$, with trained task models $g_P$ and $g_I$ and respective excess risks $\varepsilon_{g_P}$ and $\varepsilon_{g_I}$. Then
\begin{equation}
\mathcal R(g_P;Z_P)-\mathcal R(g_I;Z_I)=\varepsilon_{g_P}-\varepsilon_{g_I}.
\label{eq:retained-input-risk}
\end{equation}
\end{crdtheorem}

The teacher access measures the value of observing $T$ in addition to $B$, while representation loss measures the task information discarded when $B$ becomes $Z$. Model excess is the remaining gap between the trained task model and the best rule using $Z$. Through these terms, we distinguish an informative teacher from a useful student. The distinction matters because an observed teacher target may add information about the evaluation case, whereas a deterministic student prediction can only transform information already present in $B$. Such a transformation can still make learning easier for a finite model \citep{xu2020usable}. If both representations retain the complete native input, their representation loss is zero. Using Equation~\ref{eq:retained-input-risk}, we can then attribute their difference in population loss to model excess. Removing the native block does not by itself imply that relevant information has been lost, since the student output may preserve it through a replacement or residual sum. The decomposition applies to classification log-loss and discrete survival NLL with common observation weights. Reconstruction error and aggregate task metrics describe different quantities and are evaluated separately. Likewise, a model receiving only $T$ does not implement the joint reference $(B,T)$. We use these distinctions to decide which empirical comparisons can inform each term, with further details in Appendix \ref{app:transfer-risk}.

\textbf{Mean supervision for a shared regional output.} Individual teacher features contain variation that their regional mean removes. Whether this variation gives the student a distinct training signal depends on the learning objective. For our prediction setting in Equation~\ref{eq:teacher_targets}, let $q_i=f_\theta(B)_i$ be one output shared across all teacher vectors in region $i$. With this shared output, we compare squared error against the individual teacher features with squared error against their mean.

\begin{crdlemma}
\label{lem:shared-output-mean}
For region $i$, let $H_i=(h_{ij})_{j=1}^{K_i}$ contain $K_i\ge1$ fixed teacher vectors in $\mathbb R^d$; $d$ being the feature dimension. Define their mean $m_i=K_i^{-1}\sum_jh_{ij}$. Every shared prediction $q_i\in\mathbb R^d$ satisfies
\begin{equation}
\frac1{K_i}\sum_{j=1}^{K_i}\|q_i-h_{ij}\|_2^2
=\|q_i-m_i\|_2^2+\frac1{K_i}\sum_{j=1}^{K_i}\|h_{ij}-m_i\|_2^2.
\label{eq:shared_child_mean}
\end{equation}
The final term is independent of $q_i$ and contributes no predictor gradient.
\end{crdlemma}

In our study, $K_i=16$, so uniform squared error against all sixteen features gives the same gradient as supervision from their mean. Sampling one feature uniformly preserves the expected gradient, although its variability and the resulting training trajectory may change. An objective can provide a different expected training signal by using separate outputs or inputs that depend on the teacher target, which depart from the lemma's assumptions. We therefore interpret the equivalence as a statement about supervision, leaving the task information contained in the teacher collection as a separate question. Appendix~\ref{app:shared-output-mean} gives the proof and specifies its gradient conditions.

\textbf{Direct and residual prediction under squared-error.} Once a target $t_i$ in Equation~\ref{eq:teacher_targets} is chosen, the student can predict it directly or learn a correction to the native feature. Let $x_i$ denote the native feature and $q_i$ the student prediction, with $x_i, q_i, t_i\in\mathbb R^d$ and $B$ containing $x_i$. To compare these two prediction forms, we use the conditional mean $\mu_i(B)=\mathbb E[t_i\mid B]$, which averages possible teacher targets given the available input. This expectation differs from the regional mean $m_i$ in Lemma~\ref{lem:shared-output-mean}, which averages the teacher vectors observed in one region.

\begin{crdproposition}
\label{prop:conditional-mean}
Suppose $t_i$, $x_i$ and $q_i$ have finite second moments, with $x_i$ and $q_i$ measurable from $B$. Then
\begin{equation}
\mathbb E\|t_i-q_i\|_2^2=
\mathbb E\|t_i-\mu_i(B)\|_2^2+
\mathbb E\|\mu_i(B)-q_i\|_2^2.
\label{eq:conditional_mean}
\end{equation}
Over all square integrable functions of $B$, direct prediction is minimized uniquely almost surely by $q_i=\mu_i(B)$. A residual function $r_i$ defines prediction $q_i=x_i+r_i(B)$ and has optimum $r_i(B)=\mu_i(B)-x_i$, giving the same output.
\end{crdproposition}

The first term in Equation~\ref{eq:conditional_mean} is target variation that the specified input $B$ cannot determine, while the second measures how far the predictor remains from the conditional mean. Although direct and residual learning use the same information and share the same unrestricted optimum, differences in capacity, initialization, and optimization can still lead to different reconstruction errors. The regional mean averages teacher features across patch locations and remains unchanged when those locations are rearranged. Spatial variation describes how the features differ across positions within the region, capturing additional structure. A teacher summary can keep some of this structure through additional spatial components. This variation across patch locations is distinct from uncertainty about a target given $B$. We apply the decomposition separately to the regional mean and each spatial component. Appendix~\ref{app:conditional-mean-proof} gives the proof.


\section{Study Design}
\label{sec:experiments}
\label{sec:comparison_definitions}
\edef\StudySavedColumnsep{\the\columnsep}
\edef\StudySavedIntextsep{\the\intextsep}
\setlength{\columnsep}{6pt}
\setlength{\intextsep}{6pt}
\begin{wrapfigure}{R}{.65\textwidth}
\setlength{\abovecaptionskip}{5pt}
\centering
\includegraphics[width=\linewidth]{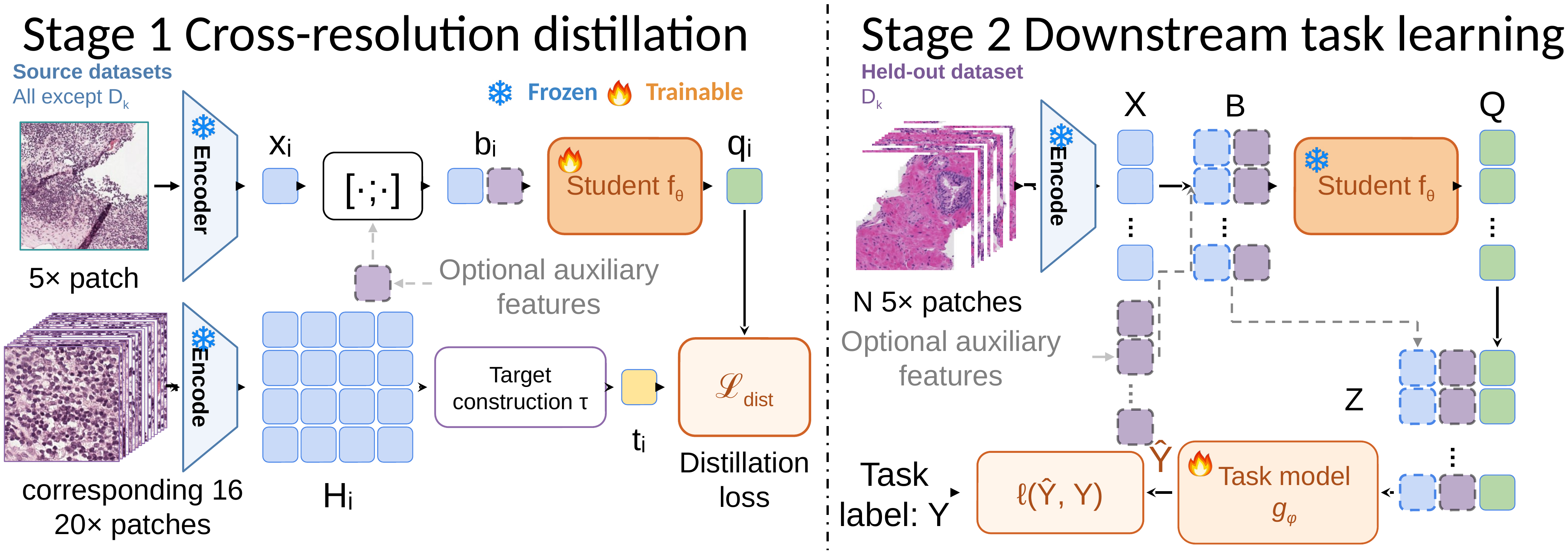}
\caption{\textbf{Cross-resolution feature prediction and evaluation pipeline.} \textbf{Stage~1} learns to predict \twenty{} teacher summaries from \five{} features on datasets without task labels. \textbf{Stage~2} evaluates the frozen predictor on an unseen dataset, using the predicted features with optionally concatenated native \five{} features for supervised slide-level learning.}
\label{fig:evaluation-protocol}
\end{wrapfigure}

Theorem~\ref{thm:transfer-risk} separates the information provided by a teacher, the representation learned by a student, and the application of that representation by a downstream task model. We design the study to identify the points where these stages no longer reinforce each other. First, we use teacher targets to define which outputs are useful for prediction. Next, we conduct reconstruction tests to evaluate whether low-magnification inputs can predict these targets more effectively than taking an average. Finally, we perform downstream comparisons to isolate the contribution of feature prediction itself and the presence of native features.

We develop experiments in two stages as shown in Figure~\ref{fig:evaluation-protocol} using CAMELYON16/17~\citep{cam16,cam17}, BRACS\citep{bracs}, PANDA\citep{panda}, and the private colorectal polyp cohort Private-CRC for classification, and TCGA-(KIRC, KIRP, LUAD, STAD, and UCEC)~\citep{NCI_TCGA} for survival analysis. We use AtlasPatch~\citep{atlaspatch} to extract non-overlapping patches of $256\times256$ pixels at \five{} and \twenty{} magnification. We choose UNI-v2~\citep{UNI} as the patch encoder for most experiments. For summary prediction, we encode the same \five{} patches with frozen DINOv3~\citep {simeoni2025dinov3}. Figure~\ref{fig:dataset-diversity}(a) summarizes the tissue and task coverage, while (b) shows how effecient the number of input token is saved using features at low-magnification. See Appendix~\ref{app:protocol} for details.

In Stage~1, we train a predictor to estimate teacher representations from \twenty{} patches using features from the corresponding \five{} regions, keeping patch encoders frozen. We use a two-layer MLP with GELU activations, which is trained on paired features excluding the dataset used for the downstream task. In Stage~2, we freeze the predictor and train the downstream model using an unseen dataset to evaluate the performance. TransMIL~\citep{TransMIL} is used as the primary downstream model. We report the AUC for CAM16, macro $F_1$ for CAM17, private-CRC and BRACS, quadratically weighted Cohen's $\kappa$ for PANDA, and the C-index for survival tasks. All experiments use the same five random seeds for dataset split, model initialization, and training. See Appendix~\ref{app:protocol} for details.

\begin{wrapfigure}{R}{.55\textwidth}
\centering
\includegraphics[width=\linewidth]
{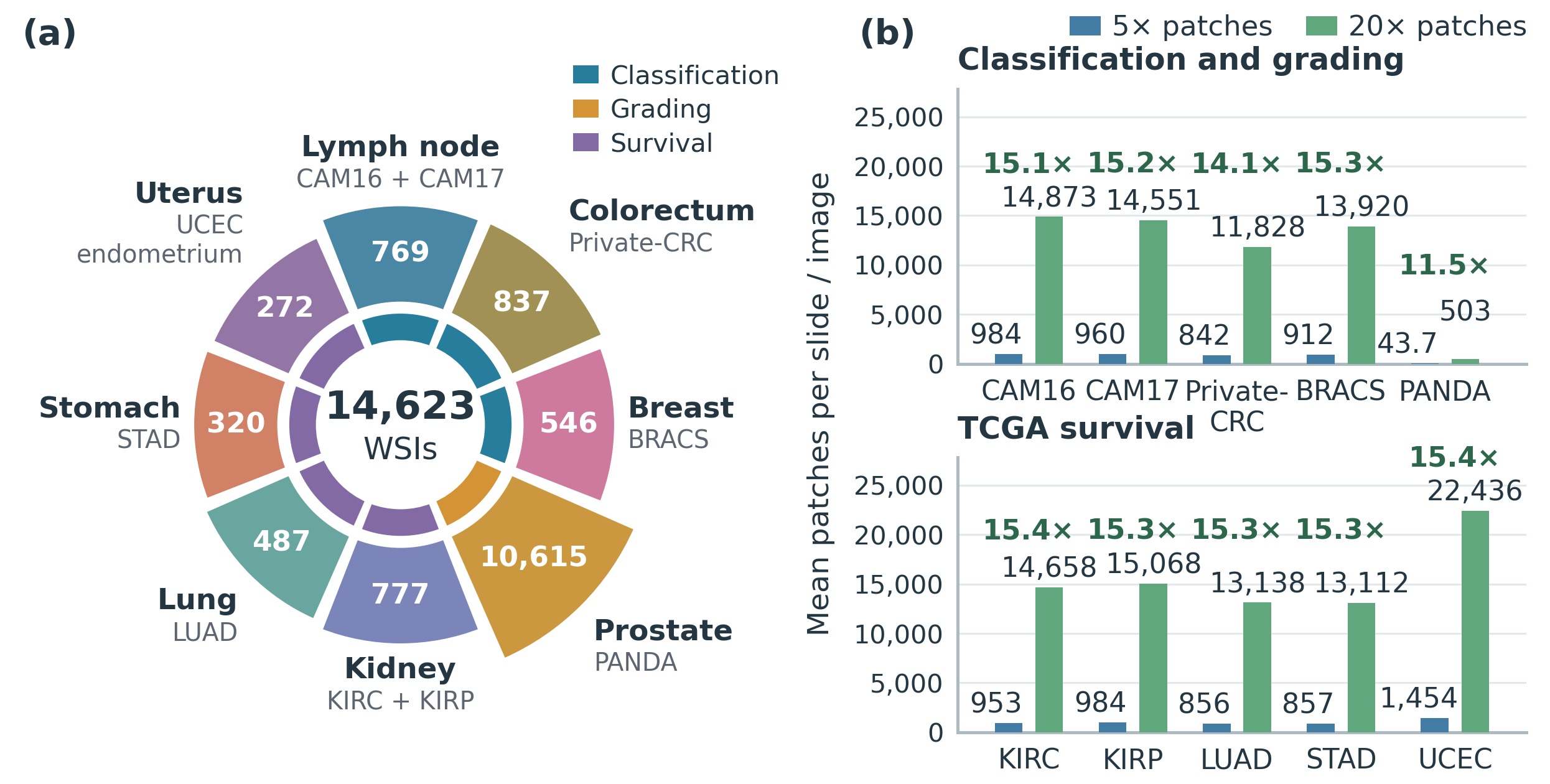}
\caption{\textbf{Data coverage and processing scale.} \textbf{(a)} We use $14,623$ slides across ten datasets covering classification, grading, and survival prediction. \textbf{(b)} We compare the average number of tokens at \five{} and \twenty{}.}
\label{fig:dataset-diversity}
\label{fig:native-patch-counts}
\end{wrapfigure}

\textbf{Evaluating teacher content separately from token count.} In Equation~\ref{eq:teacher_targets}, $\tau$ determines which teacher information is available for transfer. For each \five{} region, we average its sixteen aligned \twenty{} feature vectors, following XMAG \citep{XMAG}. We compare the resulting regional means $M$ with native \five{} features $X$, so both inputs contain one token per region. We also construct $M^{16}$ by repeating each $M$ sixteen times to compare $M^{16}$ with $M$ to test the effect of repetition without adding feature information. In addition, we compare individual \twenty{} features $H$ with $M^{16}$ to evaluate the value of the variation within the region at the same token count. To test whether a compact summary can retain useful spatial variation beyond the mean, we adapt Fourier Compressor \citep{wang2026fourier}. We apply a discrete cosine transform to the $4\times4$ grid of independently encoded \twenty{} patch features and retain four low-frequency components. The constant component is rescaled to match the regional mean, while the other three (spatial components) demonstrate how features vary across the region. (Appendix~\ref{app:spatial-components})

The teacher access term in Theorem~\ref{thm:transfer-risk} examines whether teacher summaries improve prediction when added to native \five{} features. To investigate this, we concatenate the native UNI2-h \citep{UNI} and DINOv3\citep{simeoni2025dinov3} features and then append either a teacher mean, a mean with spatial components, or zeros. We use zero padding to keep a fixed input width. This comparison evaluates whether the joint native and teacher inputs could benefit downstream tasks. (Appendix~\ref{app:joint-teacher})

\textbf{Predicting regional means and spatial variation.} For $f_\theta$ in Equation~\ref{eq:teacher_targets}, we compare two ways to predict the regional teacher mean $m_i$ from native features. Let $x_i$ be the native UNI2-h vector. Direct prediction produces $q_i=f_\theta(B)_i$, while residual prediction learns $m_i-x_i$ through a network $r_\phi$ with parameters $\phi$, giving $q_i=x_i+r_\phi(B)_i$. Both approaches use the same MLP architecture with a GELU activation \citep{hendrycks2016gelu}, taking $x_i$ and its slide mean as input. The final layer maps the hidden activations through trainable weights and a bias, producing either the direct estimate or the residual correction. Proposition~\ref{prop:conditional-mean} establishes a common unrestricted optimum under squared error. To examine how finite networks approach this optimum, we compare both prediction forms under zero and random initialization of the final layer. Zero initialization sets its weights and bias to zero, so direct prediction initially returns $q_i=0$ and residual prediction returns $q_i=x_i$. Instead, random initialization starts with an untrained direct estimate or an untrained correction to $x_i$. 

To examine whether prediction enhances both regional content and spatial variation, we train a separate MLP with GELU activation to jointly predict regional mean and spatial components. We compare this model to a constant predictor that returns the average training target for each component, regardless of input. This comparison tests whether incorporating a region's low-magnification features improves prediction beyond the training average. By reporting mean and spatial errors separately, we identify which components contribute to reconstruction progress. Our evaluation includes datasets excluded from predictor training. See Appendix~\ref{app:source-learning} for details.

\textbf{Separating native feature preservation from distillation.} We examine whether native features remain useful alongside learned teacher predictions. By keeping each prediction $q_i$, we train task models on either $[x_i;q_i]$ or $[0;q_i]$. In this scenario, only direct access to $x_i$ changes, and the predictions and their reconstruction erros remain identical. We repeat this comparison for direct and residual predictors under both final-layer initializations. Retaining the complete native bag guarantees zero representation loss in Theorem~\ref{thm:transfer-risk}. Because the residual output already includes $x_i$ in its sum, omitting a separate native block does not remove all of that information. Appendix~\ref{app:native-retention} provides details.

To examine model excess, we distinguish the benefit of distillation from the benefit of appending an untrained network's output. For summary augmentation, we retain the native bag $B=(b_i)_i$, where $b_i=[x_i;u_i]$ concatenates the native UNI2-h feature $x_i$ and DINOv3 feature $u_i$. Let $t_i$ denote the teacher summary, $q_i$ the trained predictor's output, and $q_i^I$ its output at the exact initialization used before distillation. We compare four inputs of equal width:
\begin{equation}
A_0=([b_i;0])_i,\quad A_P=([b_i;q_i])_i,\quad
A_M=([b_i;t_i])_i,\quad A_I=([b_i;q_i^I])_i.
\label{eq:summary_inputs}
\end{equation}
Each input trains a separate task model. Teacher summaries evaluate task performance with joint access to $(B,T)$, without assuming that its task model attains the optimal risk $\rho(B,T)$. To assess learned spatial outputs, we keep the trained mean constant and append trained, initial or zero spatial components. See Appendix~\ref{app:n3-joint}.

\textbf{Changing teacher supervision during encoder adaptation.} Lemma~\ref{lem:shared-output-mean} shows why averaging squared errors against individual features does not change the gradient of one shared output. To examine supervision that can distinguish these targets, we use conditional flow matching \citep{lipman2023flow}. A network conditioned on the \five{} encoder feature learns the velocity along a path from noise to a teacher target. The target determines its path input, violating the lemma's assumption of a shared output. Paths toward regional means or sampled individual features allow us to assess how teacher variation shapes the encoder. Both approaches use a common mean KD objective and flow weight .1, underusing low-rank adaptation (LoRA) with rank eight for UNI2-h \citep{hu2022lora}. Each of the three distillation partitions uses six cohorts and excludes its evaluation cohort. In Stage~2, we discard both training heads and use only adapted encoder features. With raw images as $B$ in Equation~\ref{eq:teacher_targets}, the encoder can alter both representation loss and model excess in Theorem~\ref{thm:transfer-risk}. See Appendix~\ref{app:encoder-target-law}.

\par
\setlength{\columnsep}{\StudySavedColumnsep}
\setlength{\intextsep}{\StudySavedIntextsep}

\edef\ResultsSavedColumnsep{\the\columnsep}
\edef\ResultsSavedIntextsep{\the\intextsep}
\setlength{\columnsep}{6pt}
\setlength{\intextsep}{6pt}
\newsavebox{\ResultsTableBox}

\section{Results}
\label{sec:results}

\begin{lrbox}{\ResultsTableBox}
\fontsize{7.5}{8.5}\selectfont
\setlength{\tabcolsep}{5pt}
\begin{tabular}{@{}lccc@{}}\toprule
Dataset / metric & \shortstack{Native \five{}\\$[b_i;0_{768}]$} & \shortstack{Native \five{} +\\\twenty{} mean\\$[b_i;\widetilde m_i;o_{576 }]$} & \shortstack{Native \five{} +\\\twenty{} mean, spatial\\$[b_i;\widetilde m_i;s_i]$}\\\midrule
CAM16 / AUC & $0.8455$ & \cellcolor{scorebetter}$0.8621$ & \cellcolor{scorebetter}$\mathbf{0.8932}$\\
CAM17 / $F_1$ & $0.5197$ & \cellcolor{scorebetter}$\mathbf{0.5594}$ & \cellcolor{scorebetter}$0.5198$\\
Private-CRC / $F_1$ & $0.8305$ & \cellcolor{scorebetter}$\mathbf{0.8410}$ & $0.8305$\\
PANDA / $\kappa$ & $0.8386$ & \cellcolor{scorebetter}$0.8446$ & \cellcolor{scorebetter}$\mathbf{0.8456}$\\
BRACS / $F_1$ & $0.4211$ & \cellcolor{scorebetter}$\mathbf{0.4257}$ & \cellcolor{scoreworse}$0.4094$\\
\midrule
KIRC / C-index & $0.6823$ & \cellcolor{scorebetter}$0.6892$ & \cellcolor{scorebetter}$\mathbf{0.6907}$\\
KIRP / C-index & $0.7352$ & \cellcolor{scorebetter}$\mathbf{0.7782}$ & \cellcolor{scorebetter}$0.7563$\\
LUAD / C-index & $0.5381$ & \cellcolor{scorebetter}$\mathbf{0.5795}$ & \cellcolor{scorebetter}$0.5495$\\
STAD / C-index & $0.5104$ & \cellcolor{scorebetter}$\mathbf{0.5112}$ & \cellcolor{scoreworse}$0.4947$\\
UCEC / C-index & $0.6499$ & \cellcolor{scorebetter}$\mathbf{0.6793}$ & \cellcolor{scorebetter}$0.6672$\\
\bottomrule\end{tabular}
\end{lrbox}
\begin{wraptable}[15]{R}{\wd\ResultsTableBox}
\setlength{\abovecaptionskip}{0pt}
\begingroup
\small
\caption{\textbf{Downstream performance with teacher summaries added to native \five{} features.} Inputs retain $b_i$, the projected mean $\widetilde m_i$, and spatial component $s_i$. Teacher inputs use \twenty{} features at evaluation.}
\label{tab:joint-teacher-main}
\endgroup
\usebox{\ResultsTableBox}
\end{wraptable}

Teacher information can enhance a task model, but accurate reconstruction does not guarantee that student predictions offer the same benefit. We present results in response to three questions. Section~\ref{sec:content_results} identifies which teacher targets are useful, Section~\ref{sec:transfer_results} compares how low magnification predicts them, and Section~\ref{sec:use_results} examines whether the task model benefits from those predictions.

\subsection[Is a High Magnification Target Useful?]{Is a High Magnification\\Target Useful?}\label{sec:content_results}

\textbf{Regional teacher mean supplements the student representation.} Adding the regional \twenty{} mean to native \five{} features increases the mean primary task score as shown in Table~\ref{tab:joint-teacher-main}. These results support the teacher mean as a useful addition to the low-magnification representation under this training procedure. The joint input $(B,T)$ follows the teacher access setting in Theorem~\ref{thm:transfer-risk}. The task model retains $B$ and receives an observation from the teacher, so it does not remove useful native information. A reduction in task loss can reflect both the additional teacher information and the model's ability to use the new representation. The empirical gain confirms the regional mean could be a valuable reference target for transfer.

\begin{lrbox}{\ResultsTableBox}
\fontsize{7}{9}\selectfont
\setlength{\tabcolsep}{2.3pt}
\begin{tabular}{@{}lcccc@{}}\toprule
 & \five{} input & \multicolumn{3}{c}{\twenty{} teacher construction}\\
\cmidrule(lr){3-5}
Dataset / metric & \shortstack{Native \five{}\\$X$} & \shortstack{Regional\\mean $M$} & \shortstack{Repeated\\mean $M^{16}$} & \shortstack{Individual\\\twenty{} $H$}\\
Tokens per region & 1 & 1 & 16 & 16\\\midrule
CAM16 / AUC & $0.8430$ & \cellcolor{scorebetter}$0.8993$ & \cellcolor{scorebetter}$0.8872$ & \cellcolor{scorebetter}$\mathbf{0.9047}$\\
CAM17 / $F_1$ & $0.5054$ & \cellcolor{scoreworse}$0.4856$ & \cellcolor{scoreworse}$0.4904$ & \cellcolor{scorebetter}$\mathbf{0.5418}$\\
Private-CRC / $F_1$ & $0.8273$ & \cellcolor{scoreworse}$0.8146$ & \cellcolor{scorebetter}$0.8290$ & \cellcolor{scorebetter}$\mathbf{0.8456}$\\
PANDA / $\kappa$ & $0.8535$ & \cellcolor{scorebetter}$0.8660$ & \cellcolor{scorebetter}$0.8706$ & \cellcolor{scorebetter}$\mathbf{0.8726}$\\
BRACS / $F_1$ & $0.3747$ & \cellcolor{scoreworse}$0.3731$ & \cellcolor{scorebetter}$\mathbf{0.3849}$ & \cellcolor{scorebetter}$0.3767$\\
\midrule
KIRC / C-index & $0.7127$ & \cellcolor{scoreworse}$0.7105$ & \cellcolor{scoreworse}$0.6951$ & \cellcolor{scorebetter}$\mathbf{0.7144}$\\
KIRP / C-index & $0.7793$ & \cellcolor{scorebetter}$\mathbf{0.7802}$ & \cellcolor{scoreworse}$0.7735$ & \cellcolor{scoreworse}$0.7745$\\
LUAD / C-index & $0.5494$ & \cellcolor{scorebetter}$\mathbf{0.5640}$ & \cellcolor{scorebetter}$0.5523$ & \cellcolor{scoreworse}$0.5444$\\
STAD / C-index & $0.5197$ & \cellcolor{scorebetter}$0.5402$ & \cellcolor{scorebetter}$\mathbf{0.5640}$ & \cellcolor{scorebetter}$0.5533$\\
UCEC / C-index & $0.6383$ & \cellcolor{scoreworse}$0.6158$ & \cellcolor{scorebetter}$\mathbf{0.6602}$ & \cellcolor{scorebetter}$0.6589$\\
\bottomrule\end{tabular}
\end{lrbox}
\begin{wraptable}[20]{L}{\wd\ResultsTableBox}
\setlength{\abovecaptionskip}{0pt}
\begingroup
\small
\caption{\textbf{Teacher content and token count have distinct task effects.} TransMIL receives native \five{} features, a regional \twenty{} mean, sixteen copies of that mean, or the sixteen individual features from the same tissue regions. Teacher inputs require \twenty{} images at evaluation.}
\label{tab:teacher-targets}
\endgroup
\usebox{\ResultsTableBox}
\end{wraptable}

Adding spatial components does not consistently improve on the teacher mean alone. Mean primary scores decrease in seven of the ten datasets. While additional observations cannot worsen the unrestricted optimum in Theorem~\ref{thm:transfer-risk}, the trained model's excess risk may still increase. This does not imply that spatial information is intrinsically harmful. An unrestricted predictor could ignore the extra components, whereas a trained model can generalize worse when they are supplied. Appendix~\ref{app:joint-teacher} provides further details.


\textbf{Repeating tokens changes performance without adding information.}
Table~\ref{tab:teacher-targets} shows that individual teacher features generally improve classification over native input.

However, regional means, which are the most commonly used teacher targets in prior cross-scale WSI knowledge transfer works \citep{XMAG, yang2026sigliphd, das-mil}, provide less consistent gains if not paired with native \five{} features as in Table \ref{tab:joint-teacher-main}. Repeating each mean sixteen times also changes task performance, although it adds no feature information. 

Although repeated and single means allow the same unrestricted predictions under the information argument in Theorem~\ref{thm:transfer-risk}, it provides different sequence lengths and positions to the downstream model. At the same token count, individual features $H$ do not consistently outperform repeated means $M^{16}$. Thus, neither a longer sequence nor greater within-region variation guarantees a better task representation. Frequency analysis motivates compact summaries of spatial variation, but their task effects remain mixed in Appendix~\ref{app:numbers}. Ablation on more MIL architectures is presented in Appendix~\ref{app:mil-robustness}.
\label{sec:r-content-variation}%

\begin{lrbox}{\ResultsTableBox}
\fontsize{8}{9}\selectfont
\setlength{\tabcolsep}{2.3pt}
\begin{tabular}{@{}lcccc@{}}\toprule
Dataset & \shortstack{Native \five{}\\$x_i$} & \shortstack{Constant\\$\bar m_{\mathrm{train}}$} & \shortstack{Direct\\$f_\theta(B)_i$} & \shortstack{Residual\\$x_i+r_\phi(B)_i$}\\\midrule
CAM16 & $0.1759$ & \cellcolor{scorebetter}$0.0875$ & \cellcolor{scorebetter}$\mathbf{0.0667}$ & \cellcolor{scorebetter}$0.0935$\\
CAM17 & $0.1939$ & \cellcolor{scorebetter}$0.0912$ & \cellcolor{scorebetter}$\mathbf{0.0776}$ & \cellcolor{scorebetter}$0.1064$\\
Private-CRC & $0.1876$ & \cellcolor{scorebetter}$\mathbf{0.1051}$ & \cellcolor{scorebetter}$0.1331$ & \cellcolor{scorebetter}$0.1675$\\
PANDA & $0.1554$ & \cellcolor{scorebetter}$\mathbf{0.0837}$ & \cellcolor{scorebetter}$0.0988$ & \cellcolor{scorebetter}$0.1231$\\
BRACS & $0.1882$ & \cellcolor{scorebetter}$0.1099$ & \cellcolor{scorebetter}$\mathbf{0.0936}$ & \cellcolor{scorebetter}$0.1102$\\
\midrule
KIRC & $0.1697$ & \cellcolor{scorebetter}$0.1218$ & \cellcolor{scorebetter}$\mathbf{0.1084}$ & \cellcolor{scorebetter}$0.1182$\\
KIRP & $0.1549$ & \cellcolor{scorebetter}$0.1099$ & \cellcolor{scorebetter}$\mathbf{0.0968}$ & \cellcolor{scorebetter}$0.1040$\\
LUAD & $0.2074$ & \cellcolor{scorebetter}$0.0954$ & \cellcolor{scorebetter}$\mathbf{0.0953}$ & \cellcolor{scorebetter}$0.1333$\\
STAD & $0.1828$ & \cellcolor{scorebetter}$0.0927$ & \cellcolor{scorebetter}$\mathbf{0.0866}$ & \cellcolor{scorebetter}$0.1080$\\
UCEC & $0.1857$ & \cellcolor{scorebetter}$\mathbf{0.1118}$ & \cellcolor{scorebetter}$0.1183$ & \cellcolor{scorebetter}$0.1261$\\
\bottomrule\end{tabular}
\end{lrbox}
\begin{wraptable}[19]{R}{\wd\ResultsTableBox}
\setlength{\abovecaptionskip}{0pt}
\begingroup
\small
\caption{\textbf{Direct and residual learning of the regional teacher mean.} Each column estimates the target $m_i$, and native uses the unchanged \five{} feature $x_i$. The constant $\bar m_{\mathrm{train}}$ is the target mean estimated only from distillation training groups. Direct and residual predictions use the networks defined in Section 4, with zero initialization of their final layers. }
\label{tab:regional-mean-reconstruction}
\endgroup
\usebox{\ResultsTableBox}
\end{wraptable}

Next, we evaluate how well the student predicts teacher targets from \five{} inputs. We compare direct and residual prediction, and examine prediction of the mean and spatial components separately.

\subsection{Can Low-Mag Predict the Target?}
\label{sec:transfer_results}

\textbf{Direct prediction reduces the gap to the target.} Direct prediction has lower MSE than residual prediction on every dataset in Table~\ref{tab:regional-mean-reconstruction}. The same ordering also holds under random final-layer initialization in Appendix~\ref{app:raw-test-prediction}. Although the two forms share the same optimum over unrestricted functions (Proposition~\ref{prop:conditional-mean}), their finite networks differ in parameterization and optimization. A residual connection does not guarantee better reconstruction of the teacher target. A lower error than the native feature is not sufficient evidence of success on the downstream task. The training target mean remains more accurate on Private-CRC, PANDA and UCEC, with nearly equal average error on LUAD. By contrast, all four predictors improve on the constant on separate evaluation groups from the nine datasets used for Stage~1 training.

\begin{wrapfigure}[31]{L}{.48\textwidth}
\setlength{\abovecaptionskip}{5pt}
\centering
\includegraphics[width=\linewidth]{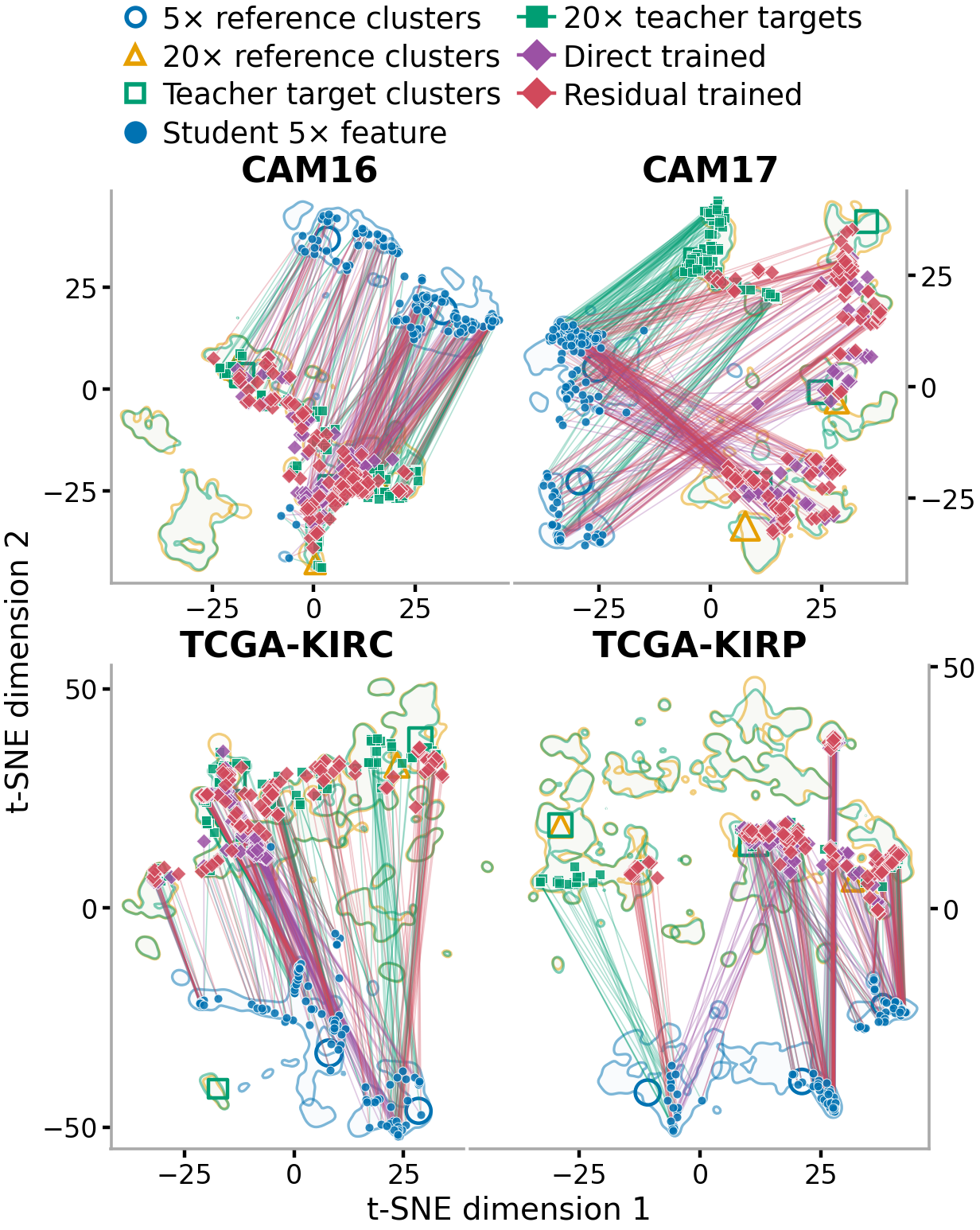}
\caption{\textbf{Regional correspondence across four datasets.} Both prediction forms are shown for one test slide per dataset. Teacher targets are regional means. Lines connect native inputs to their targets and trained outputs. Each dataset has independent t-SNE coordinates. Appendix~\ref{app:teacher-feature-maps} gives sampling details, original-coordinate measurements and the remaining six datasets.}
\label{fig:teacher-correspondence-main}
\end{wrapfigure}

Figure~\ref{fig:teacher-correspondence-main} shows different prediction patterns across four test slides. We can see that student outputs can overlap high-magnification feature groups without reproducing their corresponding regional teacher targets, while predictions do not cover the diversity of the teacher targets, showing that some information is not predictable directly from the student. The teacher means overlap the individual \twenty{} reference distribution, while student predictions cover these target groups unevenly. Both prediction forms overlap teacher features while retaining errors for their matched regional targets.

In the KIRC and KIRP slides, residual outputs extend into teacher neighbourhoods covered less extensively by direct predictions. Nevertheless, direct prediction has lower reconstruction error on all four slides in the original coordinates (Table~\ref{tab:slide-feature-errors}). Visual coverage and accuracy against corresponding regional targets describe different aspects. The quantitative analysis in Appendix~\ref{app:feature-distribution} measures variation in the original coordinates across five realizations. Direct outputs have lower centered variance than teacher means, while residual outputs have higher variance in every dataset. For both forms, distances between predicted features remain positively correlated with distances between their teacher targets. These measurements distinguish the spread of predictions from their accuracy against matched targets. Proposition~\ref{prop:conditional-mean} explains why squared error prediction need not reproduce the full variation of observed targets. Its optimum averages possible targets given the available input, rather than reproducing every possible outcome. The two prediction forms share this unrestricted optimum, but their finite architectures differ. 


\begin{wrapfigure}[16]{R}{.30\textwidth}
\setlength{\abovecaptionskip}{5pt}
\centering
\includegraphics[width=\linewidth]{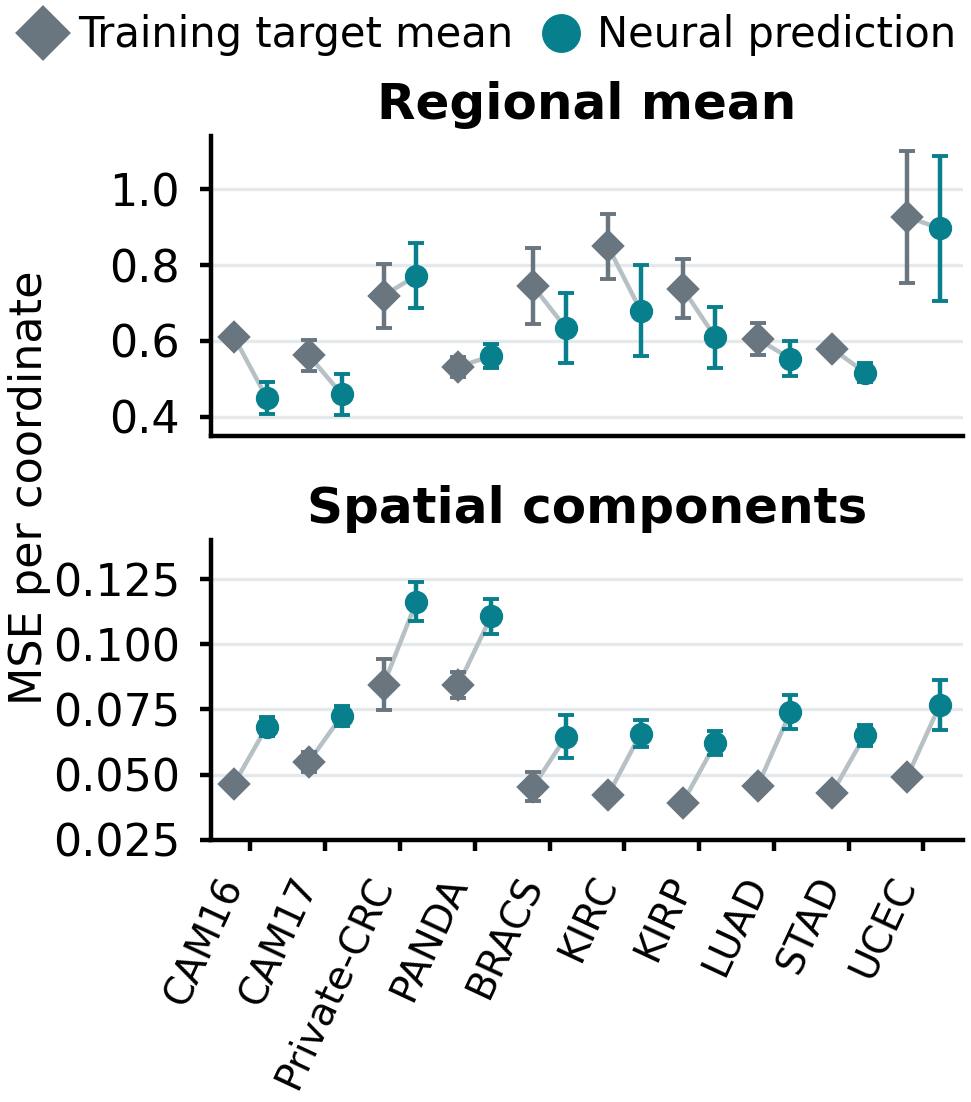}
\caption{\textbf{Reconstruction gains in the regional mean.} The network is compared with the training target, which returns a constant for inputs. }
\label{fig:component-prediction}
\end{wrapfigure}

\textbf{Reconstruction progress concentrates in the regional mean.} A separate predictor estimates regional mean and spatial components. Figure~\ref{fig:component-prediction} compares its MSE with a constant predictor that returns the average training target for each component. The network usually improves on the constant for the regional mean, but the spatial components remain less accurate than their constant baselines. This pattern also occurs on independent evaluation groups from the nine distillation datasets, not only on the excluded dataset. Thus, improvements in overall reconstruction error may primarily reflect better estimation of the regional mean, while spatial prediction remains limited.



Proposition~\ref{prop:conditional-mean} explains why the two components need separate interpretation. Error may persist as the input does not determine the target, or the network has not learned its conditional mean. The spatial results reveal a distinct learning challenge that should be measured directly rather than inferred from total error. Appendices~\ref{app:neural-correction} and~\ref{app:input-access-all} examine correction and spatial inputs, while the downstream spatial study uses separate heads retaining all feature channels.

\subsection[Does Downstream Model Benefit from Prediction?]{Does Downstream Model Benefit from Prediction?}
\label{sec:use_results}
\label{sec:r-transfer-utility}

Finally, we evaluate whether a useful prediction leads to a useful task representation. We first keep the prediction unchanged and vary native retention, then separate the effects of distillation from those of adding a network transformation. Encoder adaptation further compares how the teacher target influences the deployed representation.


\textbf{Native features can help when teacher predictions are unchanged.} Figure~\ref{fig:native-retention} shows different task outcomes for the same teacher prediction. Providing native features separately gives more consistent gains in primary scores for direct prediction, while NLL often improves for both direct and residual prediction. Reconstruction accuracy remains identical within comparisons. The task difference stems from how the model receives the prediction and native information, not from a better teacher estimate. 
\begin{wrapfigure}[20]{L}{.43\textwidth}
\setlength{\abovecaptionskip}{5pt}
\centering
\includegraphics[width=\linewidth]{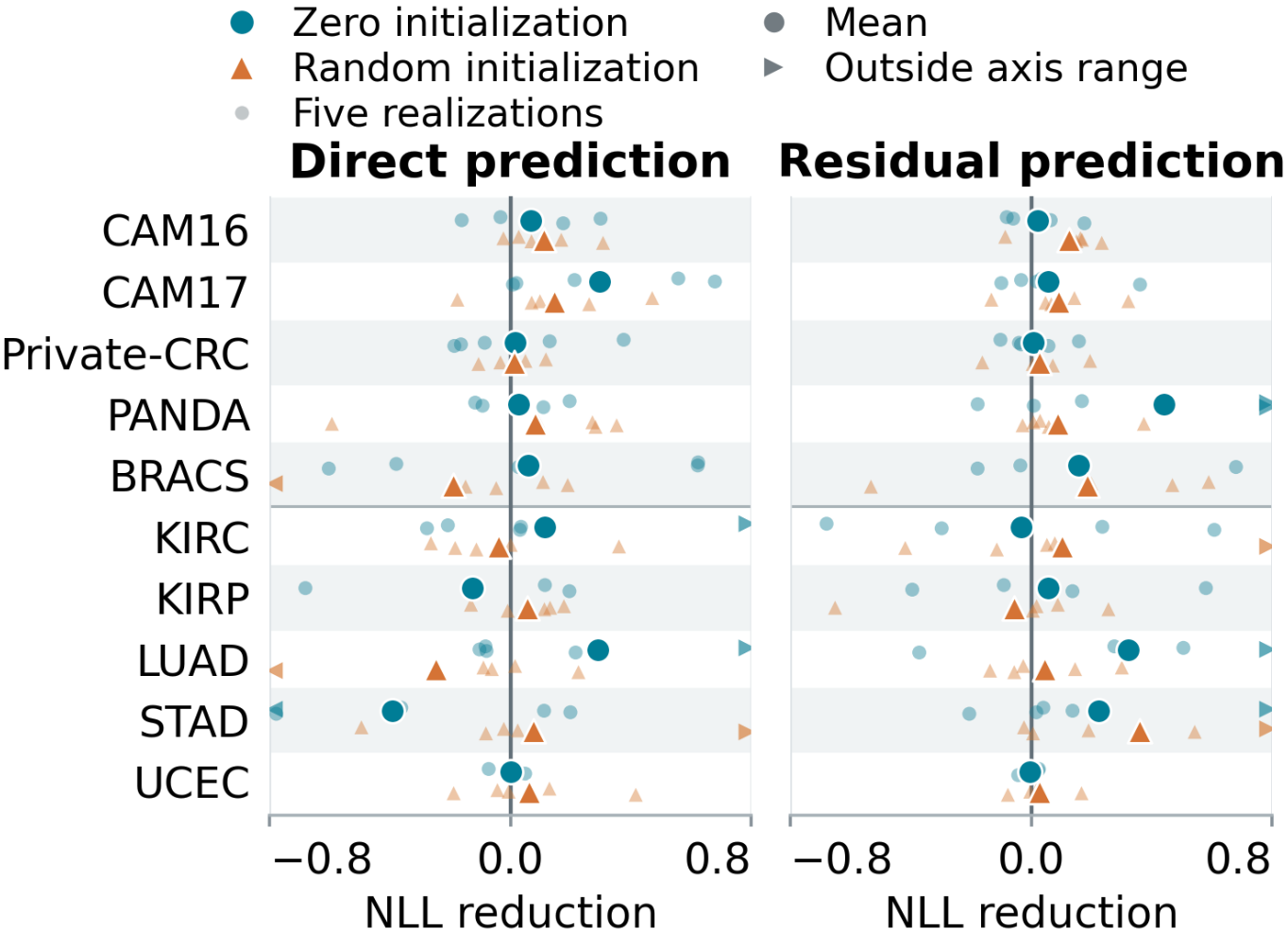}
\caption{\textbf{Retaining native features can reduce task loss without changing teacher predictions.} NLL for $[0;q_i]$ minus NLL for $[x_i;q_i]$ is positive when retention helps. Small marks show five corresponding realization differences and large marks their means. Colours distinguish final-layer initialization before distillation.}
\label{fig:native-retention}
\end{wrapfigure}

The retained input $[x_i;q_i]$ (Section 4) allows recovery of the complete native bag, resulting in zero representation loss under Theorem~\ref{thm:transfer-risk}. Removing the separate native block does not establish that all native information is lost, since some may remain in $q_i$, including through the residual connection. The task difference can therefore reflect both information preservation and how effectively the trained model uses its input. It identifies native feature retention as a consequential design choice without attributing its effect entirely to representation loss.

\begin{lrbox}{\ResultsTableBox}
\fontsize{8}{9}\selectfont
\setlength{\tabcolsep}{3pt}
\begin{tabular}{@{}lccc@{}}\toprule
 & & \multicolumn{2}{c}{Mean KD + flow toward}\\
\cmidrule(lr){3-4}
Dataset / metric & \shortstack{Native \five{}\\Frozen UNI2-h} & \shortstack{Regional\\\twenty{} mean} & \shortstack{Individual\\\twenty{} feature}\\\midrule
CAM16 / AUC & $0.8808$ & \cellcolor{scorebetter}$\mathbf{0.8827}$ & \cellcolor{scoreworse}$0.8749$\\
CAM17 / $F_1$ & $0.5025$ & \cellcolor{scorebetter}$0.5077$ & \cellcolor{scorebetter}$\mathbf{0.5211}$\\
Private-CRC / $F_1$ & $0.8411$ & \cellcolor{scoreworse}$0.8225$ & \cellcolor{scorebetter}$\mathbf{0.8492}$\\
PANDA / $\kappa$ & $\mathbf{0.8499}$ & \cellcolor{scoreworse}$0.8496$ & \cellcolor{scoreworse}$0.8495$\\
BRACS / $F_1$ & $\mathbf{0.4589}$ & \cellcolor{scoreworse}$0.4204$ & \cellcolor{scoreworse}$0.4347$\\
\midrule
KIRC / C-index & $0.7321$ & \cellcolor{scorebetter}$\mathbf{0.7378}$ & \cellcolor{scoreworse}$0.7069$\\
KIRP / C-index & $\mathbf{0.7921}$ & \cellcolor{scoreworse}$0.6861$ & \cellcolor{scoreworse}$0.7744$\\
LUAD / C-index & $0.5307$ & \cellcolor{scorebetter}$0.5537$ & \cellcolor{scorebetter}$\mathbf{0.5783}$\\
STAD / C-index & $0.5532$ & \cellcolor{scorebetter}$0.5912$ & \cellcolor{scorebetter}$\mathbf{0.5953}$\\
UCEC / C-index & $\mathbf{0.6561}$ & \cellcolor{scoreworse}$0.5924$ & \cellcolor{scoreworse}$0.5976$\\
\bottomrule\end{tabular}
\end{lrbox}
\begin{wraptable}[15]{R}{\wd\ResultsTableBox}
\setlength{\abovecaptionskip}{0pt}
\begingroup
\footnotesize
\caption{\textbf{The auxiliary teacher target changes encoder performance.} Both flow models retain mean KD and supply only encoder features to TransMIL. Table~\ref{tab:encoder-target-absolute} includes mean KD without flow.}
\label{tab:encoder-target-main}
\endgroup
\usebox{\ResultsTableBox}
\end{wraptable}

\textbf{The auxiliary teacher target changes downstream encoder performance.} Table~\ref{tab:encoder-target-main} shows that using individual teacher features or their regional mean as targets leads to distinct downstream results, even when the encoder adaptation procedure remains unchanged.

Lemma~\ref{lem:shared-output-mean} demonstrates that the gradients of these targets are equivalent only when a squared error loss supervises a single shared output. Under the flow objective, the choice of teacher target also determines the network's path input, allowing it to change the training signal. Both models keep the same mean KD objective and deploy only encoder features, as Section~\ref{sec:experiments} specifies.

This experiment uses raw images as $B$ in Equation~\ref{eq:teacher_targets}. The adapted encoder need not preserve the complete image, so both representation loss and model excess can change. Unlike native augmentation, these results cannot attribute the task effect solely to model excess. They show that the supervision objective determines whether individual teacher variation supplies a distinct learning signal, and that its task value must be evaluated for the resulting encoder. Supporting comparisons in Appendices~\ref{app:spatial-removal} and~\ref{app:all-cohort-fidelity} examine prediction delivery without adapting the encoder.

\par
\setlength{\columnsep}{\ResultsSavedColumnsep}
\setlength{\intextsep}{\ResultsSavedIntextsep}

\FloatBarrier
\section{Discussion and Conclusion}
Across ten cohorts, observed regional teacher means improve average primary scores when added to native features, but student predictions do not inherit that value automatically. In separate reconstruction tests, direct prediction has lower average error than residual prediction while its outputs vary less than teacher targets. Better reconstruction does not consistently improve task scores. Native retention changes performance at fixed predictions. Keeping the native input removes representation loss, but finite models can still differ in how they use it. Encoder adaptation also reveals no consistently preferred auxiliary teacher target. For future methods, we recommend separate tests of teacher usefulness, prediction beyond a constant estimated from training targets and downstream benefit.

The study covers one magnification pair and selected model families. Target coordinates, sampled regions and input widths differ across experiments, limiting comparisons between them. Incomplete patient linkage and sparse survival events further constrain inference. Broader studies could vary these design choices within a common protocol and examine which cohorts and cases benefit.

\clearpage
\typeout{ICLR-MAIN-TEXT-PAGES=\number\numexpr\value{page}-1\relax}
\subsection*{AI use statement}
The authors used Codex to assist in manuscript polishing, grammar checks, code implementations, and math proof reviews. The authors are responsible for the experiments, citations, mathematical arguments and interpretations.


\subsection*{Reproducibility statement}
Sections~\ref{sec:method} and~\ref{sec:experiments} define the models, training stages and comparisons. The appendices specify the protocols, seed roles, evaluation populations and uncertainty estimates. The supplementary materials include our code, numerical exports and scripts for reproduction. Appendix~\ref{app:artifacts} states what is available and what is still needed to reproduce all model fits and intervals.

\subsection*{Acknowledgment}
This work was supported by NSERC Discovery Grant RGPIN-2022-05378 [M.S.H.], Gina Cody RIF funding [M.S.H.], and FRQNT scholarship \href{https://doi.org/10.69777/2004461}{2004461} [J.C.]. We thank the \href{https://alliancecan.ca/}{Digital Research Alliance of Canada} for the computing resources used in this study.

\bibliography{references}

@article{cpath,
title = {Computational pathology: A survey review and the way forward},
journal = {Journal of Pathology Informatics},
volume = {15},
pages = {100357},
year = {2024},
issn = {2153-3539},
doi = {10.1016/j.jpi.2023.100357},
url = {https://www.sciencedirect.com/science/article/pii/S2153353923001712},
author = {Mahdi S. Hosseini and Babak Ehteshami Bejnordi and Vincent Quoc-Huy Trinh and Lyndon Chan and Danial Hasan and Xingwen Li and Stephen Yang and Taehyo Kim and Haochen Zhang and Theodore Wu and Kajanan Chinniah and Sina Maghsoudlou and Ryan Zhang and Jiadai Zhu and Samir Khaki and Andrei Buin and Fatemeh Chaji and Ala Salehi and Bich Ngoc Nguyen and Dimitris Samaras and Konstantinos N. Plataniotis},
}

@article{UNI,
  author  = {Richard J. Chen and Tong Ding and Ming Y. Lu and Drew F. K. Williamson and Guillaume Jaume and Andrew H. Song and Bowen Chen and Andrew Zhang and Daniel Shao and Muhammad Shaban and Mane Williams and Lukas Oldenburg and Luca L. Weishaupt and Judy J. Wang and Anurag Vaidya and Long Phi Le and Georg Gerber and Sharifa Sahai and Walt Williams and Faisal Mahmood},
  title   = {Towards a general-purpose foundation model for computational pathology},
  journal = {Nature Medicine},
  volume  = {30},
  number  = {3},
  pages   = {850--862},
  year    = {2024},
  doi     = {10.1038/s41591-024-02857-3},
  url     = {https://doi.org/10.1038/s41591-024-02857-3}
}

@inproceedings{TransMIL,
  author    = {Zhuchen Shao and Hao Bian and Yang Chen and Yifeng Wang and Jian Zhang and Xiangyang Ji and Yongbing Zhang},
  title     = {{TransMIL}: Transformer Based Correlated Multiple Instance Learning for Whole Slide Image Classification},
  booktitle = {Advances in Neural Information Processing Systems},
  volume    = {34},
  pages     = {2136--2147},
  year      = {2021},
  url       = {https://proceedings.neurips.cc/paper/2021/hash/10c272d06794d3e5785d5e7c5356e9ff-Abstract.html}
}

@article{RDPath,
  author  = {Joseph DiPalma and Arief A. Suriawinata and Laura J. Tafe and Lorenzo Torresani and Saeed Hassanpour},
  title   = {Resolution-based distillation for efficient histology image classification},
  journal = {Artificial Intelligence in Medicine},
  volume  = {119},
  pages   = {102136},
  year    = {2021},
  doi     = {10.1016/j.artmed.2021.102136},
  url     = {https://doi.org/10.1016/j.artmed.2021.102136}
}

@inproceedings{InterRKD,
  author    = {Roc{\'i}o del Amor and Julio Silva-Rodr{\'i}guez and Adri{\'a}n Colomer and Valery Naranjo},
  title     = {Attention to Detail: Inter-Resolution Knowledge Distillation},
  booktitle = {2023 31st European Signal Processing Conference (EUSIPCO)},
  pages     = {985--989},
  year      = {2023},
  doi       = {10.23919/EUSIPCO58844.2023.10289941},
  url       = {https://doi.org/10.23919/EUSIPCO58844.2023.10289941}
}

@article{MagAlign,
  author  = {Chu Han and Bingchao Zhao and Tianpeng Deng and Jingqi Huang and Fangfang Liu and Jiatai Lin and Shanshan Lyu and Longfei Wang and Cheng Lu and Changhong Liang and Hannah Y. Wen and Zhenwei Shi and Xiaojing Guo and Zaiyi Liu},
  title   = {A Magnification Alignment Framework Enables Computation- and Communication-Efficient Computational Pathology},
  journal = {Cancer Research},
  volume  = {86},
  number  = {12},
  pages   = {3060--3073},
  year    = {2026},
  doi     = {10.1158/0008-5472.CAN-25-3255},
  url     = {https://doi.org/10.1158/0008-5472.CAN-25-3255}
}

@article{XMAG,
  author  = {Ziyu Su and Abdul Rehman Akbar and Usama Sajjad and Anil V. Parwani and Muhammad Khalid Khan Niazi},
  title   = {Streamline pathology foundation model by cross-magnification distillation},
  journal = {arXiv preprint arXiv:2509.23097},
  year    = {2025},
  url     = {https://arxiv.org/abs/2509.23097}
}

@InProceedings{LRMIL,
        author = { Shin, Yonghan AND Jeong, Won-Ki},
        title = { { LRMIL: Efficient Low-Resolution Multiple Instance Learning via High-Resolution Knowledge Distillation for Whole Slide Image Classification } },
        booktitle = {Medical Image Computing and Computer Assisted Intervention -- MICCAI 2026},
        year = {2026},
        publisher = {Springer Nature Switzerland},
        volume = {LNCS 16891},
        month = {September},
        page = {pending}
}

@article{ResErrKD,
  author  = {Mengya Gao and Yujun Wang and Liang Wan},
  title   = {Residual error based knowledge distillation},
  journal = {Neurocomputing},
  volume  = {433},
  pages   = {154--161},
  year    = {2021},
  doi     = {10.1016/j.neucom.2020.10.113},
  url     = {https://doi.org/10.1016/j.neucom.2020.10.113}
}

@inproceedings{rahimi2007random,
 author={Ali Rahimi and Benjamin Recht},
 title={Random Features for Large-Scale Kernel Machines},
 booktitle={Advances in Neural Information Processing Systems}, volume={20}, year={2007},
 url={https://proceedings.neurips.cc/paper/2007/hash/013a006f03dbc5392effeb8f18fda755-Abstract.html}
}

@article{simeoni2025dinov3,
 author={Sim{\'e}oni, Oriane and Vo, Huy V. and Seitzer, Maximilian and Baldassarre, Federico and Oquab, Maxime and Jose, Cijo and Khalidov, Vasil and Szafraniec, Marc and Yi, Seungeun and Ramamonjisoa, Micha{\"e}l and Massa, Francisco and Haziza, Daniel and Wehrstedt, Luca and Wang, Jianyuan and Darcet, Timoth{\'e}e and Moutakanni, Th{\'e}o and Sentana, Leonel and Roberts, Claire and Vedaldi, Andrea and Tolan, Jamie and Brandt, John and Couprie, Camille and Mairal, Julien and J{\'e}gou, Herv{\'e} and Labatut, Patrick and Bojanowski, Piotr},
 title={{DINOv3}}, journal={arXiv preprint arXiv:2508.10104}, year={2025},
 url={https://arxiv.org/abs/2508.10104}
}

@inproceedings{wang2026fourier,
 author={Wang, Huanyu and Kai, Jushi and Bai, Haoli and Hou, Lu and Jiang, Bo and He, Ziwei and Lin, Zhouhan},
 editor={Favaro, Paolo and Kukelova, Zuzana and Maki, Atsuto and Rohrbach, Anna and Schindler, Konrad and Tombari, Federico},
 title={{Fourier Compressor}: Frequency-Domain Visual Token Compression for Vision-Language Models},
 booktitle={Computer Vision -- ECCV 2026}, year={2026},
 series={Lecture Notes in Computer Science}, volume={17009},
 pages={284--300}, publisher={Springer Nature Switzerland}, address={Cham},
 isbn={978-3-032-37035-8},
 doi={10.1007/978-3-032-37035-8_16},
 url={https://doi.org/10.1007/978-3-032-37035-8_16}
}

@inproceedings{qin2021fcanet,
 author={Zequn Qin and Pengyi Zhang and Fei Wu and Xi Li},
 title={{FcaNet}: Frequency Channel Attention Networks},
 booktitle={Proceedings of the IEEE/CVF International Conference on Computer Vision},
 year={2021}, pages={783--792}, url={https://openaccess.thecvf.com/content/ICCV2021/html/Qin_FcaNet_Frequency_Channel_Attention_Networks_ICCV_2021_paper.html}
}

@inproceedings{stanton2021does,
  author = {Samuel Stanton and Pavel Izmailov and Polina Kirichenko and Alexander A. Alemi and Andrew Gordon Wilson},
  title = {Does Knowledge Distillation Really Work?},
  booktitle = {Advances in Neural Information Processing Systems},
  volume = {34},
  pages = {6906--6919},
  year = {2021},
  url = {https://proceedings.neurips.cc/paper/2021/hash/376c6b9ff3bedbbea56751a84fffc10c-Abstract.html}
}

@inproceedings{xu2020usable,
  author = {Yilun Xu and Shengjia Zhao and Jiaming Song and Russell Stewart and Stefano Ermon},
  title = {A Theory of Usable Information Under Computational Constraints},
  booktitle = {International Conference on Learning Representations},
  year = {2020},
  url = {https://openreview.net/forum?id=r1eBeyHFDH}
}

@inproceedings{ilse2018attention,
  author = {Maximilian Ilse and Jakub Tomczak and Max Welling},
  title = {Attention-based Deep Multiple Instance Learning},
  booktitle = {Proceedings of the 35th International Conference on Machine Learning},
  series = {Proceedings of Machine Learning Research},
  volume = {80}, pages = {2127--2136}, year = {2018},
  url = {https://proceedings.mlr.press/v80/ilse18a.html}
}

@inproceedings{tang2024feature,
  author = {Wenhao Tang and Fengtao Zhou and Sheng Huang and Xiang Zhu and Yi Zhang and Bo Liu},
  title = {Feature Re-Embedding: Towards Foundation Model-Level Performance in Computational Pathology},
  booktitle = {Proceedings of the IEEE/CVF Conference on Computer Vision and Pattern Recognition},
  pages = {11343--11352}, year = {2024},
  url = {https://openaccess.thecvf.com/content/CVPR2024/html/Tang_Feature_Re-Embedding_Towards_Foundation_Model-Level_Performance_in_Computational_Pathology_CVPR_2024_paper.html}
}

@inproceedings{li2021dsmil,
  author = {Bin Li and Yin Li and Kevin W. Eliceiri},
  title = {Dual-Stream Multiple Instance Learning Network for Whole Slide Image Classification With Self-Supervised Contrastive Learning},
  booktitle = {Proceedings of the IEEE/CVF Conference on Computer Vision and Pattern Recognition},
  pages = {14318--14328}, year = {2021},
  url = {https://openaccess.thecvf.com/content/CVPR2021/html/Li_Dual-Stream_Multiple_Instance_Learning_Network_for_Whole_Slide_Image_Classification_CVPR_2021_paper.html}
}

@article{xu2024gigapath,
  author = {Xu, Hanwen and Usuyama, Naoto and Bagga, Jaspreet and Zhang, Sheng and Rao, Rajesh and Naumann, Tristan and Wong, Cliff and Gero, Zelalem and Gonz{\'a}lez, Javier and Gu, Yu and Xu, Yanbo and Wei, Mu and Wang, Wenhui and Ma, Shuming and Wei, Furu and Yang, Jianwei and Li, Chunyuan and Gao, Jianfeng and Rosemon, Jaylen and Bower, Tucker and Lee, Soohee and Weerasinghe, Roshanthi and Wright, Bill J. and Robicsek, Ari and Piening, Brian and Bifulco, Carlo and Wang, Sheng and Poon, Hoifung},
  title = {A whole-slide foundation model for digital pathology from real-world data},
  journal = {Nature}, volume = {630}, number = {8015}, pages = {181--188}, year = {2024},
  doi = {10.1038/s41586-024-07441-w},
  url = {https://doi.org/10.1038/s41586-024-07441-w}
}

@article{ding2025titan,
  author = {Ding, Tong and Wagner, Sophia J. and Song, Andrew H. and Chen, Richard J. and Lu, Ming Y. and Zhang, Andrew and Vaidya, Anurag J. and Jaume, Guillaume and Shaban, Muhammad and Kim, Ahrong and others},
  title = {A multimodal whole-slide foundation model for pathology},
  journal = {Nature Medicine}, volume = {31}, pages = {3749--3761}, year = {2025},
  doi = {10.1038/s41591-025-03982-3},
  url = {https://doi.org/10.1038/s41591-025-03982-3}
}

@article{vorontsov2026prism2,
  author = {Eugene Vorontsov and George Shaikovski and Adam Casson and Julian Viret and Eric Zimmermann and others},
  title = {End-to-end multimodal pathology foundation model with clinical dialogue},
  journal = {Nature Medicine}, volume = {32}, number = {9}, pages = {3203--3213}, year = {2026},
  doi = {10.1038/s41591-026-04521-4},
  url = {https://doi.org/10.1038/s41591-026-04521-4}
}

@article{cam16,
    author = {Ehteshami Bejnordi, Babak and Veta, Mitko and van Diest, Paul J. and van Ginneken, Bram and Karssemeijer, Nico and Litjens, Geert and van der Laak, Jeroen A. W. M. and {the CAMELYON16 Consortium}},
    title = {Diagnostic Assessment of Deep Learning Algorithms for Detection of Lymph Node Metastases in Women With Breast Cancer},
    journal = {JAMA},
    volume = {318},
    number = {22},
    pages = {2199--2210},
    year = {2017},
    month = {12},
    issn = {0098-7484},
    doi = {10.1001/jama.2017.14585},
    url = {https://doi.org/10.1001/jama.2017.14585},
    eprint = {https://jamanetwork.com/journals/jama/articlepdf/2665774/jama_ehteshami_bejnordi_2017_oi_170113.pdf},
}

@ARTICLE{cam17,
  author={Bándi, Péter and Geessink, Oscar and Manson, Quirine and Van Dijk, Marcory and Balkenhol, Maschenka and Hermsen, Meyke and Ehteshami Bejnordi, Babak and Lee, Byungjae and Paeng, Kyunghyun and Zhong, Aoxiao and Li, Quanzheng and Zanjani, Farhad Ghazvinian and Zinger, Svitlana and Fukuta, Keisuke and Komura, Daisuke and Ovtcharov, Vlado and Cheng, Shenghua and Zeng, Shaoqun and Thagaard, Jeppe and Dahl, Anders B. and Lin, Huangjing and Chen, Hao and Jacobsson, Ludwig and Hedlund, Martin and Çetin, Melih and Hal{\i}c{\i}, Eren and Jackson, Hunter and Chen, Richard and Both, Fabian and Franke, Jörg and Küsters-Vandevelde, Heidi and Vreuls, Willem and Bult, Peter and van Ginneken, Bram and van der Laak, Jeroen and Litjens, Geert},
  journal={IEEE Transactions on Medical Imaging},
  title={From Detection of Individual Metastases to Classification of Lymph Node Status at the Patient Level: The {CAMELYON17} Challenge},
  year={2019},
  volume={38},
  number={2},
  pages={550--560},
  url={https://doi.org/10.1109/TMI.2018.2867350},
  doi={10.1109/TMI.2018.2867350}}

@article{bracs,
    author = {Brancati, Nadia and Anniciello, Anna Maria and Pati, Pushpak and Riccio, Daniel and Scognamiglio, Giosuè and Jaume, Guillaume and De Pietro, Giuseppe and Di Bonito, Maurizio and Foncubierta, Antonio and Botti, Gerardo and Gabrani, Maria and Feroce, Florinda and Frucci, Maria},
    title = {{BRACS}: A Dataset for {BReAst} Carcinoma Subtyping in {H\&E} Histology Images},
    journal = {Database},
    volume = {2022},
    pages = {baac093},
    year = {2022},
    month = {October},
    issn = {1758-0463},
    doi = {10.1093/database/baac093},
    url = {https://doi.org/10.1093/database/baac093},
    eprint = {https://academic.oup.com/database/article-pdf/doi/10.1093/database/baac093/46880223/baac093.pdf},
}

@Article{panda,
author={Bulten, Wouter
and Kartasalo, Kimmo
and Chen, Po-Hsuan Cameron
and Str{\"o}m, Peter
and Pinckaers, Hans
and Nagpal, Kunal
and Cai, Yuannan
and Steiner, David F.
and van Boven, Hester
and Vink, Robert
and Hulsbergen-van de Kaa, Christina
and van der Laak, Jeroen
and Amin, Mahul B.
and Evans, Andrew J.
and van der Kwast, Theodorus
and Allan, Robert
and Humphrey, Peter A.
and Gr{\"o}nberg, Henrik
and Samaratunga, Hemamali
and Delahunt, Brett
and Tsuzuki, Toyonori
and H{\"a}kkinen, Tomi
and Egevad, Lars
and Demkin, Maggie
and Dane, Sohier
and Tan, Fraser
and Valkonen, Masi
and Corrado, Greg S.
and Peng, Lily
and Mermel, Craig H.
and Ruusuvuori, Pekka
and Litjens, Geert
and Eklund, Martin
and {the PANDA challenge consortium}
and Brilhante, Am{\'e}rico
and {\c{C}}ak{\i}r, Asl{\i}
and Farr{\'e}, Xavier
and Geronatsiou, Katerina
and Molini{\'e}, Vincent
and Pereira, Guilherme
and Roy, Paromita
and Saile, G{\"u}nter
and Salles, Paulo G. O.
and Schaafsma, Ewout
and Tschui, Jo{\"e}lle
and Billoch-Lima, Jorge
and Pereira, Em{\'i}io M.
and Zhou, Ming
and He, Shujun
and Song, Sejun
and Sun, Qing
and Yoshihara, Hiroshi
and Yamaguchi, Taiki
and Ono, Kosaku
and Shen, Tao
and Ji, Jianyi
and Roussel, Arnaud
and Zhou, Kairong
and Chai, Tianrui
and Weng, Nina
and Grechka, Dmitry
and Shugaev, Maxim V.
and Kiminya, Raphael
and Kovalev, Vassili
and Voynov, Dmitry
and Malyshev, Valery
and Lapo, Elizabeth
and Campos, Manuel
and Ota, Noriaki
and Yamaoka, Shinsuke
and Fujimoto, Yusuke
and Yoshioka, Kentaro
and Juvonen, Joni
and Tukiainen, Mikko
and Karlsson, Antti
and Guo, Rui
and Hsieh, Chia-Lun
and Zubarev, Igor
and Bukhar, Habib S. T.
and Li, Wenyuan
and Li, Jiayun
and Speier, William
and Arnold, Corey
and Kim, Kyungdoc
and Bae, Byeonguk
and Kim, Yeong Won
and Lee, Hong-Seok
and Park, Jeonghyuk},
title={Artificial intelligence for diagnosis and {Gleason} grading of prostate cancer: the {PANDA} challenge},
journal={Nature Medicine},
year={2022},
month={Jan},
day={01},
volume={28},
number={1},
pages={154--163},
issn={1546-170X},
doi={10.1038/s41591-021-01620-2},
url={https://doi.org/10.1038/s41591-021-01620-2}
}

@misc{NCI_TCGA,
  author       = {{National Cancer Institute}},
  title        = {{The Cancer Genome Atlas Program (TCGA)}},
  howpublished = {\url{https://www.cancer.gov/ccg/research/genome-sequencing/tcga}},
  note         = {Accessed 21 March 2026},
  year         = {n.d.},
}

@article{atlaspatch,
  title   = {{AtlasPatch}: Scalable Foundation Model-based Tissue Detection and Patch Extraction for Computational Pathology},
  author  = {Alagha, Ahmed and Leclerc, Christopher and Kotp, Yousef and Metwally, Omar and Moras, Calvin and Rentopoulos, Peter and Rostami, Ghodsiyeh and Nguyen, Bich Ngoc and Baig, Jumanah and Khellaf, Abdelhakim and Trinh, Vincent Quoc-Huy and Mizouni, Rabeb and Otrok, Hadi and Bentahar, Jamal and Hosseini, Mahdi S.},
  journal = {arXiv preprint arXiv:2602.03998},
  year    = {2026},
  url     = {https://arxiv.org/abs/2602.03998}
}

@article{hinton2015distilling,
  author  = {Geoffrey Hinton and Oriol Vinyals and Jeff Dean},
  title   = {Distilling the Knowledge in a Neural Network},
  journal = {arXiv preprint arXiv:1503.02531},
  year    = {2015},
  note    = {Presented at the NIPS 2014 Deep Learning Workshop},
  url     = {https://arxiv.org/abs/1503.02531}
}

@article{vapnik2009new,
  author  = {Vladimir Vapnik and Akshay Vashist},
  title   = {A new learning paradigm: Learning using privileged information},
  journal = {Neural Networks},
  volume  = {22},
  number  = {5--6},
  pages   = {544--557},
  year    = {2009},
  doi     = {10.1016/j.neunet.2009.06.042},
  url     = {https://doi.org/10.1016/j.neunet.2009.06.042}
}

@inproceedings{lopezpaz2016unifying,
  author    = {David Lopez-Paz and L{\'e}on Bottou and Bernhard Sch{\"o}lkopf and Vladimir Vapnik},
  title     = {Unifying distillation and privileged information},
  booktitle = {International Conference on Learning Representations},
  year      = {2016},
  url       = {https://arxiv.org/abs/1511.03643}
}

@inproceedings{ravi2025sam2,
  author={Ravi, Nikhila and Gabeur, Valentin and Hu, Yuan-Ting and Hu, Ronghang and Ryali, Chaitanya and Ma, Tengyu and Khedr, Haitham and R{\"a}dle, Roman and Rolland, Chloe and Gustafson, Laura and Mintun, Eric and Pan, Junting and Alwala, Kalyan Vasudev and Carion, Nicolas and Wu, Chao-Yuan and Girshick, Ross and Doll{\'a}r, Piotr and Feichtenhofer, Christoph},
  title={{SAM 2}: Segment Anything in Images and Videos},
  booktitle={International Conference on Learning Representations},
  volume={2025},
  pages={28085--28128},
  year={2025},
  url={https://proceedings.iclr.cc/paper_files/paper/2025/hash/45c1f6a8cbf2da59ebf2c802b4f742cd-Abstract-Conference.html}
}

@article{hendrycks2016gelu,
  author = {Dan Hendrycks and Kevin Gimpel},
  title = {Gaussian Error Linear Units ({GELUs})},
  journal = {arXiv preprint arXiv:1606.08415},
  year = {2016},
  url = {https://arxiv.org/abs/1606.08415}
}

@inproceedings{hu2022lora,
  title = {{LoRA}: Low-Rank Adaptation of Large Language Models},
  author = {Edward J. Hu and Yelong Shen and Phillip Wallis and Zeyuan Allen-Zhu and Yuanzhi Li and Shean Wang and Lu Wang and Weizhu Chen},
  booktitle = {International Conference on Learning Representations},
  year = {2022},
  url = {https://openreview.net/forum?id=nZeVKeeFYf9}
}

@inproceedings{lipman2023flow,
  title = {Flow Matching for Generative Modeling},
  author = {Yaron Lipman and Ricky T. Q. Chen and Heli Ben-Hamu and Maximilian Nickel and Matt Le},
  booktitle = {International Conference on Learning Representations},
  year = {2023},
  url = {https://openreview.net/forum?id=PqvMRDCJT9t}
}

@inproceedings{romero2015fitnets,
  author    = {Adriana Romero and Nicolas Ballas and Samira Ebrahimi Kahou and Antoine Chassang and Carlo Gatta and Yoshua Bengio},
  title     = {{FitNets}: Hints for Thin Deep Nets},
  booktitle = {International Conference on Learning Representations},
  year      = {2015},
  url       = {https://arxiv.org/abs/1412.6550}
}

@inproceedings{zagoruyko2017attention,
  author    = {Sergey Zagoruyko and Nikos Komodakis},
  title     = {Paying More Attention to Attention: Improving the Performance of Convolutional Neural Networks via Attention Transfer},
  booktitle = {International Conference on Learning Representations},
  year      = {2017},
  url       = {https://openreview.net/forum?id=Sks9_ajex}
}

@inproceedings{hewitt2021conditional,
  title = {Conditional probing: measuring usable information beyond a baseline},
  author = {Hewitt, John and Ethayarajh, Kawin and Liang, Percy and Manning, Christopher},
  editor = {Moens, Marie-Francine and Huang, Xuanjing and Specia, Lucia and Yih, Scott Wen-tau},
  booktitle = {Proceedings of the 2021 Conference on Empirical Methods in Natural Language Processing},
  month = nov,
  year = {2021},
  address = {Online and Punta Cana, Dominican Republic},
  publisher = {Association for Computational Linguistics},
  pages = {1626--1639},
  doi = {10.18653/v1/2021.emnlp-main.122},
  url = {https://aclanthology.org/2021.emnlp-main.122/}
}

@inproceedings{zhang2025rsd,
  author = {Weijia Zhang and Yuehao Liu and Wu Ran and Chao Ma},
  title = {Cross-Architecture Distillation Made Simple with Redundancy Suppression},
  booktitle = {Proceedings of the IEEE/CVF International Conference on Computer Vision},
  year = {2025},
  pages = {23256--23266},
  doi = {10.1109/ICCV51701.2025.02159},
  url = {https://openaccess.thecvf.com/content/ICCV2025/html/Zhang_Cross-Architecture_Distillation_Made_Simple_with_Redundancy_Suppression_ICCV_2025_paper.html}
}

@inproceedings{li2025pathvq,
  author = {Honglin Li and Zhongyi Shui and Yunlong Zhang and Chenglu Zhu and Lin F. Yang},
  title = {{PathVQ}: Reforming Computational Pathology Foundation Model for Whole Slide Image Analysis via Vector Quantization},
  booktitle = {Advances in Neural Information Processing Systems 38},
  year = {2025},
  doi = {10.52202/085713-2086},
  url = {https://proceedings.neurips.cc/paper_files/paper/2025/hash/59f278de1619bdb6b53fd04e8e0976e0-Abstract-Conference.html}
}

@article{chief2024,
  author = {Xiyue Wang and Junhan Zhao and Eliana Marostica and Wei Yuan and Jietian Jin and Jiayu Zhang and Ruijiang Li and Hongping Tang and Kanran Wang and Yu Li and Fang Wang and Yulong Peng and Junyou Zhu and Jing Zhang and Christopher R. Jackson and Jun Zhang and Deborah Dillon and Nancy U. Lin and Lynette Sholl and Thomas Denize and David Meredith and Keith L. Ligon and Sabina Signoretti and Shuji Ogino and Jeffrey A. Golden and MacLean P. Nasrallah and Xiao Han and Sen Yang and Kun-Hsing Yu},
  title = {A pathology foundation model for cancer diagnosis and prognosis prediction},
  journal = {Nature},
  volume = {634},
  pages = {970--978},
  year = {2024},
  doi = {10.1038/s41586-024-07894-z},
  url = {https://www.nature.com/articles/s41586-024-07894-z}
}

@inproceedings{qian2025adapt,
  author = {Chengyao Qian and Trung Le and Mehrtash Harandi},
  title = {A Good Teacher Adapts Their Knowledge for Distillation},
  booktitle = {Proceedings of the IEEE/CVF International Conference on Computer Vision},
  pages = {1239--1248},
  year = {2025},
  url = {https://openaccess.thecvf.com/content/ICCV2025/html/Qian_A_Good_Teacher_Adapts_Their_Knowledge_for_Distillation_ICCV_2025_paper.html}
}

@inproceedings{lee2025customkd,
  author = {Jungsoo Lee and Debasmit Das and Munawar Hayat and Sungha Choi and Kyuwoong Hwang and Fatih Porikli},
  title = {{CustomKD}: Customizing Large Vision Foundation for Edge Model Improvement via Knowledge Distillation},
  booktitle = {Proceedings of the IEEE/CVF Conference on Computer Vision and Pattern Recognition},
  pages = {25176--25186},
  year = {2025},
  url = {https://openaccess.thecvf.com/content/CVPR2025/html/Lee_CustomKD_Customizing_Large_Vision_Foundation_for_Edge_Model_Improvement_via_CVPR_2025_paper.html}
}

@article{vandermaaten2008visualizing,
  author = {Laurens van der Maaten and Geoffrey Hinton},
  title = {Visualizing Data using {t-SNE}},
  journal = {Journal of Machine Learning Research},
  volume = {9},
  number = {86},
  pages = {2579--2605},
  year = {2008},
  url = {https://jmlr.org/papers/v9/vandermaaten08a.html}
}

@inproceedings{tiapkin2025teacherhacking,
  author = {Daniil Tiapkin and Daniele Calandriello and Johan Ferret and Sarah Perrin and Nino Vieillard and Alexandre Rame and Mathieu Blondel},
  title = {On Teacher Hacking in Language Model Distillation},
  booktitle = {Proceedings of the 42nd International Conference on Machine Learning},
  series = {Proceedings of Machine Learning Research},
  volume = {267},
  pages = {59552--59569},
  year = {2025},
  publisher = {PMLR},
  url = {https://proceedings.mlr.press/v267/tiapkin25a.html}
}

@article{kriegeskorte2008rsa,
  title = {Representational similarity analysis---connecting the branches of systems neuroscience},
  author = {Kriegeskorte, Nikolaus and Mur, Marieke and Bandettini, Peter},
  journal = {Frontiers in Systems Neuroscience},
  volume = {2},
  pages = {4},
  year = {2008},
  doi = {10.3389/neuro.06.004.2008}
}

@article{JSSv109i03,
  title={{openTSNE}: A Modular {Python} Library for {t-SNE} Dimensionality Reduction and Embedding},
  volume={109},
  url={https://www.jstatsoft.org/index.php/jss/article/view/v109i03},
  doi={10.18637/jss.v109.i03},
  number={3},
  journal={Journal of Statistical Software},
  author={Poli{\v c}ar, Pavlin G. and Stra{\v z}ar, Martin and Zupan, Bla{\v z}},
  year={2024},
  pages={1--30}
}

@inproceedings{
yang2026sigliphd,
title={Sig{LIP}-{HD} by Fine-to-Coarse Supervision},
author={Lihe Yang and Zhen Zhao and Hengshuang Zhao},
booktitle={The Fourteenth International Conference on Learning Representations},
year={2026},
url={https://openreview.net/forum?id=XeLrfKEOZS}
}

@InProceedings{das-mil,
author="Bontempo, Gianpaolo
and Porrello, Angelo
and Bolelli, Federico
and Calderara, Simone
and Ficarra, Elisa",
editor="Greenspan, Hayit
and Madabhushi, Anant
and Mousavi, Parvin
and Salcudean, Septimiu
and Duncan, James
and Syeda-Mahmood, Tanveer
and Taylor, Russell",
title="DAS-MIL: Distilling Across Scales for MIL Classification of Histological WSIs",
booktitle="Medical Image Computing and Computer Assisted Intervention -- MICCAI 2023",
year="2023",
publisher="Springer Nature Switzerland",
address="Cham",
pages="248--258",
isbn="978-3-031-43907-0"
}
\bibliographystyle{iclr2027_conference}
\clearpage
\appendix
\addtocontents{toc}{\protect\setcounter{tocdepth}{2}}
\renewcommand{\contentsname}{Appendix contents}
\tableofcontents
\clearpage
\section{Definitions and Proofs}
\label{app:theory}
This appendix gives the proofs behind the three study questions. The risk decomposition distinguishes teacher information, information lost during delivery and the ability of a trained task model to use its input. The squared error identities clarify what a shared regional output learns and what prediction can recover.

\subsection{Information access and task risk}
\label{app:transfer-risk}
\paragraph{Definitions and facts.}
Consider a fixed joint evaluation distribution of the outcome $Y$, the complete deployment input $B$ and the teacher target $T$. The outcome may be a class label or a survival observation containing time and censoring. Evaluation takes place after the training data, transformations and model parameters have been fixed, on a population independent of the training procedure. For a randomized learning procedure, we condition on its realized training data and randomness. The resulting prediction $Q=f_\theta(B)$ then uses only the declared deployment input.

Let $U$ denote the input available to the task model, such as the native input $B$, the delivered representation $Z$ or the joint input $(B,T)$. All representations and functions are measurable. A common action space specifies the allowed predictions, such as class probabilities or a discrete survival distribution. A prediction rule $a$ maps $U$ to this space. We take $\mathcal A_U$ to be the nonempty set of all such measurable rules.

The nonnegative loss $\ell$ compares a task prediction with outcome $Y$. The expectation $\mathbb E$ averages over the fixed evaluation population, with input and outcome drawn from the same example. Every risk comparison uses the same population, loss, action space and observation weighting. For a trained task rule $g$, the population risk is $\mathcal R(g;U)=\mathbb E[\ell(g(U),Y)]$. To describe what the input permits without restricting the model architecture, we define the unrestricted risk as the infimum over all allowed rules,
\begin{equation}
\rho(U)=\inf_{a\in\mathcal A_U}\mathbb E[\ell(a(U),Y)].
\label{eq:bayes-risk-definition}
\end{equation}
The notation $\rho(B,T)$ abbreviates $\rho((B,T))$, the infimum of risk with access to both inputs. This infimum need not be attained by a rule. For the decomposition, assume that $\rho(B,T)$, $\rho(B)$, $\rho(Z)$ and $\mathcal R(g;Z)$ are finite, so the differences are defined. These unrestricted optima describe the information available in a representation, whereas a finite set of trained models may achieve higher risks.

Let $V$ be another representation, obtained through a measurable map $c$ with $V=c(U)$ almost surely. Every rule $a\in\mathcal A_V$ can be used with $U$ by composing it with $c$. The resulting rule $a\circ c\in\mathcal A_U$ makes the same predictions almost surely and therefore has the same risk. Thus the rules available from $U$ include those available from $V$. Taking infima, we obtain
\begin{equation}
\rho(U)\le\rho(V).
\label{eq:risk-information-nesting}
\end{equation}
This step relies on containment of the admissible rules, so an optimal rule or a conditional density is not needed. If each representation can be recovered measurably from the other almost surely, applying the inequality in both directions gives equal optimal risks.

\begin{proof}[Proof of Theorem~\ref{thm:transfer-risk}]
The joint input $(B,T)$ contains $B$. Under either delivery form in Equation~\ref{eq:teacher_targets}, $Z$ is a measurable function of $B$. Applying Equation~\ref{eq:risk-information-nesting} to these two containments gives
\begin{equation}
\rho(B,T)\le\rho(B)\le\rho(Z).
\label{eq:transfer-risk-order}
\end{equation}
The fitted model $g$ is one admissible rule using $Z$, so the definition of the infimum also gives
\begin{equation}
\rho(Z)\le\mathcal R(g;Z),\qquad
\varepsilon_g=\mathcal R(g;Z)-\rho(Z)\ge0.
\end{equation}
Adding and subtracting the finite quantities $\rho(B)$ and $\rho(Z)$, we obtain
\begin{align}
\mathcal R(g;Z)-\rho(B,T)
={}&\bigl[\rho(B)-\rho(B,T)\bigr]\notag\\
&+\bigl[\rho(Z)-\rho(B)\bigr]
+\bigl[\mathcal R(g;Z)-\rho(Z)\bigr].
\end{align}
The first two terms are nonnegative by Equation~\ref{eq:transfer-risk-order}, and the final term is $\varepsilon_g$. This proves the decomposition and the sign of every term.

Now suppose a measurable map $c$ recovers $B$ from $Z$ almost surely. The information comparison then works in both directions. Equation~\ref{eq:risk-information-nesting} gives $\rho(Z)\le\rho(B)$, which together with Equation~\ref{eq:transfer-risk-order} yields $\rho(Z)=\rho(B)$. Representation loss is therefore zero. For two deterministic augmentations $Z_P$ and $Z_I$ that each preserve the same complete $B$, both optimal risks equal $\rho(B)$. Let $g_P$ and $g_I$ denote their trained task models, and assume both trained risks are finite. Their excess risks $\varepsilon_{g_P}$ and $\varepsilon_{g_I}$ are defined relative to the respective optima. The trained risks can therefore be written as
\begin{equation}
\mathcal R(g_P;Z_P)=\rho(B)+\varepsilon_{g_P},\qquad
\mathcal R(g_I;Z_I)=\rho(B)+\varepsilon_{g_I}
\end{equation}
and, by subtracting these expressions, we obtain Equation~\ref{eq:retained-input-risk}. The equality holds even when the trained models have different parameters and neither attains its optimal risk.
\end{proof}

\paragraph{Interpretation of the three terms.}
Teacher access measures how much joint access to $B$ and $T$ can lower unrestricted risk. Its value depends on the chosen target $T$, which may summarize only part of the original high magnification images. A model receiving $T$ alone has a different input and may also incur model excess. We therefore distinguish that model's score from the joint optimum $\rho(B,T)$.

Representation loss measures the task information lost when $B$ is transformed into $Z$. Concatenating $Z=[B,f_\theta(B)]$ makes this term zero because the first block recovers $B$. An invertible change of coordinates of $B$ has the same property. A residual sum or replacement may not recover the complete native input but may still retain everything relevant to the task. The absence of a demonstrated recovery map is therefore insufficient to establish positive representation loss.

Model excess measures the trained model's gap to the unrestricted risk for its representation. Architecture, finite data, estimation, optimization and model selection can all contribute. Equal information can support different performance in finite models \citep{xu2020usable}, including when an initial predictor transforms $B$ before any teacher training. With $B$ retained, we compare the trained predictor with its exact initialization to examine what teacher training contributes beyond the initial transformation.

The decomposition applies to a loss defined for each observation. Discrete survival NLL has this form when predictions are survival distributions and $Y$ contains time and censoring, with common time bins and observation weights. Losses that couple observations through risk sets require a joint formulation. AUC, C-index, macro $F_1$ and quadratic $\kappa$ depend on comparisons or counts across observations and remain separate evaluation criteria. For this reason, we interpret test scores as estimates of trained model performance and do not identify them with the unrestricted terms of the decomposition.

\paragraph{Classification log loss as a special case.}
For classification, let $H$ denote Shannon entropy and $\operatorname{KL}$ denote Kullback-Leibler divergence, both using natural logarithms. Consider a finite class label and allow every vector of class probabilities as a prediction. For ordinary unweighted log loss under the fixed evaluation distribution, the optimal risk is the conditional entropy $\rho(U)=H(Y\mid U)$. To verify this identity, we write $p(\cdot\mid U)$ for the true conditional class distribution. A prediction rule $a$ returns class probabilities, with $a_Y(U)$ denoting the probability assigned to the observed class. Its conditional expected log loss is
\begin{equation}
\mathbb E[-\log a_Y(U)\mid U]
=H\bigl(p(\cdot\mid U)\bigr)
+\operatorname{KL}\bigl(p(\cdot\mid U)\,\|\,a(U)\bigr).
\end{equation}
The divergence is nonnegative and vanishes at $a(U)=p(\cdot\mid U)$. Terms with zero true class probability contribute zero, using the convention $0\log 0=0$. Assigning probability zero to a class with positive true probability gives infinite loss and divergence. Averaging the identity establishes the optimal risk. Let $I$ denote conditional mutual information. Because $Z$ is a function of $B$, conditioning on $(B,Z)$ is equivalent to conditioning on $B$, so $H(Y\mid B,Z)=H(Y\mid B)$. The definitions of conditional mutual information then express the first two terms of Equation~\ref{eq:transfer-risk} as
\begin{equation}
\rho(B)-\rho(B,T)=I(Y;T\mid B),\qquad
\rho(Z)-\rho(B)=I(Y;B\mid Z).
\end{equation}
The general risk identity recovers the classification information decomposition while also accommodating survival losses without this entropy representation.

\subsection{Mean supervision for a shared regional output}
\label{app:shared-output-mean}
\paragraph{Definitions and facts.}
For vectors $u,v\in\mathbb R^d$, the Euclidean inner product is $u^\top v$ and $\|u\|_2^2=u^\top u$. Deviations from a regional mean sum to zero. Together with the expansion $\|u+v\|_2^2=\|u\|_2^2+\|v\|_2^2+2u^\top v$, this lets us separate error relative to the mean from variation among the teacher vectors.

\begin{proof}[Proof of Lemma~\ref{lem:shared-output-mean}]
Fix one region with $K\ge1$ teacher vectors, omit its index, and set $m=K^{-1}\sum_{j=1}^Kh_j$. The same prediction $q$ is compared with every teacher vector. Writing $q-h_j=(q-m)+(m-h_j)$, we expand each squared norm and average to obtain
\begin{align}
\frac1K\sum_{j=1}^K\|q-h_j\|_2^2
={}&\|q-m\|_2^2+\frac1K\sum_{j=1}^K\|m-h_j\|_2^2\notag\\
&+2(q-m)^\top\frac1K\sum_{j=1}^K(m-h_j).
\label{eq:app-child-mean}
\end{align}
The cross term vanishes because $K^{-1}\sum_j(m-h_j)=m-m=0$, leaving Equation~\ref{eq:shared_child_mean}. The variation among teacher vectors contributes a term independent of $q$. Only $\|q-m\|_2^2$ depends on the prediction, so both objectives give the output gradient $2(q-m)$. For a differentiable predictor $q_\theta$ with teacher features fixed with respect to $\theta$, the chain rule gives the same parameter gradient $2J_\theta^\top(q_\theta-m)$, where $J_\theta=\partial q_\theta/\partial\theta$ is its output Jacobian.
\end{proof}

\paragraph{Scope of the identity.}
The comparison uses fixed teacher features, uniform weights within each region, common coordinates and the same fixed region weights in both objectives. For a shared output $q$ that does not depend on the sampled index $J$, uniform sampling gives $\mathbb E_J[2(q-h_J)]=2(q-m)$. The expected parameter gradient also agrees when the output Jacobian is shared. Sampling can still change gradient variance and the finite training trajectory. Equal weighting of all patches gives regions with more teacher patches greater weight, which can change the aggregate objective.

Separate positional outputs depart from the assumption of a shared output, and other loss functions are not covered by this squared error identity. Jointly learning the teacher can introduce gradients through the variation term, so equality of student gradients with the teacher held fixed does not imply equality of the full parameter gradients. The flow objective in Equation~\ref{eq:encoder-flow-kd} also falls outside the setting of a shared output. Its sampled target changes the path input $z_s$, allowing the network output to depend on which target was drawn. We therefore use the lemma to establish equivalence of supervision under its assumptions, leaving the task information contained in the teacher collection as a separate question.

\subsection{Direct and residual prediction under squared error}
\label{sec:conditional_mean}
\label{app:conditional-mean-proof}
\paragraph{Definitions and facts.}
A vector $V$ is square-integrable when $\mathbb E\|V\|_2^2<\infty$. Its conditional mean $\mu(B)=\mathbb E[V\mid B]$ averages possible target values for a given deployment input, coordinate by coordinate. This differs from averaging observed vectors within one region. Conditional Jensen's inequality gives $\mathbb E\|\mu(B)\|_2^2\le\mathbb E\|V\|_2^2$, so the conditional mean is also square integrable. For an integrable random variable $W$, the tower property $\mathbb E\mathbb E[W\mid B]=\mathbb E W$ allows expectations to be evaluated by first conditioning on $B$. A measurable function of $B$ can be taken outside that conditional expectation when the product is integrable. For the products below, square-integrability and the Cauchy-Schwarz inequality ensure this condition holds.

\begin{proof}[Proof of Proposition~\ref{prop:conditional-mean}]
Suppress the region index and write $V=t_i$, $q=q_i$, $x=x_i$ and $\mu=\mathbb E[V\mid B]$. To separate the part of the error that depends on the prediction, we expand $V-q=(V-\mu)+(\mu-q)$ and take expectations,
\begin{align}
\mathbb E\|V-q\|_2^2
={}&\mathbb E\|V-\mu\|_2^2+\mathbb E\|\mu-q\|_2^2\notag\\
&+2\mathbb E[(V-\mu)^\top(\mu-q)].
\label{eq:app-projection-expansion}
\end{align}
The definition of the conditional mean gives $\mathbb E[V-\mu\mid B]=0$. Both $\mu$ and $q$ are functions of $B$, so their difference can be taken outside the conditional expectation. The cross term therefore vanishes,
\begin{equation}
\mathbb E[(V-\mu)^\top(\mu-q)]
=\mathbb E\left[\mathbb E[V-\mu\mid B]^\top(\mu-q)\right]=0.
\end{equation}
Substituting this result gives
\begin{equation}
\mathbb E\|V-q\|_2^2=\mathbb E\|V-\mu\|_2^2+\mathbb E\|\mu-q\|_2^2.
\label{eq:app-projection}
\end{equation}
The first term is independent of the prediction. The second is nonnegative and vanishes precisely when $q=\mu$ almost surely. Since $\mu$ is square integrable, this unique optimum belongs to the allowed prediction class.

For residual prediction, $x$ is already a square integrable function of $B$. Every square integrable residual $r(B)$ therefore defines a square integrable output $q=x+r$. Conversely, every such output defines a square integrable residual $r=q-x$. This bijection means that the unrestricted direct and residual predictors can produce exactly the same outputs. The residual optimum is consequently $r(B)=\mu(B)-x$, giving the same optimal prediction $\mu(B)$.
\end{proof}

\paragraph{Interpretation of prediction error.}
The conditional mean can vary less across observations than the teacher target. To see why, we apply Equation~\ref{eq:app-projection} to the constant prediction $q=\mathbb E[V]$. The tower property gives $\mathbb E[\mu]=\mathbb E[V]$, yielding
\begin{equation}
\mathbb E\|V-\mathbb E[V]\|_2^2
=\mathbb E\|V-\mu\|_2^2+\mathbb E\|\mu-\mathbb E[\mu]\|_2^2.
\label{eq:conditional-total-variation}
\end{equation}
The first term is nonnegative, so the total centered variation of $\mu$ cannot exceed that of $V$. Reduced variation is therefore compatible with optimal prediction under squared error. This explains a property of the unknown conditional mean, but does not identify why a particular learned network produces less varied features.

Equation~\ref{eq:app-projection} applies to each aligned region and feature coordinate in a bag. It also extends to weighting by a measurable function $w(B)$ with $0<w(B)<\infty$ almost surely, provided $\mathbb E[w(B)\|V\|_2^2]$ and $\mathbb E[w(B)\|q\|_2^2]$ are finite. Conditional Jensen's inequality ensures a finite weighted second moment for $\mu$ as well. Multiplying each squared error inside the expectations by $w(B)$ gives the corresponding weighted identity. The weight is fixed when conditioning on $B$, so the cross term still vanishes and the conditional mean remains the unique optimum almost surely. A weight that depends on the target beyond $B$ can change that optimum.

Observed reconstruction errors estimate the sum of conditional target variation and predictor error under the evaluation distribution. Reconstruction error alone does not separate these contributions, which depend on the unknown conditional mean. In interpreting the experiments, we also distinguish both terms from representation loss and model excess in Theorem~\ref{thm:transfer-risk}.

The shared unrestricted optimum leaves room for differences between neural predictors. A native skip around a bottleneck can change both the available outputs and the initial function, while training procedures can also affect the fitted result. Our experiments compare these finite predictors. Their errors depend on the evaluation population as well, since a population shift can change the conditional mean and conditional variation. Poor prediction by a trained model alone cannot establish an information limit in the available input $B$.

\clearpage
\section{Data and comparison design}
\label{app:protocol}
\label{app:datasets}
\label{app:evaluation-protocol}
The prediction and transfer studies use cross-resolution distillation in Stage~1, followed by downstream task training in Stage~2. Teacher comparisons instead fit task models directly on observed features. This appendix describes the datasets, physical fields, extraction and training procedures for these studies, then explains their differences and the variation across training realizations.
\subsection{Datasets and task partitions}
The study includes five classification and grading datasets and five cancer-specific TCGA cohorts for overall survival. Each TCGA cancer cohort forms a separate evaluation dataset. Public data come from CAMELYON16, CAMELYON17, BRACS, PANDA and TCGA \citep{cam16,cam17,bracs,panda,NCI_TCGA}. Private-CRC is a private colorectal polyp dataset. Figure~\ref{fig:dataset-diversity}(a) summarizes their anatomical specimen sites, tasks and included slide/image counts. Table~\ref{tab:cohort-inventory} distinguishes slides or images from the groups used for task splitting. It shows seed 42 because membership and counts can change across realizations. All reported evaluations use internal holdouts.

\begin{table}[htbp]
\centering\small
\caption{Included datasets, observation counts and task partitions at seed 42. The Slides/images column gives the number of observations, while training, validation and test counts refer to the grouping unit in the preceding column. Slides or image identifiers do not establish patient independence. Higher values are better for every primary metric.}
\label{tab:cohort-inventory}
\begin{tabular}{@{}llrll@{}}\toprule
Dataset & Metric & Slides/images & Group & Train / val. / test\\\midrule
CAM16 & AUC & 270 & Slide & 194 / 22 / 54\\
CAM17 & Macro $F_1$ & 499 & Challenge group & 70 / 10 / 20\\
Private-CRC & Macro $F_1$ & 837 & Slide & 604 / 67 / 166\\
BRACS & Macro $F_1$ & 546 & Patient & 135 / 17 / 37\\
PANDA & Quadratic $\kappa$ & 10,615 & Image & 7645 / 847 / 2123\\\midrule
TCGA-KIRC & Case C-index & 494 & Case & 337 / 52 / 101\\
TCGA-KIRP & Case C-index & 283 & Case & 176 / 28 / 55\\
TCGA-LUAD & Case C-index & 487 & Case & 291 / 45 / 88\\
TCGA-STAD & Case C-index & 320 & Case & 199 / 33 / 63\\
TCGA-UCEC & Case C-index & 272 & Case & 172 / 26 / 50\\\bottomrule
\end{tabular}
\end{table}

Figure~\ref{fig:dataset-diversity}(a) combines CAM16 and CAM17 into 769 lymph node records and KIRC and KIRP into 777 kidney records. The remaining sites each correspond to one evaluation dataset. CAM16 and CAM17 sample lymph nodes from breast cancer patients, while BRACS samples breast tissue and UCEC samples uterine endometrium. These specimen sites describe the tissue sampled rather than distinct disease types or independent patient populations. For a site with $n$ included records, the colored outer radius is $r=.80+.20\ln(n/270)/\ln(10615/270)$. The normalization uses the smallest and largest individual dataset counts, retaining the same anchors when sites are grouped. Equal angular sectors and compressed radii do not give areas proportional to sample size or disease prevalence.

CAM16 predicts tumour presence. CAM17 distinguishes negative slides, isolated tumour cells, micrometastases and macrometastases rather than patient pN stage. Its challenge groups combine slides from different biological patients within a centre \citep{cam17}. Private-CRC distinguishes hyperplastic polyp, sessile serrated lesion, tubular adenoma and tubulovillous adenoma, with 212, 201, 207 and 217 slides, respectively. Normal slides are excluded. BRACS predicts seven tissue and lesion categories. PANDA predicts ISUP categories 0 through 5, including benign tissue in category 0. The TCGA cohorts comprise kidney renal clear cell carcinoma (KIRC), kidney renal papillary cell carcinoma (KIRP), lung adenocarcinoma (LUAD), stomach adenocarcinoma (STAD) and uterine corpus endometrial carcinoma (UCEC) \citep{NCI_TCGA}. Their task models use four survival bins derived from the task training partition.

CAM16 and Private-CRC partitions are stratified by slide label. CAM17 assigns 14, 2 and 4 challenge groups per centre to training, validation and testing. BRACS groups slides by patient and stratifies by each patient's multilabel vector. PANDA stratifies image identifiers by provider and grade. TCGA groups by case and stratifies by event and time quartile, except UCEC, which uses event alone. The five holdouts overlap across realizations and are not independent folds. Task membership and selection rules match across models within each contrast.

The test sets differ substantially in the support available for each metric. One BRACS class has only three test patients at seed 42. Across realizations, CAM17 has 3 to 8 test slides with isolated tumour cells, and seed 44 has none in validation. At seed 42, the survival cohorts KIRC, KIRP, LUAD, STAD and UCEC contain 34, 9, 33, 27 and 7 test events and 2,083, 252, 1,550, 1,012 and 247 comparable pairs, respectively. Sparse class and event counts limit the precision of task comparisons.

\subsection{Physical fields and frozen features}
AtlasPatch selects tissue patches for separate pretrained feature encoders \citep{atlaspatch}. Its fine-tuned SAM2 segmenter \citep{ravi2025sam2} predicts a tissue mask from a box spanning a $1.25\times$ thumbnail capped at 1024 pixels. Mask contours are mapped to level-0 slide coordinates. The extraction configuration disables the minimum tissue area filter and rejection of white or black patches. A patch is accepted when at least one of four probes, offset by a quarter patch width around its centre, lies in tissue and the centre is not strictly inside a detected hole.

Both magnifications use $256\times256$ pixel patches at a stride of 256 pixels. Extraction follows the recorded objective power and resizes pyramid reads to the requested width when needed. It does not enforce a common number of microns per pixel. Physical patch extent therefore depends on slide metadata, even when nominal magnification and pixel width agree.

A \five{} patch with level-0 origin $(x,y)$ covers sixteen \twenty{} positions with origins $(x+cw,y+rw)$ for $r,c\in\{0,1,2,3\}$, where $w$ is the \twenty{} footprint. The parent footprint is $4w$, and fields extending outside the slide are excluded. New reads follow these positions from left to right within each row and then from top to bottom, independently of the global grid origins in the earlier caches.

The complete field study selects at most 128 eligible regions per slide without using task outcomes. It first samples one parent from each occupied cell of a $16\times8$ coverage grid, then fills the remaining places uniformly from eligible parents. All representations and task realizations use this selection. Reserve parents recorded in advance replace unreadable fields before bags are assembled. The resulting support retains all 837 Private-CRC slides and all 10,615 PANDA images. Native features come from earlier caches, while the aligned teacher fields use new reads. Requiring every teacher position distinguishes this support from unrestricted native \five{} bags.

The released UNI2-h checkpoint uses a modified ViT-H/14 and represents a patch by its 1,536-dimensional class token after the final LayerNorm \citep{UNI}. The earlier AtlasPatch extractor obtains its evaluation transform from the published configuration, using a 224-pixel bilinear resize, crop fraction one and ImageNet channel normalization. New extraction implements the same intended transform with an antialiased tensor resize. Before encoding, the reader applies OpenCV linear resizing when the pyramid read differs from 256 pixels.

Summary predictors also receive the 768-dimensional normalized class token from the published DINOv3 ViT-B/16 on the same \five{} field \citep{simeoni2025dinov3}. DINOv3 uses the full 256-pixel patch. Both encoders normalize RGB values in $[0,1]$ with means $(.485,.456,.406)$ and standard deviations $(.229,.224,.225)$, without random cropping or stain normalization. Concatenating the cached UNI2-h and DINOv3 vectors gives the 2,304-coordinate native base. Aligned \twenty{} UNI2-h vectors define the teacher targets. The two encoders remain frozen during summary prediction and task training.

\paragraph{Differences between extraction pipelines.}
We checked whether the earlier caches and new extraction produced the same features by re-encoding eighty physical regions. Each check uses one \five{} region from the first successfully processed slide of an extraction chunk. The relative $L_2$ discrepancy has median $.0195$ and range $[.0064,.1870]$. All ten datasets contribute between four and twenty checks. This deterministic sample shows that the pipelines can produce different vectors for the same region, while leaving the contributions of image reading, interpolation and precision unresolved.

These extraction differences qualify the comparison $M-X$, which changes both magnification content and the pipeline. By contrast, $H-M^{16}$ and $M^{16}-M$ reuse exactly the same \twenty{} vectors, as do spatial summaries and regional means. Their differences cannot arise from using cached rather than newly extracted teacher features. Trained predictors and their initial references likewise share feature banks within each study. Isolating magnification itself in $M-X$ would require extracting both scales through a common pipeline and repeating the task comparison.

\subsection{Token order and recipient processing}
Complete-field bags preserve the saved order of selected regions. Within a region, the sixteen teacher vectors follow the $4\times4$ grid from left to right and top to bottom. The repeated mean occupies those same sixteen positions. Bags are not sorted globally by coordinates. A 2,048-token cap accommodates every regional bag of at most 128 tokens and every expanded bag of at most 2,048 tokens, so loaders preserve all rows and their order. Historical bags use a separate cap of 49,152 tokens and, when needed, deterministic selection spaced over coordinates.

TransMIL first projects each token to 512 coordinates. For a bag of $N$ tokens, it sets $L=\lceil\sqrt N\rceil$ and repeats the first $L^2-N$ vectors to form a square sequence. A class token and Nystr\"om attention layer precede the pyramid position encoding generator, whose depthwise convolutions of sizes $7$, $5$ and $3$ operate on the $L\times L$ sequence and add their outputs to its input. A second attention layer and the normalized class token produce the prediction. This square follows serialized order, since TransMIL does not receive slide coordinates. The inputs $H$ and $M^{16}$ consequently share token count, order, padding and positional processing. Comparing $M^{16}$ with $M$ changes these operations together with sequence length, so it tests their combined effect.

\subsection{Patch counts and representation size}
\label{app:efficiency}
Native \five{} bags contain fewer patch vectors in every evaluation dataset, as Figure~\ref{fig:native-patch-counts}(b) shows using the patch-count presentation of XMAG \citep{XMAG}. We count stored UNI2-h feature rows for the same 14,623 slide or image records at both magnifications, covering all included task partitions. Every record has a readable, nonempty bag at both scales. The caches contain 4,247,420 \five{} rows and 62,321,475 \twenty{} rows before parent selection or task token caps. These native bags retain the patches accepted at each magnification without requiring identical tissue coverage or patchwise correspondence. Their empirical ratios differ from the exact sixteen children per parent imposed by the complete-field construction.

The two rows in Figure~\ref{fig:native-patch-counts}(b) use the same vertical scale. Each bar pair averages over the same observations at both magnifications. Reduction labels divide the \twenty{} mean by the \five{} mean, which equals the ratio of their total counts. This differs from averaging ratios calculated separately for each slide. Table~\ref{tab:cohort-inventory} retains the exact observation counts. The figure describes the included study sample rather than the complete original dataset releases.

Reducing the number of task tokens does not reduce the number of feature coordinates by the same factor when augmentation increases token width. Figure~\ref{fig:representation-budget} derives both quantities for one complete physical field. Regional means and native UNI2-h features each occupy one token, while individual teacher features and repeated means occupy sixteen. Summary augmentation also occupies one token, retaining the 2,304-coordinate native base and appending either 768 projected coordinates or four concatenated 1,536-coordinate vectors. Relative to individual teacher features, these two augmented inputs reduce the total number of feature entries by factors of 8 and approximately 2.91, respectively, while each reduces token count by a factor of sixteen. The mean and spatial controls within each augmentation study retain its full delivery width, including repeated means or zero coordinates where specified.

\begin{figure}[h]
\centering
\includegraphics[width=\linewidth]{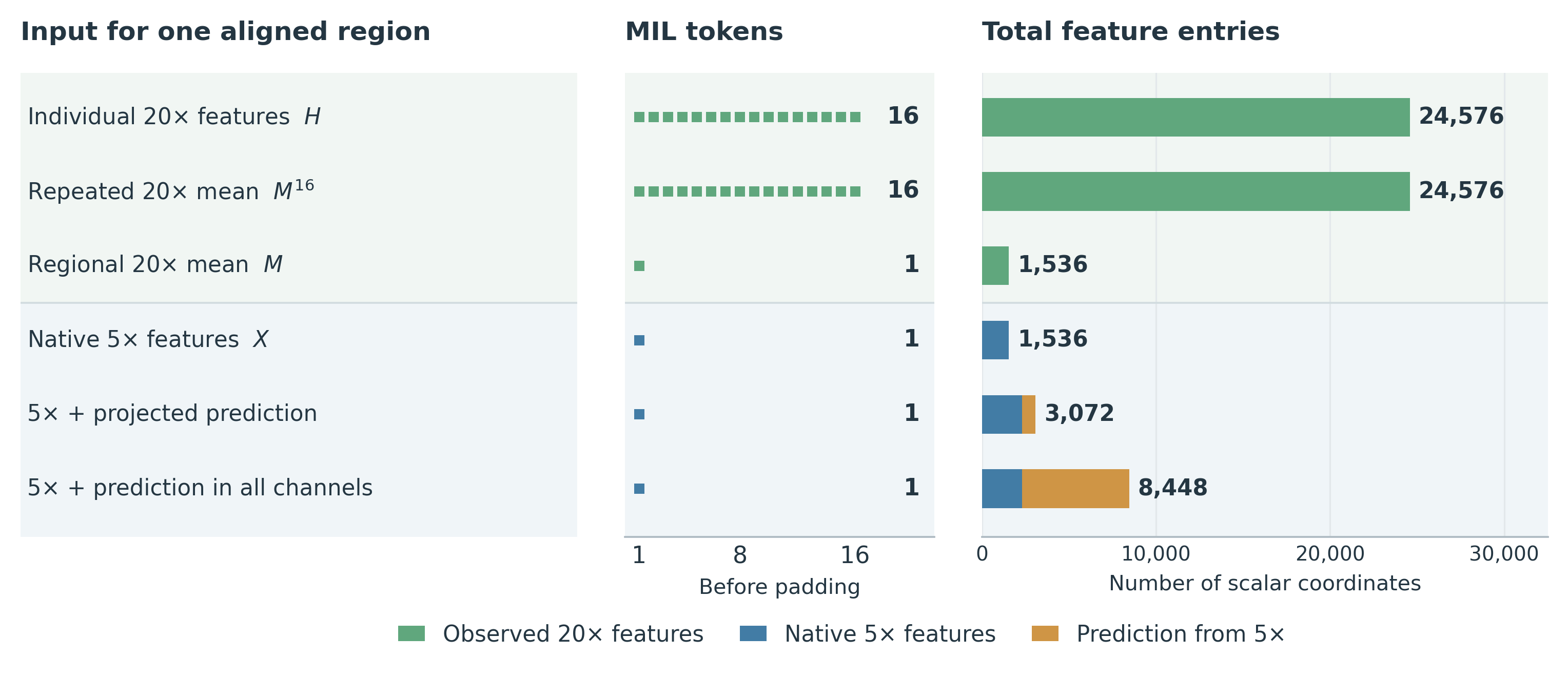}
\caption{\textbf{Token count and feature width jointly determine input size.} The middle panel counts task tokens per complete region before recipient padding. The right panel counts their scalar coordinates, with the native base and appended prediction shown separately for augmentation. Individual teacher features $H$ and repeated means $M^{16}$ each use sixteen 1,536-coordinate tokens. The regional mean $M$ and native \five{} vector $X$ each use one. Projected augmentation and augmentation in all feature channels deliver one token of width 3,072 and 8,448, respectively. The four summaries that retain every feature channel are concatenated within that token. These are derived representation dimensions, without measurements of memory use or runtime.}
\label{fig:representation-budget}
\end{figure}

Observed teacher summaries still require \twenty{} extraction. Their distilled predictions use only \five{} images at deployment, with both UNI2-h and DINOv3 supplying the retained native base. Patch counts and input dimensions describe different parts of that computation. Encoder work, recipient projection and padding, and intermediate activations also contribute to runtime and memory use, so neither figure establishes a processing speedup.

\subsection{Cross-resolution distillation and downstream fitting}
For summary prediction, the LODO strategy excludes the entire downstream evaluation dataset from Stage~1 training, transformation estimation and predictor selection. Aligned magnifications from the other nine datasets train the predictor without task labels. This exclusion is repeated for each of the ten evaluation datasets. When one TCGA cancer cohort is excluded, the other TCGA cancer cohorts remain eligible for distillation. All transformations and coordinate scales exclude the evaluated dataset, and distillation validation error selects the predictor checkpoint. Stage~2 freezes the encoders, transforms and predictor, then trains a new task model using the excluded dataset's training labels. Validation groups from that dataset select the task model, and separate test groups provide task scores. Stage~2 uses these within-dataset partitions rather than repeating the Stage~1 dataset exclusion. Test outcomes select neither stage. Figure~\ref{fig:evaluation-protocol} illustrates these roles. The separate encoder experiment in Appendix~\ref{app:encoder-target-law} instead adapts UNI2-h during Stage~1 before freezing it for task training.

In Figure~\ref{fig:evaluation-protocol}, the native vector $b_i=[x_i;u_i]$ combines UNI2-h and DINOv3 features. The mean target is $t_i=[a_i;\psi(a_i)]$, where $a_i$ averages projected teacher features and the cosine map acts on that average. The diagram's prediction bag $T$ corresponds to $Q$ in Equation~\ref{eq:teacher_targets}, so its deployed representation is $Z=[B;Q]$. Its DINOv3 feature bag $U$ is distinct from the generic task input $U$ in Theorem~\ref{thm:transfer-risk}. Appendix~\ref{app:targets} specifies this mean summary, while raw and DCT targets use their own constructions.

\begin{table}[h]
\centering\small
\caption{Task training settings for the primary TransMIL comparisons. Decay is the AdamW weight decay, and accumulation counts slide batches per optimizer update. CAM16 selects checkpoints by validation macro $F_1$ and reports test AUC. Other task models and encoder adaptation use the exceptions specified in their respective sections.}
\label{tab:task-recipe}
\begin{tabular}{@{}lrrrcl@{}}\toprule
Dataset & Epochs & Learning rate & Decay & Accum. & Selection\\\midrule
CAM16 & 15 & $.0004$ & $.05$ & 16 & Macro $F_1$\\
CAM17 & 15 & $.0002$ & $.01$ & 16 & Macro $F_1$\\
Private-CRC & 15 & $.0002$ & $.01$ & 16 & Macro $F_1$\\
BRACS & 15 & $.0002$ & $.01$ & 16 & Macro $F_1$\\
PANDA & 20 & $.0004$ & $.05$ & 16 & Quadratic $\kappa$\\
KIRC & 20 & $.0001$ & $.001$ & 1 & Case C-index\\
KIRP & 20 & $.0002$ & $.00001$ & 1 & Case C-index\\
LUAD & 20 & $.0001$ & $.00001$ & 1 & Case C-index\\
STAD & 20 & $.00005$ & $.00001$ & 1 & Case C-index\\
UCEC & 20 & $.00005$ & $.00001$ & 1 & Case C-index\\\bottomrule
\end{tabular}
\end{table}

The mean predictor illustrated in Figure~\ref{fig:evaluation-protocol} uses a GELU network of width $2304\to512\to768$ and samples the permitted training population from the other nine datasets. Its separate preprocessing and selection panels contain at most 32 groups and 256 regions per group per distillation dataset. Direct and residual prediction uses sixteen fitting groups, eight validation groups and at most 32 regions per group. The component diagnostic divides each distillation dataset into eight base training groups, eight correction groups, four validation groups for checkpoint selection and four independent evaluation groups. Appendix~\ref{app:source-learning} specifies the diagnostic, and Appendix~\ref{app:all-cohort-fidelity} gives the direct and residual architectures. These distinct samples define the scope of each reconstruction comparison.

Unless a comparison states an exception, frozen-feature predictors train for 4,096 AdamW updates with batch size 1,024, learning rate $.001$, weight decay $.0001$ and gradient clipping at 1. Checkpoints are evaluated every 512 updates. Appendix~\ref{app:neural-correction} gives the correction and denser evaluation schedules, while Appendix~\ref{app:encoder-target-law} specifies encoder adaptation. Moments estimated from distillation training data weight datasets equally, groups equally within datasets and regions equally within groups. The default coordinate scale is $s_d=\max(\mathrm{SD}_d,\tau)$, where $\tau=\max(10^{-6},10^{-3}\operatorname{median}_{d:\mathrm{SD}_d>0}\mathrm{SD}_d)$. Models within a comparison share the specified transformations and delivery scales.

The primary TransMIL models use AdamW, a cosine schedule and warmup ratio $.05$, with task settings given in Table~\ref{tab:task-recipe}. Complete fields supply at most 128 regional tokens or 2,048 expanded tokens. Studies using the original bags cap input at 49,152 tokens. Differences in support prevent absolute scores across these studies from isolating a single factor.

\subsection{Comparison designs}
\label{app:protocol-map}
Tissue support differs across studies in three ways. Complete fields contain at most 128 selected parents per slide with all sixteen newly encoded teacher positions. Historical shared support retains each \five{} row with at least one aligned cached teacher vector, allowing the number of teacher patches to vary. Unrestricted native bags retain the original \five{} rows without requiring teacher observations during task evaluation. The latter two can reach the 49,152-token cap. We compute effects within studies that share support, width and distillation populations.

\paragraph{Teacher content and token count.}
Teacher replacement in Table~\ref{tab:teacher-targets} uses complete fields and 1,536-coordinate inputs. Native features $X$ and teacher means $M$ provide one token per parent, while repeated means $M^{16}$ and individual features $H$ provide sixteen. Each receives a separate TransMIL fit without standardization using distillation training data or an additional native block. Seeds 42 through 46 vary task splitting and initialization. Means and sample SDs describe those realizations. The additional task models in Appendix~\ref{app:mil-robustness} instead use historical shared support, distillation and split seed 42, and task seeds 1002 through 1006.

\paragraph{Teacher augmentation with native features retained.}
Teacher augmentation in Table~\ref{tab:joint-teacher-main} retains the 2,304-coordinate UNI2-h and DINOv3 base on complete fields. Four projected DCT blocks add 192 coordinates each. True teacher components, a true mean with spatial zeros, and all zeros share normalization estimated from distillation training data and the total width of 3,072. Seeds 42 through 46 vary distillation transformations and task realizations. The comparison reports primary task metrics and classification NLL. Appendix~\ref{app:joint-teacher} provides the full comparison.

\paragraph{Neural prediction of regional mean and spatial variation.}
The component diagnostic in Table~\ref{tab:neural-prediction-errors} samples complete fields separately from task bags. Each distillation dataset contributes eight base training groups, eight correction groups, four validation groups for checkpoint selection and four independent evaluation groups, with at most 32 fields per group. Eight diagnostic validation groups come from the excluded dataset. A $2304\to512\to768$ GELU network predicts four 192-coordinate blocks. Moments of the native \five{} features in the base training groups determine the transformations, centering the mean block and scaling all blocks. Five distillation realizations quantify reconstruction without fitting task models. Correction keeps the base predictor frozen and learns a second head on its own groups.

\paragraph{Spatial summaries with all channels retained.}
The full-channel study in Appendix~\ref{app:n3-joint} trains separate heads for the 1,536-coordinate mean and 4,608 nonconstant coordinates. Each distillation dataset supplies sixteen fitting groups and eight selection groups, with at most 32 fields per group. Moments of native \five{} features from the distillation training groups normalize both targets, and those groups determine a common scale for all decoded summaries. Complete fields support task evaluation. Each token has the 2,304-coordinate native base plus four concatenated vectors of 1,536 coordinates, totaling 8,448. Five realizations jointly vary distillation, task splitting and task training. Primary metrics and NLL describe separate aspects of task performance. The alternative compression in Appendix~\ref{app:compression-all} shares support and total width but has its own calibration on distillation training data and three teacher representations.

\paragraph{Direct and residual prediction.}
Direct and residual predictors in Figure~\ref{fig:held_fidelity} learn the same regional mean from complete fields in the distillation datasets. Each distillation dataset contributes sixteen fitting groups and eight selection groups. The head receives one native UNI2-h vector and the slide mean, for 3,072 coordinates with common standardization and target error scales. Its 1,536-coordinate output replaces the native vector in unrestricted task bags. The residual skip changes access to native information despite identical output width. Five realizations jointly vary distillation, task splitting and task training.

The independent analysis in Appendix~\ref{app:raw-test-prediction} reuses these selected heads on complete fields from task test groups, excluding their distillation training and validation groups. Its constant reference is the training target mean estimated from the original distillation training sample. Permutation moves native features, slide context and the residual skip together. The earlier training-field figure and this test analysis retain separate evaluation populations.

\paragraph{Native retention independently of prediction.}
Native retention in Figure~\ref{fig:native-retention} holds the prediction $q_i$ unchanged and supplies either $[x_i;q_i]$ or $[0;q_i]$ at width 3,072. All native rows, task assignments, initial task weights and training settings match. Four hundred task fits cross direct or residual prediction, predictor final-layer initialization and retention over ten datasets and five realizations. The main figure shows NLL changes within each realization. Figure~\ref{fig:native-retention-realizations} adds the primary scores, and the appendix tables report absolute results for each input. Appendix~\ref{app:native-retention} explains recovery of the complete native input from the retained version.

\paragraph{Spatial access at predictor input.}
The input study in Appendix~\ref{app:input-access-all} supplies either repeated global DINOv3 features or its local grid to a predictor of the same 768-coordinate regional mean. Predictions are appended to the 2,304-coordinate global base on complete fields. The distillation training and validation groups, normalization and checkpoint selection match, and five joint realizations provide task scores for each input. The target contains no nonconstant spatial blocks. Because the global base need not recover every local feature used by the predictor, it does not establish the complete-input recovery required to cancel representation loss in the theorem.

\subsection{Experimental variation and reproducibility}
\label{app:artifacts}
Each comparison uses five realizations, but the factors that vary depend on the study. Teacher content and token count comparisons vary task splits and initialization using seeds 42 through 46. Teacher augmentation also varies the transformations estimated from Stage~1 training data. Comparisons between trained and initial predictions, together with the spatial prediction studies, additionally vary Stage~1 training. In the comparison across MIL architectures, distillation and the task split remain fixed at seed 42, while task initialization varies over seeds 1002 through 1006.

We compute each metric separately for each realization and report the equally weighted mean of the five values. Classification metrics use slide or image observations. For survival evaluation, we average slide risks within each case before computing the C-index from case outcomes. Survival NLL in the spatial studies is computed over slides using time bins estimated from the training partition of each realization. We compare absolute NLL values only within the same dataset and evaluation protocol, without pooling predictions across test sets.

The tables report absolute performance for each input and identify the reference and metric direction. Sample SDs summarize variation across the five realizations without separating the contributions of data splitting, initialization and predictor training. Since test sets can overlap, these SDs do not estimate uncertainty from five independent samples. The reported means and SDs describe the observed comparisons and do not by themselves establish a population improvement or equivalence.

The model fits underlying the main Results section use NVIDIA RTX 6000 Ada GPUs with PyTorch 2.6 and CUDA 12.4. The appendix studies of neural correction, teacher-summary compression and predictor inputs also use NVIDIA models. Training generally uses BF16 forward passes and FP32 loss computation. TF32 is disabled for CUDA matrix multiplication in the task models.

The spatial studies that retain all feature channels in Appendices~\ref{app:n3-joint} and~\ref{app:spatial-learning} use AMD MI350X GPUs, as does the comparison across MIL architectures in Appendix~\ref{app:mil-robustness}. These task models train in FP32, while the spatial predictors use BF16 forward passes and FP32 losses. The recorded software environment for the spatial study is PyTorch 2.11 with ROCm 7.2. Feature-geometry measurements and correspondence maps are computed on CPUs using saved predictors trained on NVIDIA GPUs.

The accompanying code provides training implementations, configuration examples, public split records, numerical exports and scripts for regenerating the tables and figures. Rerunning model training additionally requires the corresponding images or feature bags and pretrained weights. Private data, historical fitted checkpoints and the complete set of case predictions are not bundled. The code documentation distinguishes regeneration of the reported displays from replication of the experiments with the required inputs.

\clearpage
\section{Teacher Targets and Their Task Utility}
\label{app:numbers}
Before predicting teacher features, we test which parts help the task. This appendix defines regional means and spatial summaries, then compares their utility. Replacing native features tests a teacher representation on its own. Adding teacher features to the native input tests whether they contribute useful complementary information.
\subsection{The mean summary in the training diagram}
\label{app:targets}
Figure~\ref{fig:evaluation-protocol} illustrates distillation with a 768-coordinate summary of the regional mean. Let $P\in\mathbb R^{384\times1536}$ have independent Gaussian entries of variance $1/384$, and define $a_i=K_i^{-1}\sum_jPh_{ij}$. A random feature map follows \citet{rahimi2007random},
\begin{equation}
\psi(v)=\sqrt{\frac{2}{384}}\cos\!\left(\Omega\frac{v}{\max(\|v\|_2,10^{-6})}+\beta\right),\qquad
\Omega_{jk}\sim\mathcal N(0,1),\quad \beta_j\sim\operatorname{Unif}[0,2\pi).
\end{equation}
The diagram uses $t_i=[a_i;\psi(a_i)]$, with the random map held constant. Both blocks are functions of the same regional mean, so the second block adds no independent teacher observation. This target explains the example in the diagram. The main teacher comparison instead uses the DCT components defined below, and the raw mean prediction study uses all 1,536 original coordinates.

\subsection{The four spatial components}
\label{app:spatial-components}
Fourier Compressor reduces a set of visual tokens by applying a spatial DCT, retaining low frequency components and decoding them with an inverse DCT \citep{wang2026fourier}. We use this construction to summarize sixteen independently encoded \twenty{} patches within one \five{} region. For their $4\times4$ feature field $V$, let $C=\operatorname{DCT}_{\rm ortho}(V)/4$, where the transform is an orthonormal two-dimensional DCT of type II. With $c_r=\cos(\pi(2r+1)/8)$ for $r=0,1,2,3$, the four retained vectors are
\begin{align}
C_{00}&=\frac1{16}\sum_{r=0}^3\sum_{s=0}^3V_{rs},
& C_{01}&=\frac{\sqrt2}{16}\sum_{r=0}^3\sum_{s=0}^3c_sV_{rs},\\
C_{10}&=\frac{\sqrt2}{16}\sum_{r=0}^3\sum_{s=0}^3c_rV_{rs},
& C_{11}&=\frac1{8}\sum_{r=0}^3\sum_{s=0}^3c_rc_sV_{rs}.
\label{eq:spatial-four-components}
\end{align}
The constant component $C_{00}$ equals the regional mean. The other three components describe broad horizontal, vertical and joint changes across the feature field. These are spatial frequencies of embeddings, rather than image pixels. We call the retained set four low frequency components (LF4). FcaNet also uses DCT components, for channel attention \citep{qin2021fcanet}. We adapt the compression method to the transfer question by projecting the features, normalizing them with distillation training data and predicting them from \five{} inputs.

The projected diagnostic in Appendix~\ref{app:source-learning} uses $V_{rs}=P_{192}h_{rs}$, where $P_{192}$ contains the first 192 rows of the projection bank. Concatenating four components gives 768 outputs. Moments of native \five{} features in the base training groups center the constant block and supply the common scales. The separate study in Appendix~\ref{app:n3-joint} retains all 1,536 channels and decodes the four normalized components with an orthonormal inverse DCT on their $2\times2$ coefficient grid,
\begin{equation}
G=2\operatorname{IDCT}_{2,\rm ortho}(\widetilde C).
\label{eq:spatial-decoder}
\end{equation}
Decoding a mean-only target gives four copies of its normalized mean because of the factor of two. The four vectors summarize smooth variation across the field and are neither observed patches nor quadrant averages. Flattening them gives 6,144 coordinates. Every representation in this study combines them with the 2,304-coordinate native base $b_i=[x_i;u_i]$ using
\begin{equation}
\alpha=\sqrt{\frac{\mathbb E_{\rm train}\|b_i\|^2}{2\mathbb E_{\rm train}\|\operatorname{vec}(G_i)\|^2}},
\end{equation}
where $\mathbb E_{\rm train}$ averages over the distillation training data, and the teacher reference determines $\alpha$ on that fitting population. The projected diagnostic measures reconstruction error. The study retaining all channels also tests the decoded predictions in a task model. Their target widths, distillation samples and output transformations differ, so they provide separate evidence rather than an isolated comparison of channel count.

\subsection{Spatial frequency of the teacher features}
\label{app:frequency-sample}
To describe the variation removed by regional pooling, we examine the spatial frequencies of the teacher feature field. The magnitude heatmap follows Figure~1 of Fourier Compressor \citep{wang2026fourier}. Squared energy and permutation of patch positions distinguish the mean from organized nonconstant variation. We sample 1,000 unique complete regions per dataset from saved pools of 1,208 to 1,280 eligible regions. The resulting 10,000 regions contain 160,000 feature vectors from 381 recorded groups and 593 slides. They describe the sampled feature banks, without representing a uniform patient sample or the exact downstream test population.

Original 1,536-coordinate features and their 192-coordinate projections use the same fields. Absolute coefficient magnitudes are averaged over channels and regions, then equally across datasets before applying the display logarithm. Energy fractions divide aggregate squared coefficients, rather than averaging a ratio from each region. In the original features, DC retains 62.56\% of total energy and LF4 retains 75.98\%. The three nonconstant LF4 modes retain 35.84\% of the energy remaining after DC is removed, compared with 20.06\% after permutation. The corresponding projected fractions are 62.72\%, 76.09\% and 35.88\%. These summaries use one descriptive sample per dataset, without model fitting or repeated training seeds.

The regional mean contains most feature energy, while the retained spatial components capture a substantial share of the variation that pooling removes. Figure~\ref{fig:spatial_energy} shows that their share of nonconstant energy decreases after patch permutation. Under a uniform permutation, each of the fifteen nonconstant orthonormal modes has the same expected energy, giving a 20 percent reference for the three retained modes. The saved permutation approximates this reference. It describes spatial organization in the embeddings and motivates testing the retained components for task utility, which energy alone cannot establish.

\begin{figure}[htbp]
\centering
\includegraphics[width=\linewidth]{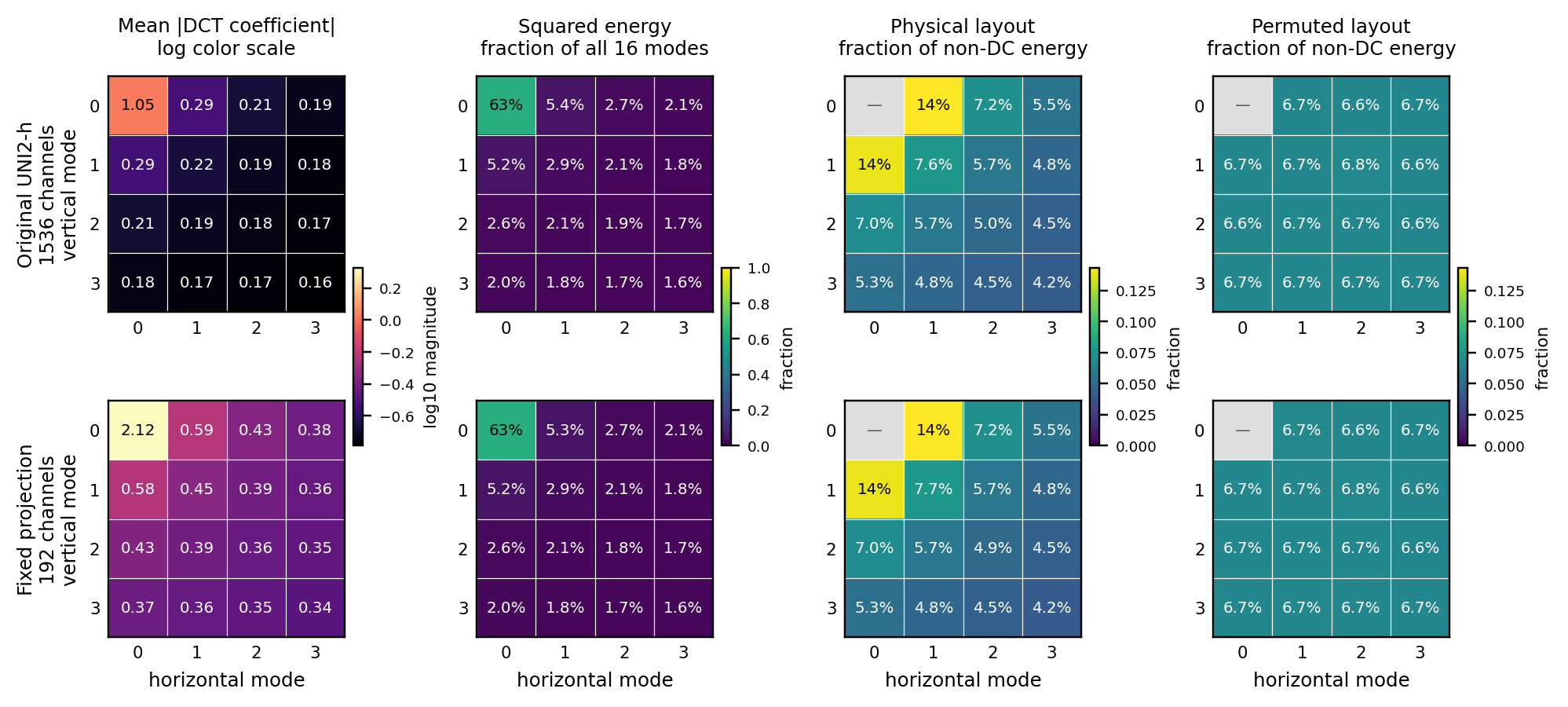}
\caption{\textbf{Broad spatial patterns capture variation beyond the regional mean.} Each row summarizes the same 1,000 complete regions from each of ten datasets, with equal dataset weight, giving 10,000 regions and 160,000 patch vectors. The top row uses all 1,536 feature channels and the bottom row their 192-coordinate projection. Axes index horizontal and vertical DCT modes, not tissue positions. Zero is constant and larger indices represent faster changes. From left to right, panels show mean absolute coefficient magnitude, fractions of total squared energy, and fractions of nonconstant energy before and after patch permutation. Brighter colors indicate larger values within each scale. Magnitude colors use a base-ten log while cell labels retain the original values. Gray marks the removed DC component. The retained LF4 modes occupy the upper-left $2\times2$ block. These descriptive feature summaries motivate testing spatial targets for task benefit.}
\label{fig:spatial_energy}
\end{figure}

\subsection{Separating regional content from token count}
To separate teacher content from repetition, we train task models on four representations of the same complete fields. Native input $X$ has one \five{} vector per region. The teacher mean $M$ replaces it with the average of sixteen aligned \twenty{} vectors, while $M^{16}$ repeats that average in sixteen slots. Individual features $H$ occupy the same slots in physical order. Comparing $H$ with $M^{16}$ holds token count and width constant. Every representation receives a separate task fit with shared initialization and selection rules. Table~\ref{tab:teacher-targets} reports their absolute performance.

\begin{table}[h]
\centering\footnotesize
\setlength{\tabcolsep}{4pt}
\caption{\textbf{Teacher content and token count have distinct task effects.} TransMIL receives native \five{} features, a regional \twenty{} mean, sixteen copies of that mean, or the sixteen individual features from the same tissue regions. Teacher inputs require \twenty{} images at evaluation and cover these regions rather than the full slide. Each score comes from a separate TransMIL fit without distillation. Full results for Table~\ref{tab:teacher-targets} are reported as mean $\pm$ sample SD across five realizations.}
\label{tab:teacher-targets-sd}
\begin{tabular}{@{}lcccc@{}}\toprule
 & \five{} input & \multicolumn{3}{c}{\twenty{} teacher construction}\\
\cmidrule(lr){3-5}
Dataset / metric & Native \five{} $X$ & Regional mean $M$ & Repeated mean $M^{16}$ & Individual \twenty{} $H$\\
Tokens per region & 1 & 1 & 16 & 16\\\midrule
CAM16 / AUC & $0.8430\pm0.0215$ & \cellcolor{scorebetter}$0.8993\pm0.0118$ & \cellcolor{scorebetter}$0.8872\pm0.0540$ & \cellcolor{scorebetter}$\mathbf{0.9047}\pm0.0360$\\
CAM17 / $F_1$ & $0.5054\pm0.0612$ & \cellcolor{scoreworse}$0.4856\pm0.0476$ & \cellcolor{scoreworse}$0.4904\pm0.0456$ & \cellcolor{scorebetter}$\mathbf{0.5418}\pm0.0593$\\
Private-CRC / $F_1$ & $0.8273\pm0.0332$ & \cellcolor{scoreworse}$0.8146\pm0.0232$ & \cellcolor{scorebetter}$0.8290\pm0.0365$ & \cellcolor{scorebetter}$\mathbf{0.8456}\pm0.0419$\\
PANDA / $\kappa$ & $0.8535\pm0.0125$ & \cellcolor{scorebetter}$0.8660\pm0.0159$ & \cellcolor{scorebetter}$0.8706\pm0.0069$ & \cellcolor{scorebetter}$\mathbf{0.8726}\pm0.0108$\\
BRACS / $F_1$ & $0.3747\pm0.0532$ & \cellcolor{scoreworse}$0.3731\pm0.0462$ & \cellcolor{scorebetter}$\mathbf{0.3849}\pm0.0454$ & \cellcolor{scorebetter}$0.3767\pm0.1097$\\
\midrule
KIRC / C-index & $0.7127\pm0.0293$ & \cellcolor{scoreworse}$0.7105\pm0.0418$ & \cellcolor{scoreworse}$0.6951\pm0.0592$ & \cellcolor{scorebetter}$\mathbf{0.7144}\pm0.0470$\\
KIRP / C-index & $0.7793\pm0.0995$ & \cellcolor{scorebetter}$\mathbf{0.7802}\pm0.1074$ & \cellcolor{scoreworse}$0.7735\pm0.1319$ & \cellcolor{scoreworse}$0.7745\pm0.1178$\\
LUAD / C-index & $0.5494\pm0.0539$ & \cellcolor{scorebetter}$\mathbf{0.5640}\pm0.0511$ & \cellcolor{scorebetter}$0.5523\pm0.0461$ & \cellcolor{scoreworse}$0.5444\pm0.0415$\\
STAD / C-index & $0.5197\pm0.0572$ & \cellcolor{scorebetter}$0.5402\pm0.0317$ & \cellcolor{scorebetter}$\mathbf{0.5640}\pm0.0683$ & \cellcolor{scorebetter}$0.5533\pm0.0692$\\
UCEC / C-index & $0.6383\pm0.1067$ & \cellcolor{scoreworse}$0.6158\pm0.1350$ & \cellcolor{scorebetter}$\mathbf{0.6602}\pm0.0645$ & \cellcolor{scorebetter}$0.6589\pm0.1348$\\
\bottomrule\end{tabular}
\end{table}

Teacher means do not consistently improve on native features when they replace them. CAM16 and PANDA improve in every realization, while other datasets show mixed directions. Individual teacher features can also help beyond repeated means, but this effect varies across realizations. CAM17 illustrates why the reference matters. Individual features improve on repeated means even though the mean itself performs worse than native input. Useful variation within a teacher representation need not make that representation a better replacement for native features.

Repeating the mean changes task scores without adding feature information. Recovering a single mean from its copies, or generating the copies from that mean, gives equal unrestricted risks under Theorem~\ref{thm:transfer-risk}. The observed changes reflect how finite models process the representation. Sequence layout, attention approximation, normalization and optimization remain possible contributors that this experiment changes together. Appendix~\ref{app:mil-robustness} examines regional means with additional task models on the historical shared support.

\subsection{Teacher information alongside the same native input}
\label{app:joint-teacher}
To test complementary teacher information, we retain the UNI2-h and DINOv3 base $b_i=[x_i;u_i]$ from Section 4 and append teacher components. Each \twenty{} feature is projected to 192 coordinates before the transform in Equation~\ref{eq:spatial-four-components}. The four blocks contain the regional mean and three spatial components, normalized using common statistics from distillation training data. The task model receives 3,072 coordinates per region. This construction follows the joint native and teacher input in Theorem~\ref{thm:transfer-risk}.

Using the symbols from Section 4, the task inputs are $[b_i;0_{768}]$, $[b_i;\widetilde m_i;0_{576}]$ and $[b_i;\widetilde m_i;s_i]$. The vector $b_i$ contains native \five{} features, $\widetilde m_i$ is the normalized projected \twenty{} regional mean, and $s_i$ concatenates its three normalized spatial components. Each representation has a separate TransMIL fit, with shared fields, task assignments, initial weights, optimization and selection within a realization. Bags contain at most 128 regions. The five realizations jointly vary distillation transformations, task splitting and training. These results reuse 150 completed reference fits and require teacher observations during evaluation.

A benefit from this augmentation shows that the teacher helps the trained recipient beyond its native input. Under the common loss and population of the theorem, that benefit combines the teacher information contribution with a change in model excess. The experiment measures complementary utility for this training procedure, while leaving those contributions unseparated.

\paragraph{Task loss.}
Classification NLL uses saved slide or image probabilities, floored at $10^{-7}$ and renormalized before averaging log loss. Survival NLL is the discrete censored likelihood averaged over slides, with four time bins derived from training cases. An event contributes survival through earlier bins and the hazard in its event bin. A censored observation contributes survival through its censoring bin. Models within a realization share bin boundaries and evaluated slides. Case C-index instead averages slide risks within a case and evaluates their ordering. NLL is reported separately for each dataset because task definitions and training-derived time bins differ. The means and sample SDs summarize the five saved realizations and do not establish a population ranking.

\begin{table}[h]
\centering\footnotesize
\setlength{\tabcolsep}{3pt}
\caption{Task NLL for teacher summaries appended to the same native input. Lower is better. The reference is native \five{} with zero padding. The direction control retains the teacher mean and component norms.}
\label{tab:joint-teacher-losses}
\begin{tabular}{@{}lcccc@{}}\toprule
Dataset & \shortstack{Native \five{}\\$[b_i;0_{768}]$} & \shortstack{Native \five{} + mean\\$[b_i;\widetilde m_i;0_{576}]$} & \shortstack{Native \five{} +\\mean and spatial\\$[b_i;\widetilde m_i;s_i]$} & \shortstack{Native \five{} +\\mean and direction control\\$[b_i;\widetilde m_i;s_i^{\rm rnd}]$}\\\midrule
CAM16 & $0.6110\!\pm\!0.1605$ & \cellcolor{scoreworse}$0.6717\!\pm\!0.3244$ & \cellcolor{scorebetter}$\mathbf{0.5663}\!\pm\!0.2296$ & \cellcolor{scorebetter}$0.5702\!\pm\!0.2110$\\
CAM17 & $\mathbf{0.8392}\!\pm\!0.4304$ & \cellcolor{scoreworse}$0.8605\!\pm\!0.3460$ & \cellcolor{scoreworse}$1.0058\!\pm\!0.5968$ & \cellcolor{scoreworse}$0.9509\!\pm\!0.5255$\\
Private-CRC & $0.6230\!\pm\!0.0551$ & \cellcolor{scorebetter}$\mathbf{0.5563}\!\pm\!0.0356$ & \cellcolor{scoreworse}$0.6503\!\pm\!0.1546$ & \cellcolor{scoreworse}$0.6924\!\pm\!0.2528$\\
PANDA & $1.4584\!\pm\!0.5817$ & \cellcolor{scorebetter}$\mathbf{1.1328}\!\pm\!0.4491$ & \cellcolor{scorebetter}$1.3559\!\pm\!0.6075$ & \cellcolor{scoreworse}$1.4657\!\pm\!0.6855$\\
BRACS & $1.8613\!\pm\!0.5405$ & \cellcolor{scorebetter}$\mathbf{1.7576}\!\pm\!0.5130$ & \cellcolor{scoreworse}$1.8878\!\pm\!0.6489$ & \cellcolor{scoreworse}$1.9374\!\pm\!0.4963$\\
KIRC & $2.0562\!\pm\!0.4878$ & \cellcolor{scorebetter}$\mathbf{1.6173}\!\pm\!0.4347$ & \cellcolor{scorebetter}$1.7888\!\pm\!0.5891$ & \cellcolor{scorebetter}$1.6763\!\pm\!0.4250$\\
KIRP & $1.2932\!\pm\!0.2403$ & \cellcolor{scorebetter}$1.2128\!\pm\!0.2770$ & \cellcolor{scorebetter}$\mathbf{1.0900}\!\pm\!0.1797$ & \cellcolor{scorebetter}$1.2335\!\pm\!0.2100$\\
LUAD & $2.7670\!\pm\!0.6057$ & \cellcolor{scorebetter}$2.5482\!\pm\!0.4887$ & \cellcolor{scorebetter}$\mathbf{2.4012}\!\pm\!0.4809$ & \cellcolor{scorebetter}$2.5907\!\pm\!0.6864$\\
STAD & $2.4675\!\pm\!0.3891$ & \cellcolor{scoreworse}$2.5468\!\pm\!0.5305$ & \cellcolor{scoreworse}$2.5089\!\pm\!0.4825$ & \cellcolor{scorebetter}$\mathbf{2.3861}\!\pm\!0.4141$\\
UCEC & $1.0027\!\pm\!0.1632$ & \cellcolor{scorebetter}$0.9424\!\pm\!0.1180$ & \cellcolor{scorebetter}$0.9429\!\pm\!0.1190$ & \cellcolor{scorebetter}$\mathbf{0.9311}\!\pm\!0.1241$\\
\bottomrule\end{tabular}
\end{table}

The regional teacher mean raises the average primary score in every dataset, as Table~\ref{tab:joint-teacher-main} shows, but its NLL effect is less consistent. Adding spatial components also has mixed effects on NLL. These results concern the same trained models. They show why probabilistic loss and discrimination or grading metrics provide distinct evidence of task utility, as the scope of Theorem~\ref{thm:transfer-risk} requires.

\paragraph{Spatial direction at the same component magnitude.}
We next test whether the direction of each spatial vector matters beyond its magnitude. Fifty additional task fits retain the true mean and replace each spatial block with a Gaussian direction scaled to its original norm, written $s_i^{\rm rnd}$. Fields, input width, task initialization and training remain unchanged. Table~\ref{tab:joint-teacher-direction} reports primary scores alongside the same native and teacher references. The original directions help some datasets and reduce performance on others. Every CAM16 realization favors the original directions, but this pattern does not generalize across datasets. The replacement retains teacher-derived norms and alters covariance together with direction. It tests sensitivity to the supplied vectors rather than an isolated biological property.

\begin{table}[h]
\centering\footnotesize
\setlength{\tabcolsep}{3pt}
\caption{Primary task scores for teacher summaries and their spatial direction control. Higher is better. The reference is native \five{} with zero padding. The task metrics follow Table~\ref{tab:joint-teacher-main}.}
\label{tab:joint-teacher-direction}
\begin{tabular}{@{}lcccc@{}}\toprule
Dataset & \shortstack{Native \five{}\\$[b_i;0_{768}]$} & \shortstack{Native \five{} + mean\\$[b_i;\widetilde m_i;0_{576}]$} & \shortstack{Native \five{} +\\mean and spatial\\$[b_i;\widetilde m_i;s_i]$} & \shortstack{Native \five{} +\\mean and direction control\\$[b_i;\widetilde m_i;s_i^{\rm rnd}]$}\\\midrule
CAM16 & $0.8455\!\pm\!0.0485$ & \cellcolor{scorebetter}$0.8621\!\pm\!0.0600$ & \cellcolor{scorebetter}$\mathbf{0.8932}\!\pm\!0.0285$ & \cellcolor{scorebetter}$0.8662\!\pm\!0.0348$\\
CAM17 & $0.5197\!\pm\!0.0336$ & \cellcolor{scorebetter}$\mathbf{0.5594}\!\pm\!0.0457$ & \cellcolor{scorebetter}$0.5198\!\pm\!0.0414$ & \cellcolor{scorebetter}$0.5297\!\pm\!0.0346$\\
Private-CRC & $0.8305\!\pm\!0.0276$ & \cellcolor{scorebetter}$\mathbf{0.8410}\!\pm\!0.0272$ & $0.8305\!\pm\!0.0233$ & \cellcolor{scoreworse}$0.8263\!\pm\!0.0342$\\
PANDA & $0.8386\!\pm\!0.0067$ & \cellcolor{scorebetter}$0.8446\!\pm\!0.0124$ & \cellcolor{scorebetter}$\mathbf{0.8456}\!\pm\!0.0114$ & \cellcolor{scoreworse}$0.8357\!\pm\!0.0089$\\
BRACS & $0.4211\!\pm\!0.0710$ & \cellcolor{scorebetter}$0.4257\!\pm\!0.0783$ & \cellcolor{scoreworse}$0.4094\!\pm\!0.0704$ & \cellcolor{scorebetter}$\mathbf{0.4381}\!\pm\!0.0722$\\
KIRC & $0.6823\!\pm\!0.0548$ & \cellcolor{scorebetter}$0.6892\!\pm\!0.0470$ & \cellcolor{scorebetter}$0.6907\!\pm\!0.0435$ & \cellcolor{scorebetter}$\mathbf{0.6990}\!\pm\!0.0550$\\
KIRP & $0.7352\!\pm\!0.1486$ & \cellcolor{scorebetter}$0.7782\!\pm\!0.0956$ & \cellcolor{scorebetter}$0.7563\!\pm\!0.1345$ & \cellcolor{scorebetter}$\mathbf{0.8324}\!\pm\!0.0926$\\
LUAD & $0.5381\!\pm\!0.0614$ & \cellcolor{scorebetter}$\mathbf{0.5795}\!\pm\!0.0276$ & \cellcolor{scorebetter}$0.5495\!\pm\!0.0214$ & \cellcolor{scoreworse}$0.5378\!\pm\!0.0327$\\
STAD & $0.5104\!\pm\!0.0188$ & \cellcolor{scorebetter}$0.5112\!\pm\!0.0209$ & \cellcolor{scoreworse}$0.4947\!\pm\!0.0497$ & \cellcolor{scorebetter}$\mathbf{0.5169}\!\pm\!0.0456$\\
UCEC & $0.6499\!\pm\!0.1086$ & \cellcolor{scorebetter}$\mathbf{0.6793}\!\pm\!0.1251$ & \cellcolor{scorebetter}$0.6672\!\pm\!0.1257$ & \cellcolor{scorebetter}$0.6525\!\pm\!0.1023$\\
\bottomrule\end{tabular}
\end{table}

\subsection{Spatial utility when all feature channels are retained}
\label{app:n3-joint}
This study tests teacher utility and distillation without projecting away feature channels. Separate $2304\to512\to1536$ and $2304\to512\to4608$ networks predict the regional mean and nonconstant components, giving 100 distillation fits across ten datasets and five realizations. Six task representations retain the same native base and use 8,448 coordinates, giving 300 task fits.

For this comparison, $Z=[b_i;0_{6144}]$ is native \five{} input with zero augmentation. $M$ appends four copies of the teacher regional mean, and $S$ appends the teacher mean and spatial components after decoding. Their neural predictions are $\widehat M$ and $\widehat S$. The input $I_S$ combines the trained mean prediction with the exact initial spatial head. These symbols identify complete native-plus-summary inputs in Table~\ref{tab:spatial-complete-arms}. They are distinct from the standalone teacher bags in Table~\ref{tab:teacher-targets}.

The three neural representations use the same trained mean head. Only the spatial output differs, through omission, the trained head or its initial state. Comparing trained and initial spatial outputs tests the contribution of distillation. Comparing trained outputs with omission tests whether supplying them helps beyond the predicted mean. Teacher spatial components improve the average primary score over the teacher mean in several datasets, with adverse directions in others. Neural spatial predictions also have mixed effects. These comparisons do not establish a general primary-score advantage for retaining spatial variation. Appendix~\ref{app:spatial-learning} develops the NLL findings using the same inputs.

\begin{table}[h]
\centering\footnotesize
\setlength{\tabcolsep}{2pt}
\caption{Task performance when all feature channels are retained. Every input contains the native base, with different teacher or neural summaries in the remaining coordinates. The reference is $Z=[b_i;0_{6144}]$. The upper and lower panels report primary scores and NLL separately.}
\label{tab:spatial-complete-arms}
\begin{tabular}{@{}lcccccc@{}}\toprule

\multicolumn{7}{c}{\textbf{Primary metric, higher is better}}\\\midrule
Dataset & \shortstack{Native \five{}\\$Z$} & \shortstack{Teacher mean\\$M$} & \shortstack{Teacher spatial\\$S$} & \shortstack{Predicted mean\\$\widehat M$} & \shortstack{Predicted spatial\\$\widehat S$} & \shortstack{Initial spatial\\$I_S$}\\\midrule
CAM16 & \shortstack{$0.8722$\\$\pm 0.0275$} & \cellcolor{scorebetter}\shortstack{$\mathbf{0.8955}$\\$\pm 0.0126$} & \cellcolor{scoreworse}\shortstack{$0.8653$\\$\pm 0.0389$} & \cellcolor{scoreworse}\shortstack{$0.8341$\\$\pm 0.0620$} & \cellcolor{scorebetter}\shortstack{$0.8835$\\$\pm 0.0272$} & \cellcolor{scoreworse}\shortstack{$0.8463$\\$\pm 0.0226$}\\
CAM17 & \shortstack{$\mathbf{0.5261}$\\$\pm 0.0418$} & \cellcolor{scoreworse}\shortstack{$0.5111$\\$\pm 0.0505$} & \cellcolor{scoreworse}\shortstack{$0.5242$\\$\pm 0.0357$} & \cellcolor{scoreworse}\shortstack{$0.4968$\\$\pm 0.0308$} & \cellcolor{scoreworse}\shortstack{$0.5042$\\$\pm 0.0407$} & \cellcolor{scoreworse}\shortstack{$0.5092$\\$\pm 0.0425$}\\
Private-CRC & \shortstack{$0.8284$\\$\pm 0.0179$} & \cellcolor{scoreworse}\shortstack{$0.8097$\\$\pm 0.0475$} & \cellcolor{scorebetter}\shortstack{$0.8313$\\$\pm 0.0318$} & \cellcolor{scoreworse}\shortstack{$0.8209$\\$\pm 0.0286$} & \cellcolor{scoreworse}\shortstack{$0.8258$\\$\pm 0.0411$} & \cellcolor{scorebetter}\shortstack{$\mathbf{0.8345}$\\$\pm 0.0274$}\\
PANDA & \shortstack{$0.8467$\\$\pm 0.0131$} & \cellcolor{scorebetter}\shortstack{$0.8566$\\$\pm 0.0069$} & \cellcolor{scorebetter}\shortstack{$\mathbf{0.8608}$\\$\pm 0.0089$} & \cellcolor{scoreworse}\shortstack{$0.8460$\\$\pm 0.0119$} & \cellcolor{scoreworse}\shortstack{$0.8446$\\$\pm 0.0139$} & \cellcolor{scoreworse}\shortstack{$0.8385$\\$\pm 0.0160$}\\
BRACS & \shortstack{$0.4276$\\$\pm 0.0577$} & \cellcolor{scoreworse}\shortstack{$0.3966$\\$\pm 0.0875$} & \cellcolor{scorebetter}\shortstack{$\mathbf{0.4414}$\\$\pm 0.0755$} & \cellcolor{scorebetter}\shortstack{$0.4398$\\$\pm 0.0827$} & \cellcolor{scoreworse}\shortstack{$0.4224$\\$\pm 0.0871$} & \cellcolor{scoreworse}\shortstack{$0.4147$\\$\pm 0.0525$}\\
KIRC & \shortstack{$0.6986$\\$\pm 0.0340$} & \cellcolor{scoreworse}\shortstack{$0.6945$\\$\pm 0.0614$} & \cellcolor{scoreworse}\shortstack{$0.6982$\\$\pm 0.0640$} & \cellcolor{scoreworse}\shortstack{$0.6976$\\$\pm 0.0242$} & \cellcolor{scoreworse}\shortstack{$0.6801$\\$\pm 0.0490$} & \cellcolor{scorebetter}\shortstack{$\mathbf{0.7123}$\\$\pm 0.0358$}\\
KIRP & \shortstack{$0.8073$\\$\pm 0.1143$} & \cellcolor{scoreworse}\shortstack{$0.7684$\\$\pm 0.1740$} & \cellcolor{scorebetter}\shortstack{$\mathbf{0.8546}$\\$\pm 0.0371$} & \cellcolor{scoreworse}\shortstack{$0.8007$\\$\pm 0.1032$} & \cellcolor{scorebetter}\shortstack{$0.8086$\\$\pm 0.0982$} & \cellcolor{scoreworse}\shortstack{$0.7732$\\$\pm 0.1654$}\\
LUAD & \shortstack{$0.5660$\\$\pm 0.0362$} & \cellcolor{scoreworse}\shortstack{$0.5423$\\$\pm 0.0707$} & \cellcolor{scoreworse}\shortstack{$0.5635$\\$\pm 0.0268$} & \cellcolor{scoreworse}\shortstack{$0.5440$\\$\pm 0.0359$} & \cellcolor{scoreworse}\shortstack{$0.5332$\\$\pm 0.0548$} & \cellcolor{scorebetter}\shortstack{$\mathbf{0.5729}$\\$\pm 0.0334$}\\
STAD & \shortstack{$0.5164$\\$\pm 0.0626$} & \cellcolor{scorebetter}\shortstack{$\mathbf{0.5279}$\\$\pm 0.0293$} & \cellcolor{scoreworse}\shortstack{$0.5037$\\$\pm 0.0597$} & \cellcolor{scoreworse}\shortstack{$0.5054$\\$\pm 0.0506$} & \cellcolor{scorebetter}\shortstack{$0.5193$\\$\pm 0.0430$} & \cellcolor{scorebetter}\shortstack{$0.5179$\\$\pm 0.0436$}\\
UCEC & \shortstack{$0.6218$\\$\pm 0.0866$} & \cellcolor{scorebetter}\shortstack{$0.6282$\\$\pm 0.0799$} & \cellcolor{scoreworse}\shortstack{$0.6198$\\$\pm 0.0800$} & \cellcolor{scorebetter}\shortstack{$0.6424$\\$\pm 0.1200$} & \cellcolor{scorebetter}\shortstack{$0.6378$\\$\pm 0.0750$} & \cellcolor{scorebetter}\shortstack{$\mathbf{0.6516}$\\$\pm 0.0658$}\\
\midrule
\multicolumn{7}{c}{\textbf{NLL, lower is better}}\\\midrule
Dataset & \shortstack{Native \five{}\\$Z$} & \shortstack{Teacher mean\\$M$} & \shortstack{Teacher spatial\\$S$} & \shortstack{Predicted mean\\$\widehat M$} & \shortstack{Predicted spatial\\$\widehat S$} & \shortstack{Initial spatial\\$I_S$}\\\midrule
CAM16 & \shortstack{$\mathbf{0.5407}$\\$\pm 0.1000$} & \cellcolor{scoreworse}\shortstack{$0.5841$\\$\pm 0.2263$} & \cellcolor{scoreworse}\shortstack{$0.6975$\\$\pm 0.2300$} & \cellcolor{scoreworse}\shortstack{$0.6398$\\$\pm 0.0860$} & \cellcolor{scoreworse}\shortstack{$0.6897$\\$\pm 0.2603$} & \cellcolor{scoreworse}\shortstack{$0.6895$\\$\pm 0.1239$}\\
CAM17 & \shortstack{$0.7237$\\$\pm 0.2901$} & \cellcolor{scoreworse}\shortstack{$0.7563$\\$\pm 0.2721$} & \cellcolor{scoreworse}\shortstack{$0.8016$\\$\pm 0.2640$} & \cellcolor{scoreworse}\shortstack{$1.1302$\\$\pm 0.5135$} & \cellcolor{scoreworse}\shortstack{$0.9913$\\$\pm 0.4640$} & \cellcolor{scorebetter}\shortstack{$\mathbf{0.7099}$\\$\pm 0.1403$}\\
Private-CRC & \shortstack{$0.6161$\\$\pm 0.1126$} & \cellcolor{scoreworse}\shortstack{$0.6279$\\$\pm 0.0905$} & \cellcolor{scorebetter}\shortstack{$0.5790$\\$\pm 0.0802$} & \cellcolor{scorebetter}\shortstack{$\mathbf{0.5591}$\\$\pm 0.0641$} & \cellcolor{scorebetter}\shortstack{$0.6064$\\$\pm 0.1424$} & \cellcolor{scorebetter}\shortstack{$0.5768$\\$\pm 0.1226$}\\
PANDA & \shortstack{$1.0584$\\$\pm 0.4364$} & \cellcolor{scorebetter}\shortstack{$0.9727$\\$\pm 0.1695$} & \cellcolor{scoreworse}\shortstack{$1.3085$\\$\pm 0.7090$} & \cellcolor{scorebetter}\shortstack{$0.9212$\\$\pm 0.0914$} & \cellcolor{scorebetter}\shortstack{$\mathbf{0.8832}$\\$\pm 0.0301$} & \cellcolor{scorebetter}\shortstack{$0.9300$\\$\pm 0.0329$}\\
BRACS & \shortstack{$2.1090$\\$\pm 0.5462$} & \cellcolor{scorebetter}\shortstack{$1.4396$\\$\pm 0.3561$} & \cellcolor{scorebetter}\shortstack{$2.0671$\\$\pm 0.5976$} & \cellcolor{scorebetter}\shortstack{$1.4789$\\$\pm 0.2412$} & \cellcolor{scorebetter}\shortstack{$\mathbf{1.4129}$\\$\pm 0.2799$} & \cellcolor{scorebetter}\shortstack{$1.7211$\\$\pm 0.5466$}\\
KIRC & \shortstack{$\mathbf{1.6016}$\\$\pm 0.1369$} & \cellcolor{scoreworse}\shortstack{$1.9291$\\$\pm 0.2473$} & \cellcolor{scoreworse}\shortstack{$1.8049$\\$\pm 0.4113$} & \cellcolor{scoreworse}\shortstack{$1.7670$\\$\pm 0.1869$} & \cellcolor{scoreworse}\shortstack{$1.7589$\\$\pm 0.2831$} & \cellcolor{scoreworse}\shortstack{$2.1189$\\$\pm 0.4656$}\\
KIRP & \shortstack{$1.2996$\\$\pm 0.1336$} & \cellcolor{scorebetter}\shortstack{$1.1313$\\$\pm 0.0762$} & \cellcolor{scorebetter}\shortstack{$\mathbf{1.0616}$\\$\pm 0.1776$} & \cellcolor{scorebetter}\shortstack{$1.2368$\\$\pm 0.1122$} & \cellcolor{scorebetter}\shortstack{$1.1533$\\$\pm 0.2994$} & \cellcolor{scoreworse}\shortstack{$1.3178$\\$\pm 0.2141$}\\
LUAD & \shortstack{$\mathbf{1.9082}$\\$\pm 0.2802$} & \cellcolor{scoreworse}\shortstack{$2.2922$\\$\pm 0.5052$} & \cellcolor{scoreworse}\shortstack{$2.3800$\\$\pm 0.5268$} & \cellcolor{scoreworse}\shortstack{$2.4893$\\$\pm 0.7290$} & \cellcolor{scoreworse}\shortstack{$1.9505$\\$\pm 0.3098$} & \cellcolor{scoreworse}\shortstack{$2.2323$\\$\pm 0.6362$}\\
STAD & \shortstack{$2.4946$\\$\pm 0.5711$} & \cellcolor{scoreworse}\shortstack{$2.6962$\\$\pm 0.5397$} & \cellcolor{scorebetter}\shortstack{$2.3705$\\$\pm 0.5052$} & \cellcolor{scorebetter}\shortstack{$\mathbf{2.3427}$\\$\pm 0.4329$} & \cellcolor{scoreworse}\shortstack{$2.5259$\\$\pm 0.4128$} & \cellcolor{scorebetter}\shortstack{$2.4605$\\$\pm 0.3984$}\\
UCEC & \shortstack{$1.0279$\\$\pm 0.1880$} & \cellcolor{scoreworse}\shortstack{$1.0393$\\$\pm 0.2126$} & \cellcolor{scorebetter}\shortstack{$1.0164$\\$\pm 0.1590$} & \cellcolor{scorebetter}\shortstack{$\mathbf{0.9815}$\\$\pm 0.0727$} & \cellcolor{scorebetter}\shortstack{$0.9907$\\$\pm 0.1317$} & \cellcolor{scorebetter}\shortstack{$1.0041$\\$\pm 0.1715$}\\
\bottomrule\end{tabular}
\end{table}

\paragraph{Reconstruction by the downstream predictors.}
To connect the task results to their Stage~1 training, Table~\ref{tab:spatial-complete-source} reports reconstruction by the same mean and spatial heads. Each trains and selects checkpoints using its own component MSE. Every mean head improves on the training target mean in the selection groups, while every spatial head has higher error relative to that constant. Both improve on their own initialization, and all choose step 512, the first eligible checkpoint. Stage~1 training can reduce spatial error without overcoming the constant reference.

These groups select the checkpoints, so their errors describe model selection rather than independent generalization. Independent component errors are unavailable for these heads. The projected diagnostic in Section~\ref{sec:transfer_results} evaluates different predictors and cannot supply that missing evidence.

\begin{table}[h]
\centering\footnotesize
\setlength{\tabcolsep}{3pt}
\caption{Reconstruction by the mean and spatial heads used downstream. MSE is evaluated on the groups used for checkpoint selection, with the constant equal to the training target mean for that component. Lower is better. Mean and spatial errors use their own scales.}
\label{tab:spatial-complete-source}
\begin{tabular}{@{}lcccc@{}}\toprule
& \multicolumn{2}{c}{Regional mean} & \multicolumn{2}{c}{Nonconstant components}\\
\cmidrule(lr){2-3}\cmidrule(l){4-5}
Excluded dataset & Predictor & Constant & Predictor & Constant\\\midrule
CAM16 & \cellcolor{scorebetter}$\mathbf{0.3538}\!\pm\!0.0079$ & $0.6935\!\pm\!0.0160$ & \cellcolor{scoreworse}$0.0730\!\pm\!0.0018$ & $\mathbf{0.0563}\!\pm\!0.0012$\\
CAM17 & \cellcolor{scorebetter}$\mathbf{0.3461}\!\pm\!0.0107$ & $0.6956\!\pm\!0.0149$ & \cellcolor{scoreworse}$0.0713\!\pm\!0.0018$ & $\mathbf{0.0553}\!\pm\!0.0012$\\
Private-CRC & \cellcolor{scorebetter}$\mathbf{0.3812}\!\pm\!0.0102$ & $0.7227\!\pm\!0.0231$ & \cellcolor{scoreworse}$0.0752\!\pm\!0.0017$ & $\mathbf{0.0555}\!\pm\!0.0008$\\
PANDA & \cellcolor{scorebetter}$\mathbf{0.3565}\!\pm\!0.0070$ & $0.7000\!\pm\!0.0179$ & \cellcolor{scoreworse}$0.0694\!\pm\!0.0019$ & $\mathbf{0.0508}\!\pm\!0.0017$\\
BRACS & \cellcolor{scorebetter}$\mathbf{0.3450}\!\pm\!0.0070$ & $0.6702\!\pm\!0.0130$ & \cellcolor{scoreworse}$0.0718\!\pm\!0.0020$ & $\mathbf{0.0556}\!\pm\!0.0015$\\
KIRC & \cellcolor{scorebetter}$\mathbf{0.3460}\!\pm\!0.0083$ & $0.6745\!\pm\!0.0065$ & \cellcolor{scoreworse}$0.0747\!\pm\!0.0020$ & $\mathbf{0.0575}\!\pm\!0.0014$\\
KIRP & \cellcolor{scorebetter}$\mathbf{0.3360}\!\pm\!0.0083$ & $0.6798\!\pm\!0.0175$ & \cellcolor{scoreworse}$0.0737\!\pm\!0.0023$ & $\mathbf{0.0569}\!\pm\!0.0015$\\
LUAD & \cellcolor{scorebetter}$\mathbf{0.3432}\!\pm\!0.0105$ & $0.7001\!\pm\!0.0201$ & \cellcolor{scoreworse}$0.0737\!\pm\!0.0021$ & $\mathbf{0.0571}\!\pm\!0.0017$\\
STAD & \cellcolor{scorebetter}$\mathbf{0.3312}\!\pm\!0.0067$ & $0.6879\!\pm\!0.0151$ & \cellcolor{scoreworse}$0.0723\!\pm\!0.0018$ & $\mathbf{0.0562}\!\pm\!0.0014$\\
UCEC & \cellcolor{scorebetter}$\mathbf{0.3288}\!\pm\!0.0106$ & $0.6676\!\pm\!0.0110$ & \cellcolor{scoreworse}$0.0729\!\pm\!0.0021$ & $\mathbf{0.0566}\!\pm\!0.0015$\\
\bottomrule\end{tabular}
\end{table}

\subsection{An alternative way to deliver spatial summaries}
\label{app:compression-all}
We also compare three summaries that retain every feature channel. The first decodes the low frequency components, the second repeats the regional mean four times, and the third decodes components after permuting the sixteen teacher positions. Concatenating the four 1,536-coordinate vectors with the 2,304-coordinate native base gives one 8,448-coordinate token per region. All three use the same normalization estimated from distillation training data, delivery scale, fields and task training within each realization, for 150 fits across ten datasets. A region's permutation is unchanged throughout the experiment and preserves its mean, but changes the retained frequency energy. This study has its own calibration on distillation training data, so comparison with Appendix~\ref{app:n3-joint} does not isolate a single design choice.

\begin{table}[h]
\centering\footnotesize
\setlength{\tabcolsep}{3pt}
\caption{Task performance with three teacher summaries at a common input width of 8,448. Every input retains the native base and appends four vectors from the same observed \twenty{} field. The repeated regional mean is the reference. Higher is better.}
\label{tab:compression-all}
\begin{tabular}{@{}llccc@{}}\toprule
Dataset & Metric & Repeated mean & Spatial summary & Permuted summary\\\midrule
CAM16 & AUC & $\mathbf{0.8820}\!\pm\!0.0287$ & \cellcolor{scoreworse}$0.8716\!\pm\!0.0170$ & \cellcolor{scoreworse}$0.8730\!\pm\!0.0563$\\
CAM17 & $F_1$ & $\mathbf{0.5141}\!\pm\!0.0307$ & \cellcolor{scoreworse}$0.5038\!\pm\!0.0433$ & \cellcolor{scoreworse}$0.5126\!\pm\!0.0409$\\
Private-CRC & $F_1$ & $\mathbf{0.8575}\!\pm\!0.0194$ & \cellcolor{scoreworse}$0.8491\!\pm\!0.0396$ & \cellcolor{scoreworse}$0.8455\!\pm\!0.0264$\\
PANDA & $\kappa$ & $0.8569\!\pm\!0.0082$ & \cellcolor{scorebetter}$\mathbf{0.8609}\!\pm\!0.0118$ & \cellcolor{scorebetter}$0.8607\!\pm\!0.0110$\\
BRACS & $F_1$ & $\mathbf{0.4319}\!\pm\!0.0873$ & \cellcolor{scoreworse}$0.4138\!\pm\!0.0936$ & \cellcolor{scoreworse}$0.4094\!\pm\!0.1125$\\
KIRC & C-index & $\mathbf{0.7163}\!\pm\!0.0448$ & \cellcolor{scoreworse}$0.6978\!\pm\!0.0569$ & \cellcolor{scoreworse}$0.7013\!\pm\!0.0607$\\
KIRP & C-index & $0.7748\!\pm\!0.1123$ & \cellcolor{scorebetter}$\mathbf{0.8447}\!\pm\!0.1206$ & \cellcolor{scorebetter}$0.8212\!\pm\!0.0797$\\
LUAD & C-index & $0.5636\!\pm\!0.0234$ & \cellcolor{scorebetter}$\mathbf{0.5822}\!\pm\!0.0307$ & \cellcolor{scorebetter}$0.5638\!\pm\!0.0433$\\
STAD & C-index & $0.5394\!\pm\!0.0307$ & \cellcolor{scoreworse}$0.5095\!\pm\!0.0417$ & \cellcolor{scorebetter}$\mathbf{0.5448}\!\pm\!0.0413$\\
UCEC & C-index & $0.6089\!\pm\!0.1076$ & \cellcolor{scorebetter}$\mathbf{0.6313}\!\pm\!0.1085$ & \cellcolor{scoreworse}$0.5952\!\pm\!0.1106$\\
\bottomrule\end{tabular}
\end{table}

Table~\ref{tab:compression-all} shows that spatial summaries improve average task performance in some datasets and reduce it in others. Permuting the teacher positions can also change that ordering. The results do not establish a common benefit from preserving spatial order. The frequency pattern motivates a compact spatial target, but the downstream comparison is still needed to determine whether its retained variation helps the task.

\subsection{Regional means across task models}
\label{app:mil-robustness}
To test whether regional mean utility persists across task models, we compare TransMIL~\citep{TransMIL}, gated ABMIL~\citep{ilse2018attention} and RRT-MIL~\citep{tang2024feature} on all ten datasets and DSMIL~\citep{li2021dsmil} on the five classification datasets. All receive frozen features. DSMIL combines instance and bag losses equally, uses mutually exclusive cross entropy and evaluates bag logits. Survival models train with discrete time NLL, then average slide risks within each case for C-index. The comparison covers complete training procedures whose differences extend beyond architecture.

The distillation realization and task split use seed 42. Hyperparameter selection uses task seed 1001 and four candidates that cross learning rate multipliers $\{0.5,1\}$ with weight decay multipliers $\{0.1,1\}$. The mean validation metric across representations selects one configuration per dataset and recipient. Evaluation then uses this configuration with five new task initializations, numbered 1002 through 1006. Initial weights match across input representations within each recipient and seed. Training lasts 15 epochs for CAM16, CAM17, Private-CRC and BRACS, and 20 for PANDA and survival. All inputs use the same 49,152-token limit and the dataset's validation selection rule.

Table~\ref{tab:recipient-means-absolute} uses the historical shared support, where every region has at least one aligned \twenty{} feature. Patch counts may vary, unlike the complete fields in Table~\ref{tab:teacher-targets}. Native and mean inputs are compared within each task model. All four models favor the regional mean on average for CAM16 and PANDA, while the direction varies elsewhere. These results show that the mean can help multiple recipients without identifying a consistently superior recipient for teacher information. Five task initializations on one split describe training variability and do not establish a general ranking of architectures.

\begin{table}[h]
\centering\footnotesize
\setlength{\tabcolsep}{2pt}
\caption{Native \five{} and regional \twenty{} mean inputs across downstream models. Each model uses its own native reference on the same tissue support. Higher is better. Five task initializations share one data split. DSMIL is evaluated only for classification and grading.}
\label{tab:recipient-means-absolute}
\begin{tabular}{@{}lcccccccc@{}}\toprule
& \multicolumn{2}{c}{TransMIL} & \multicolumn{2}{c}{Gated ABMIL} & \multicolumn{2}{c}{DSMIL} & \multicolumn{2}{c}{RRT-MIL}\\
\cmidrule(lr){2-3}\cmidrule(lr){4-5}\cmidrule(lr){6-7}\cmidrule(l){8-9}
Dataset & Native \five{} & Mean \twenty{} & Native \five{} & Mean \twenty{} & Native \five{} & Mean \twenty{} & Native \five{} & Mean \twenty{}\\\midrule
CAM16 & \shortstack{$0.9605$\\$\pm 0.0240$} & \cellcolor{scorebetter}\shortstack{$\mathbf{0.9642}$\\$\pm 0.0205$} & \shortstack{$0.9537$\\$\pm 0.0173$} & \cellcolor{scorebetter}\shortstack{$\mathbf{0.9705}$\\$\pm 0.0128$} & \shortstack{$0.9580$\\$\pm 0.0036$} & \cellcolor{scorebetter}\shortstack{$\mathbf{0.9759}$\\$\pm 0.0097$} & \shortstack{$0.9494$\\$\pm 0.0229$} & \cellcolor{scorebetter}\shortstack{$\mathbf{0.9682}$\\$\pm 0.0073$}\\
CAM17 & \shortstack{$\mathbf{0.6743}$\\$\pm 0.0552$} & \cellcolor{scoreworse}\shortstack{$0.6325$\\$\pm 0.0212$} & \shortstack{$\mathbf{0.6352}$\\$\pm 0.0259$} & \cellcolor{scoreworse}\shortstack{$0.5690$\\$\pm 0.0300$} & \shortstack{$\mathbf{0.5522}$\\$\pm 0.0537$} & \cellcolor{scoreworse}\shortstack{$0.5373$\\$\pm 0.0712$} & \shortstack{$\mathbf{0.6294}$\\$\pm 0.0313$} & \cellcolor{scoreworse}\shortstack{$0.5992$\\$\pm 0.0500$}\\
Private-CRC & \shortstack{$0.8010$\\$\pm 0.0317$} & \cellcolor{scorebetter}\shortstack{$\mathbf{0.8497}$\\$\pm 0.0134$} & \shortstack{$\mathbf{0.8288}$\\$\pm 0.0250$} & \cellcolor{scoreworse}\shortstack{$0.8093$\\$\pm 0.0262$} & \shortstack{$0.8397$\\$\pm 0.0114$} & \cellcolor{scorebetter}\shortstack{$\mathbf{0.8444}$\\$\pm 0.0126$} & \shortstack{$\mathbf{0.8443}$\\$\pm 0.0123$} & \cellcolor{scoreworse}\shortstack{$0.8394$\\$\pm 0.0347$}\\
PANDA & \shortstack{$0.8509$\\$\pm 0.0040$} & \cellcolor{scorebetter}\shortstack{$\mathbf{0.8630}$\\$\pm 0.0089$} & \shortstack{$0.8363$\\$\pm 0.0038$} & \cellcolor{scorebetter}\shortstack{$\mathbf{0.8518}$\\$\pm 0.0045$} & \shortstack{$0.8612$\\$\pm 0.0041$} & \cellcolor{scorebetter}\shortstack{$\mathbf{0.8833}$\\$\pm 0.0085$} & \shortstack{$0.8488$\\$\pm 0.0050$} & \cellcolor{scorebetter}\shortstack{$\mathbf{0.8832}$\\$\pm 0.0044$}\\
BRACS & \shortstack{$0.3709$\\$\pm 0.0319$} & \cellcolor{scorebetter}\shortstack{$\mathbf{0.4035}$\\$\pm 0.0246$} & \shortstack{$\mathbf{0.4141}$\\$\pm 0.0530$} & \cellcolor{scoreworse}\shortstack{$0.3753$\\$\pm 0.0155$} & \shortstack{$\mathbf{0.4278}$\\$\pm 0.0360$} & \cellcolor{scoreworse}\shortstack{$0.3490$\\$\pm 0.0322$} & \shortstack{$\mathbf{0.4603}$\\$\pm 0.0303$} & \cellcolor{scoreworse}\shortstack{$0.4116$\\$\pm 0.0282$}\\
KIRC & \shortstack{$\mathbf{0.6800}$\\$\pm 0.0147$} & \cellcolor{scoreworse}\shortstack{$0.6768$\\$\pm 0.0261$} & \shortstack{$\mathbf{0.6787}$\\$\pm 0.0221$} & \cellcolor{scoreworse}\shortstack{$0.6355$\\$\pm 0.0083$} & \multicolumn{2}{c}{Not evaluated} & \shortstack{$\mathbf{0.6988}$\\$\pm 0.0290$} & \cellcolor{scoreworse}\shortstack{$0.6846$\\$\pm 0.0322$}\\
KIRP & \shortstack{$\mathbf{0.8437}$\\$\pm 0.0710$} & \cellcolor{scoreworse}\shortstack{$0.7603$\\$\pm 0.0578$} & \shortstack{$\mathbf{0.8563}$\\$\pm 0.0165$} & \cellcolor{scoreworse}\shortstack{$0.8000$\\$\pm 0.0471$} & \multicolumn{2}{c}{Not evaluated} & \shortstack{$\mathbf{0.8429}$\\$\pm 0.0821$} & \cellcolor{scoreworse}\shortstack{$0.7889$\\$\pm 0.0490$}\\
LUAD & \shortstack{$0.5559$\\$\pm 0.0340$} & \cellcolor{scorebetter}\shortstack{$\mathbf{0.5768}$\\$\pm 0.0541$} & \shortstack{$\mathbf{0.5677}$\\$\pm 0.0190$} & \cellcolor{scoreworse}\shortstack{$0.5565$\\$\pm 0.0144$} & \multicolumn{2}{c}{Not evaluated} & \shortstack{$0.5741$\\$\pm 0.0729$} & \cellcolor{scorebetter}\shortstack{$\mathbf{0.5763}$\\$\pm 0.0369$}\\
STAD & \shortstack{$0.5241$\\$\pm 0.0185$} & \cellcolor{scorebetter}\shortstack{$\mathbf{0.5589}$\\$\pm 0.0355$} & \shortstack{$\mathbf{0.5379}$\\$\pm 0.0130$} & \cellcolor{scoreworse}\shortstack{$0.5318$\\$\pm 0.0152$} & \multicolumn{2}{c}{Not evaluated} & \shortstack{$0.5986$\\$\pm 0.0381$} & \cellcolor{scorebetter}\shortstack{$\mathbf{0.6008}$\\$\pm 0.0218$}\\
UCEC & \shortstack{$0.5976$\\$\pm 0.0476$} & \cellcolor{scorebetter}\shortstack{$\mathbf{0.6000}$\\$\pm 0.0466$} & \shortstack{$0.4955$\\$\pm 0.0251$} & \cellcolor{scorebetter}\shortstack{$\mathbf{0.4972}$\\$\pm 0.0053$} & \multicolumn{2}{c}{Not evaluated} & \shortstack{$\mathbf{0.5773}$\\$\pm 0.0521$} & \cellcolor{scoreworse}\shortstack{$0.5263$\\$\pm 0.0557$}\\
\bottomrule\end{tabular}
\end{table}

\clearpage
\section{Neural Prediction of Teacher Targets}
\label{app:prediction-evidence}
This appendix examines prediction of high magnification targets from low magnification inputs. We first evaluate raw regional means on independent test fields, then separate the errors of projected mean and spatial components. Additional studies examine error correction, checkpoint selection and local image features. Each target retains its own coordinates and evaluation population.
\subsection{Regional mean prediction on independent test fields}
\label{app:raw-test-prediction}
To test generalization of regional mean prediction, we reuse the selected direct and residual heads from Appendix~\ref{app:all-cohort-fidelity}, with their original transformations and constant reference, the training target mean. Each realization samples eight groups from the excluded dataset and eight from each of the nine distillation datasets, taking at most 32 complete fields per group. The resulting task-test samples contain 231 to 256 fields from the excluded dataset and 2,278 to 2,304 from distillation datasets. None of these groups contributed to the corresponding distillation training or validation.

Group identities follow the available records. CAM17 uses challenge groups, CAM16 and Private-CRC use slides, PANDA uses images, and the other datasets use patient or case identifiers where available. The analysis preserves separation and weighting at these recorded units. Slide or image separation alone does not establish patient independence.

We reverse the distillation transformations before computing MSE in the original 1,536-coordinate UNI2-h space defined by Equation~\ref{eq:raw-reconstruction-error}. Evaluation within the distillation datasets weights datasets equally, groups equally within a dataset and fields equally within a group. Evaluation on the excluded dataset uses the latter two levels. Two references make the errors interpretable. The native reference returns the region's original \five{} vector, while the constant returns the training target mean estimated from the original distillation training sample. Table~\ref{tab:raw-test-errors} reports their errors alongside all four trained predictors.

\begin{table}[htbp]
\centering\footnotesize
\setlength{\tabcolsep}{2pt}
\caption{Raw regional mean MSE on independent test fields from the excluded dataset. Native \five{} is the reference. Zero and random identify final-layer initialization. The constant $\bar m_{\mathrm{train}}$ is estimated from distillation training targets.}
\label{tab:raw-test-errors}
\begin{tabular}{@{}lcccccc@{}}\toprule
 & \shortstack{Native \five{}\\$x_i$} & \shortstack{Constant\\$\bar m_{\mathrm{train}}$} & \multicolumn{2}{c}{\shortstack{Direct\\$f_\theta(B)_i$}} & \multicolumn{2}{c}{\shortstack{Residual\\$x_i+r_\phi(B)_i$}}\\
\cmidrule(lr){4-5}\cmidrule(l){6-7}
Dataset & & & Zero & Random & Zero & Random\\\midrule
CAM16 & \shortstack{$0.1759$\\$\pm 0.0103$} & \cellcolor{scorebetter}\shortstack{$0.0875$\\$\pm 0.0079$} & \cellcolor{scorebetter}\shortstack{$\mathbf{0.0667}$\\$\pm 0.0067$} & \cellcolor{scorebetter}\shortstack{$0.0686$\\$\pm 0.0066$} & \cellcolor{scorebetter}\shortstack{$0.0935$\\$\pm 0.0061$} & \cellcolor{scorebetter}\shortstack{$0.0972$\\$\pm 0.0076$}\\
CAM17 & \shortstack{$0.1939$\\$\pm 0.0044$} & \cellcolor{scorebetter}\shortstack{$0.0912$\\$\pm 0.0038$} & \cellcolor{scorebetter}\shortstack{$\mathbf{0.0776}$\\$\pm 0.0049$} & \cellcolor{scorebetter}\shortstack{$0.0801$\\$\pm 0.0053$} & \cellcolor{scorebetter}\shortstack{$0.1064$\\$\pm 0.0048$} & \cellcolor{scorebetter}\shortstack{$0.1100$\\$\pm 0.0045$}\\
Private-CRC & \shortstack{$0.1876$\\$\pm 0.0162$} & \cellcolor{scorebetter}\shortstack{$\mathbf{0.1051}$\\$\pm 0.0034$} & \cellcolor{scorebetter}\shortstack{$0.1331$\\$\pm 0.0040$} & \cellcolor{scorebetter}\shortstack{$0.1551$\\$\pm 0.0085$} & \cellcolor{scorebetter}\shortstack{$0.1675$\\$\pm 0.0171$} & \cellcolor{scoreworse}\shortstack{$0.2017$\\$\pm 0.0235$}\\
PANDA & \shortstack{$0.1554$\\$\pm 0.0067$} & \cellcolor{scorebetter}\shortstack{$\mathbf{0.0837}$\\$\pm 0.0071$} & \cellcolor{scorebetter}\shortstack{$0.0988$\\$\pm 0.0120$} & \cellcolor{scorebetter}\shortstack{$0.1088$\\$\pm 0.0112$} & \cellcolor{scorebetter}\shortstack{$0.1231$\\$\pm 0.0107$} & \cellcolor{scorebetter}\shortstack{$0.1364$\\$\pm 0.0106$}\\
BRACS & \shortstack{$0.1882$\\$\pm 0.0094$} & \cellcolor{scorebetter}\shortstack{$0.1099$\\$\pm 0.0080$} & \cellcolor{scorebetter}\shortstack{$\mathbf{0.0936}$\\$\pm 0.0082$} & \cellcolor{scorebetter}\shortstack{$0.0941$\\$\pm 0.0089$} & \cellcolor{scorebetter}\shortstack{$0.1102$\\$\pm 0.0065$} & \cellcolor{scorebetter}\shortstack{$0.1175$\\$\pm 0.0074$}\\
KIRC & \shortstack{$0.1697$\\$\pm 0.0020$} & \cellcolor{scorebetter}\shortstack{$0.1218$\\$\pm 0.0045$} & \cellcolor{scorebetter}\shortstack{$\mathbf{0.1084}$\\$\pm 0.0089$} & \cellcolor{scorebetter}\shortstack{$0.1132$\\$\pm 0.0096$} & \cellcolor{scorebetter}\shortstack{$0.1182$\\$\pm 0.0099$} & \cellcolor{scorebetter}\shortstack{$0.1263$\\$\pm 0.0104$}\\
KIRP & \shortstack{$0.1549$\\$\pm 0.0074$} & \cellcolor{scorebetter}\shortstack{$0.1099$\\$\pm 0.0074$} & \cellcolor{scorebetter}\shortstack{$\mathbf{0.0968}$\\$\pm 0.0041$} & \cellcolor{scorebetter}\shortstack{$0.1023$\\$\pm 0.0068$} & \cellcolor{scorebetter}\shortstack{$0.1040$\\$\pm 0.0054$} & \cellcolor{scorebetter}\shortstack{$0.1148$\\$\pm 0.0099$}\\
LUAD & \shortstack{$0.2074$\\$\pm 0.0142$} & \cellcolor{scorebetter}\shortstack{$0.0954$\\$\pm 0.0090$} & \cellcolor{scorebetter}\shortstack{$\mathbf{0.0953}$\\$\pm 0.0104$} & \cellcolor{scorebetter}\shortstack{$0.1041$\\$\pm 0.0136$} & \cellcolor{scorebetter}\shortstack{$0.1333$\\$\pm 0.0131$} & \cellcolor{scorebetter}\shortstack{$0.1504$\\$\pm 0.0182$}\\
STAD & \shortstack{$0.1828$\\$\pm 0.0083$} & \cellcolor{scorebetter}\shortstack{$0.0927$\\$\pm 0.0107$} & \cellcolor{scorebetter}\shortstack{$\mathbf{0.0866}$\\$\pm 0.0123$} & \cellcolor{scorebetter}\shortstack{$0.0902$\\$\pm 0.0131$} & \cellcolor{scorebetter}\shortstack{$0.1080$\\$\pm 0.0117$} & \cellcolor{scorebetter}\shortstack{$0.1144$\\$\pm 0.0129$}\\
UCEC & \shortstack{$0.1857$\\$\pm 0.0093$} & \cellcolor{scorebetter}\shortstack{$\mathbf{0.1118}$\\$\pm 0.0119$} & \cellcolor{scorebetter}\shortstack{$0.1183$\\$\pm 0.0139$} & \cellcolor{scorebetter}\shortstack{$0.1294$\\$\pm 0.0155$} & \cellcolor{scorebetter}\shortstack{$0.1261$\\$\pm 0.0093$} & \cellcolor{scorebetter}\shortstack{$0.1386$\\$\pm 0.0080$}\\
\bottomrule\end{tabular}
\end{table}

All four predictors improve on the constant across the independent test samples from the distillation datasets, but generalization to excluded datasets varies. Direct prediction with zero initialization improves on the constant in seven dataset averages. Residual prediction can approximate the teacher target more accurately than the native feature while remaining less accurate than the constant. These references distinguish progress toward the teacher from prediction of differences between regions.

To test whether prediction depends on correspondence between input and target, we permute complete input tuples within each dataset and evaluation population while keeping targets in place. The native feature, slide context and residual skip move together. All four predictors have lower MSE with aligned inputs in every excluded-dataset realization, as summarized in Table~\ref{tab:raw-test-permutation}. The complete tuple changes, so the comparison does not separate local features from slide context. It evaluates trained heads rather than training on shuffled supervision.

\begin{table}[htbp]
\centering\footnotesize
\setlength{\tabcolsep}{2.5pt}
\caption{Raw regional mean MSE with aligned or permuted input tuples. The native vector, slide context and residual skip are permuted together. Permuted input is the reference within each prediction form.}
\label{tab:raw-test-permutation}
\begin{tabular}{@{}lcccc@{}}\toprule
\multicolumn{5}{c}{\textbf{Zero final-layer initialization}}\\
 & \multicolumn{2}{c}{Direct} & \multicolumn{2}{c}{Residual}\\
\cmidrule(lr){2-3}\cmidrule(l){4-5}
Dataset & Permuted & Aligned & Permuted & Aligned\\\midrule
CAM16 & $0.1205\!\pm\!0.0136$ & \cellcolor{scorebetter}$\mathbf{0.0667}\!\pm\!0.0067$ & $0.1652\!\pm\!0.0131$ & \cellcolor{scorebetter}$\mathbf{0.0935}\!\pm\!0.0061$\\
CAM17 & $0.1216\!\pm\!0.0037$ & \cellcolor{scorebetter}$\mathbf{0.0776}\!\pm\!0.0049$ & $0.1733\!\pm\!0.0039$ & \cellcolor{scorebetter}$\mathbf{0.1064}\!\pm\!0.0048$\\
Private-CRC & $0.1496\!\pm\!0.0041$ & \cellcolor{scorebetter}$\mathbf{0.1331}\!\pm\!0.0040$ & $0.2475\!\pm\!0.0130$ & \cellcolor{scorebetter}$\mathbf{0.1675}\!\pm\!0.0171$\\
PANDA & $0.1145\!\pm\!0.0105$ & \cellcolor{scorebetter}$\mathbf{0.0988}\!\pm\!0.0120$ & $0.1769\!\pm\!0.0123$ & \cellcolor{scorebetter}$\mathbf{0.1231}\!\pm\!0.0107$\\
BRACS & $0.1217\!\pm\!0.0084$ & \cellcolor{scorebetter}$\mathbf{0.0936}\!\pm\!0.0082$ & $0.1719\!\pm\!0.0088$ & \cellcolor{scorebetter}$\mathbf{0.1102}\!\pm\!0.0065$\\
KIRC & $0.1584\!\pm\!0.0029$ & \cellcolor{scorebetter}$\mathbf{0.1084}\!\pm\!0.0089$ & $0.2234\!\pm\!0.0051$ & \cellcolor{scorebetter}$\mathbf{0.1182}\!\pm\!0.0099$\\
KIRP & $0.1447\!\pm\!0.0080$ & \cellcolor{scorebetter}$\mathbf{0.0968}\!\pm\!0.0041$ & $0.1994\!\pm\!0.0064$ & \cellcolor{scorebetter}$\mathbf{0.1040}\!\pm\!0.0054$\\
LUAD & $0.1296\!\pm\!0.0079$ & \cellcolor{scorebetter}$\mathbf{0.0953}\!\pm\!0.0104$ & $0.2090\!\pm\!0.0138$ & \cellcolor{scorebetter}$\mathbf{0.1333}\!\pm\!0.0131$\\
STAD & $0.1226\!\pm\!0.0114$ & \cellcolor{scorebetter}$\mathbf{0.0866}\!\pm\!0.0123$ & $0.1883\!\pm\!0.0185$ & \cellcolor{scorebetter}$\mathbf{0.1080}\!\pm\!0.0117$\\
UCEC & $0.1536\!\pm\!0.0133$ & \cellcolor{scorebetter}$\mathbf{0.1183}\!\pm\!0.0139$ & $0.2152\!\pm\!0.0155$ & \cellcolor{scorebetter}$\mathbf{0.1261}\!\pm\!0.0093$\\
\midrule
\multicolumn{5}{c}{\textbf{Random final-layer initialization}}\\
 & \multicolumn{2}{c}{Direct} & \multicolumn{2}{c}{Residual}\\
\cmidrule(lr){2-3}\cmidrule(l){4-5}
Dataset & Permuted & Aligned & Permuted & Aligned\\\midrule
CAM16 & $0.1229\!\pm\!0.0145$ & \cellcolor{scorebetter}$\mathbf{0.0686}\!\pm\!0.0066$ & $0.1682\!\pm\!0.0149$ & \cellcolor{scorebetter}$\mathbf{0.0972}\!\pm\!0.0076$\\
CAM17 & $0.1246\!\pm\!0.0051$ & \cellcolor{scorebetter}$\mathbf{0.0801}\!\pm\!0.0053$ & $0.1764\!\pm\!0.0045$ & \cellcolor{scorebetter}$\mathbf{0.1100}\!\pm\!0.0045$\\
Private-CRC & $0.1732\!\pm\!0.0083$ & \cellcolor{scorebetter}$\mathbf{0.1551}\!\pm\!0.0085$ & $0.2799\!\pm\!0.0161$ & \cellcolor{scorebetter}$\mathbf{0.2017}\!\pm\!0.0235$\\
PANDA & $0.1256\!\pm\!0.0093$ & \cellcolor{scorebetter}$\mathbf{0.1088}\!\pm\!0.0112$ & $0.1905\!\pm\!0.0115$ & \cellcolor{scorebetter}$\mathbf{0.1364}\!\pm\!0.0106$\\
BRACS & $0.1229\!\pm\!0.0097$ & \cellcolor{scorebetter}$\mathbf{0.0941}\!\pm\!0.0089$ & $0.1783\!\pm\!0.0097$ & \cellcolor{scorebetter}$\mathbf{0.1175}\!\pm\!0.0074$\\
KIRC & $0.1660\!\pm\!0.0044$ & \cellcolor{scorebetter}$\mathbf{0.1132}\!\pm\!0.0096$ & $0.2313\!\pm\!0.0055$ & \cellcolor{scorebetter}$\mathbf{0.1263}\!\pm\!0.0104$\\
KIRP & $0.1509\!\pm\!0.0090$ & \cellcolor{scorebetter}$\mathbf{0.1023}\!\pm\!0.0068$ & $0.2095\!\pm\!0.0081$ & \cellcolor{scorebetter}$\mathbf{0.1148}\!\pm\!0.0099$\\
LUAD & $0.1391\!\pm\!0.0106$ & \cellcolor{scorebetter}$\mathbf{0.1041}\!\pm\!0.0136$ & $0.2244\!\pm\!0.0172$ & \cellcolor{scorebetter}$\mathbf{0.1504}\!\pm\!0.0182$\\
STAD & $0.1271\!\pm\!0.0124$ & \cellcolor{scorebetter}$\mathbf{0.0902}\!\pm\!0.0131$ & $0.1939\!\pm\!0.0187$ & \cellcolor{scorebetter}$\mathbf{0.1144}\!\pm\!0.0129$\\
UCEC & $0.1664\!\pm\!0.0152$ & \cellcolor{scorebetter}$\mathbf{0.1294}\!\pm\!0.0155$ & $0.2263\!\pm\!0.0143$ & \cellcolor{scorebetter}$\mathbf{0.1386}\!\pm\!0.0080$\\
\bottomrule\end{tabular}
\end{table}

Direct prediction has lower MSE than residual prediction in every dataset average under both initializations. Proposition~\ref{prop:conditional-mean} gives the same optimum for unrestricted direct and residual functions. The empirical difference characterizes the finite networks and their optimization rather than a different target optimum.

\subsection{Feature distributions before and after distillation}
\label{app:feature-distribution}
The task scores in Tables~\ref{tab:joint-teacher-main} and~\ref{tab:teacher-targets} compare observed teacher representations through separately trained slide models. Their inputs have no before and after distillation state. To examine learning separately, we compare native \five{} features, observed \twenty{} regional means and the saved initial and trained regional-mean predictors on the independent fields from Appendix~\ref{app:raw-test-prediction}. The same fields, targets and training transformations are used before and after training. Each excluded dataset contributes five realizations with 231 to 256 fields from eight recorded groups per realization. Their total is 12,728 field appearances, with possible overlap across realizations. No task labels enter these measurements and no model is retrained.

For each evaluation field, $m_i$ is the regional teacher mean from Lemma~\ref{lem:shared-output-mean}, and $q_i$ is its prediction. Both use the original $d=1536$ UNI2-h coordinates. Weights $w_i$ give each recorded group equal mass and each field equal weight within its group. They sum to one, giving $\bar m=\sum_iw_im_i$ and $\bar q=\sum_iw_iq_i$. We compare target MSE and the cosine similarity of corresponding vectors. For positive target variance, the centered variance ratio is
\begin{equation}
 v=\frac{\sum_iw_i\|q_i-\bar q\|_2^2}{\sum_iw_i\|m_i-\bar m\|_2^2}.
\label{eq:feature-variance-ratio}
\end{equation}
A ratio of one indicates equal total centered variation, without requiring the distributions or individual predictions to agree. We compute the Pearson correlation between $\|q_i-q_j\|_2$ and $\|m_i-m_j\|_2$ over $i<j$, weighted proportionally to $w_iw_j$. It describes relative distances between regions without requiring identical feature axes, following the representational similarity approach \citep{kriegeskorte2008rsa}. Constant outputs have no defined correlation of pairwise distances. A zero output also has no defined cosine similarity. These cases remain undefined rather than receiving an artificial score.

Raw MSE depends on vector magnitude as well as direction. For nonzero $q_i$ and $m_i$, the Euclidean inner product gives
\begin{equation}
\|q_i-m_i\|_2^2=\|q_i\|_2^2+\|m_i\|_2^2-2\|q_i\|_2\|m_i\|_2\cos(q_i,m_i).
\label{eq:feature-norm-angle}
\end{equation}
Cosine similarity alone does not determine reconstruction error in these unnormalized coordinates. To separate agreement in the common feature average from agreement between regions, we also compute centered MSE. With the averages and weights defined above,
\begin{equation}
\frac1d\sum_iw_i\|q_i-m_i\|_2^2
=\frac1d\|\bar q-\bar m\|_2^2
+\frac1d\sum_iw_i\|(q_i-\bar q)-(m_i-\bar m)\|_2^2.
\label{eq:feature-centered-error}
\end{equation}
The cross term vanishes because each centered representation has weighted mean zero. The first term measures average displacement and the second measures mismatch after each representation's average is removed. Centering uses evaluation vectors only to describe their errors and does not alter the predictor. These quantities distinguish common feature structure from reconstruction of regional differences, without attributing small raw errors to uniformly similar pathology patches.

\begin{table}[htbp]
\centering\footnotesize
\setlength{\tabcolsep}{2pt}
\caption{Feature matching before and after distillation in the original 1,536 coordinates. Both predictors use random final-layer initialization. Native \five{} is the reference for MSE and correlation of pairwise distances. Variance ratios are descriptive, with teacher variation equal to one.}
\label{tab:feature-distribution}
\begin{tabular}{@{}lccccc@{}}\toprule
\multicolumn{6}{c}{\textbf{MSE $\downarrow$}}\\
 & Native \five{} & \multicolumn{2}{c}{Direct} & \multicolumn{2}{c}{Residual}\\
\cmidrule(lr){3-4}\cmidrule(l){5-6}
Dataset & $x_i$ & Initial & Trained & Initial & Trained\\\midrule
CAM16 & $0.1759\!\pm\!0.0103$ & \cellcolor{scorebetter}$0.1206\!\pm\!0.0108$ & \cellcolor{scorebetter}$\mathbf{0.0686}\!\pm\!0.0066$ & \cellcolor{scoreworse}$0.1998\!\pm\!0.0116$ & \cellcolor{scorebetter}$0.0972\!\pm\!0.0076$\\
CAM17 & $0.1939\!\pm\!0.0044$ & \cellcolor{scorebetter}$0.1231\!\pm\!0.0032$ & \cellcolor{scorebetter}$\mathbf{0.0801}\!\pm\!0.0053$ & \cellcolor{scoreworse}$0.2183\!\pm\!0.0045$ & \cellcolor{scorebetter}$0.1100\!\pm\!0.0045$\\
Private-CRC & $0.1876\!\pm\!0.0162$ & \cellcolor{scorebetter}$0.1574\!\pm\!0.0042$ & \cellcolor{scorebetter}$\mathbf{0.1551}\!\pm\!0.0085$ & \cellcolor{scoreworse}$0.2322\!\pm\!0.0128$ & \cellcolor{scoreworse}$0.2017\!\pm\!0.0235$\\
PANDA & $0.1554\!\pm\!0.0067$ & \cellcolor{scorebetter}$0.1183\!\pm\!0.0069$ & \cellcolor{scorebetter}$\mathbf{0.1088}\!\pm\!0.0112$ & \cellcolor{scoreworse}$0.1855\!\pm\!0.0056$ & \cellcolor{scorebetter}$0.1364\!\pm\!0.0106$\\
BRACS & $0.1882\!\pm\!0.0094$ & \cellcolor{scorebetter}$0.1412\!\pm\!0.0096$ & \cellcolor{scorebetter}$\mathbf{0.0941}\!\pm\!0.0089$ & \cellcolor{scoreworse}$0.2111\!\pm\!0.0111$ & \cellcolor{scorebetter}$0.1174\!\pm\!0.0074$\\
KIRC & $0.1697\!\pm\!0.0020$ & \cellcolor{scorebetter}$0.1693\!\pm\!0.0079$ & \cellcolor{scorebetter}$\mathbf{0.1132}\!\pm\!0.0096$ & \cellcolor{scoreworse}$0.1992\!\pm\!0.0033$ & \cellcolor{scorebetter}$0.1263\!\pm\!0.0104$\\
KIRP & $0.1549\!\pm\!0.0074$ & \cellcolor{scorebetter}$0.1486\!\pm\!0.0082$ & \cellcolor{scorebetter}$\mathbf{0.1023}\!\pm\!0.0068$ & \cellcolor{scoreworse}$0.1806\!\pm\!0.0076$ & \cellcolor{scorebetter}$0.1148\!\pm\!0.0099$\\
LUAD & $0.2074\!\pm\!0.0142$ & \cellcolor{scorebetter}$0.1431\!\pm\!0.0134$ & \cellcolor{scorebetter}$\mathbf{0.1041}\!\pm\!0.0136$ & \cellcolor{scoreworse}$0.2398\!\pm\!0.0164$ & \cellcolor{scorebetter}$0.1504\!\pm\!0.0182$\\
STAD & $0.1828\!\pm\!0.0083$ & \cellcolor{scorebetter}$0.1279\!\pm\!0.0104$ & \cellcolor{scorebetter}$\mathbf{0.0902}\!\pm\!0.0131$ & \cellcolor{scoreworse}$0.2047\!\pm\!0.0095$ & \cellcolor{scorebetter}$0.1144\!\pm\!0.0129$\\
UCEC & $0.1857\!\pm\!0.0093$ & \cellcolor{scorebetter}$0.1527\!\pm\!0.0136$ & \cellcolor{scorebetter}$\mathbf{0.1294}\!\pm\!0.0155$ & \cellcolor{scoreworse}$0.2136\!\pm\!0.0110$ & \cellcolor{scorebetter}$0.1386\!\pm\!0.0080$\\
\midrule
\multicolumn{6}{c}{\textbf{Centered variance ratio}}\\
 & Native \five{} & \multicolumn{2}{c}{Direct} & \multicolumn{2}{c}{Residual}\\
\cmidrule(lr){3-4}\cmidrule(l){5-6}
Dataset & $x_i$ & Initial & Trained & Initial & Trained\\\midrule
CAM16 & $2.0689\!\pm\!0.0950$ & $0.2571\!\pm\!0.0160$ & $0.7269\!\pm\!0.0635$ & $2.3289\!\pm\!0.1013$ & $1.4388\!\pm\!0.0469$\\
CAM17 & $1.7259\!\pm\!0.0926$ & $0.2260\!\pm\!0.0156$ & $0.5245\!\pm\!0.0261$ & $1.9529\!\pm\!0.1095$ & $1.1984\!\pm\!0.0582$\\
Private-CRC & $2.6330\!\pm\!0.1512$ & $0.4104\!\pm\!0.0326$ & $0.7964\!\pm\!0.0392$ & $3.0351\!\pm\!0.1435$ & $2.5611\!\pm\!0.1402$\\
PANDA & $1.9803\!\pm\!0.1472$ & $0.2748\!\pm\!0.0320$ & $0.5430\!\pm\!0.0589$ & $2.2597\!\pm\!0.1771$ & $1.7437\!\pm\!0.0949$\\
BRACS & $1.7239\!\pm\!0.1369$ & $0.2342\!\pm\!0.0217$ & $0.5302\!\pm\!0.0356$ & $1.9601\!\pm\!0.1590$ & $1.3829\!\pm\!0.0799$\\
KIRC & $1.4746\!\pm\!0.0529$ & $0.2098\!\pm\!0.0239$ & $0.5387\!\pm\!0.0494$ & $1.6791\!\pm\!0.0530$ & $1.2758\!\pm\!0.0789$\\
KIRP & $1.4454\!\pm\!0.1226$ & $0.1988\!\pm\!0.0172$ & $0.5414\!\pm\!0.0352$ & $1.6430\!\pm\!0.1403$ & $1.2545\!\pm\!0.1020$\\
LUAD & $1.9351\!\pm\!0.1296$ & $0.2836\!\pm\!0.0240$ & $0.6153\!\pm\!0.0695$ & $2.2214\!\pm\!0.1513$ & $1.7643\!\pm\!0.0939$\\
STAD & $1.6633\!\pm\!0.1838$ & $0.2057\!\pm\!0.0280$ & $0.4802\!\pm\!0.0593$ & $1.8749\!\pm\!0.2114$ & $1.3153\!\pm\!0.1212$\\
UCEC & $1.5191\!\pm\!0.0819$ & $0.2119\!\pm\!0.0167$ & $0.5601\!\pm\!0.0534$ & $1.7307\!\pm\!0.0967$ & $1.2919\!\pm\!0.0684$\\
\midrule
\multicolumn{6}{c}{\textbf{Pearson correlation of pairwise distances $\uparrow$}}\\
 & Native \five{} & \multicolumn{2}{c}{Direct} & \multicolumn{2}{c}{Residual}\\
\cmidrule(lr){3-4}\cmidrule(l){5-6}
Dataset & $x_i$ & Initial & Trained & Initial & Trained\\\midrule
CAM16 & $0.7720\!\pm\!0.0334$ & \cellcolor{scoreworse}$0.7717\!\pm\!0.0324$ & \cellcolor{scorebetter}$\mathbf{0.8889}\!\pm\!0.0248$ & \cellcolor{scorebetter}$0.7786\!\pm\!0.0325$ & \cellcolor{scorebetter}$0.8477\!\pm\!0.0293$\\
CAM17 & $0.5915\!\pm\!0.0561$ & \cellcolor{scorebetter}$0.6104\!\pm\!0.0336$ & \cellcolor{scorebetter}$\mathbf{0.7572}\!\pm\!0.0346$ & \cellcolor{scorebetter}$0.5998\!\pm\!0.0557$ & \cellcolor{scorebetter}$0.6435\!\pm\!0.0532$\\
Private-CRC & $0.6855\!\pm\!0.0434$ & \cellcolor{scoreworse}$0.6244\!\pm\!0.0468$ & \cellcolor{scorebetter}$\mathbf{0.7843}\!\pm\!0.0175$ & \cellcolor{scorebetter}$0.6870\!\pm\!0.0472$ & \cellcolor{scorebetter}$0.6915\!\pm\!0.0404$\\
PANDA & $0.5973\!\pm\!0.0506$ & \cellcolor{scorebetter}$0.6179\!\pm\!0.0388$ & \cellcolor{scorebetter}$\mathbf{0.7768}\!\pm\!0.0231$ & \cellcolor{scorebetter}$0.6062\!\pm\!0.0498$ & \cellcolor{scorebetter}$0.6569\!\pm\!0.0266$\\
BRACS & $0.6206\!\pm\!0.0424$ & \cellcolor{scoreworse}$0.6128\!\pm\!0.0469$ & \cellcolor{scorebetter}$\mathbf{0.7640}\!\pm\!0.0375$ & \cellcolor{scorebetter}$0.6283\!\pm\!0.0419$ & \cellcolor{scorebetter}$0.6827\!\pm\!0.0523$\\
KIRC & $0.7043\!\pm\!0.0240$ & \cellcolor{scoreworse}$0.6879\!\pm\!0.0407$ & \cellcolor{scorebetter}$\mathbf{0.8140}\!\pm\!0.0334$ & \cellcolor{scorebetter}$0.7100\!\pm\!0.0212$ & \cellcolor{scorebetter}$0.7685\!\pm\!0.0378$\\
KIRP & $0.6957\!\pm\!0.0312$ & \cellcolor{scoreworse}$0.6713\!\pm\!0.0317$ & \cellcolor{scorebetter}$\mathbf{0.7919}\!\pm\!0.0294$ & \cellcolor{scorebetter}$0.7009\!\pm\!0.0314$ & \cellcolor{scorebetter}$0.7615\!\pm\!0.0178$\\
LUAD & $0.6741\!\pm\!0.0377$ & \cellcolor{scoreworse}$0.6489\!\pm\!0.0351$ & \cellcolor{scorebetter}$\mathbf{0.7889}\!\pm\!0.0456$ & \cellcolor{scorebetter}$0.6804\!\pm\!0.0338$ & \cellcolor{scorebetter}$0.7068\!\pm\!0.0172$\\
STAD & $0.6911\!\pm\!0.0369$ & \cellcolor{scoreworse}$0.6142\!\pm\!0.0590$ & \cellcolor{scorebetter}$0.7365\!\pm\!0.0318$ & \cellcolor{scorebetter}$0.6918\!\pm\!0.0371$ & \cellcolor{scorebetter}$\mathbf{0.7446}\!\pm\!0.0508$\\
UCEC & $0.7244\!\pm\!0.0301$ & \cellcolor{scoreworse}$0.6694\!\pm\!0.0552$ & \cellcolor{scorebetter}$\mathbf{0.7859}\!\pm\!0.0676$ & \cellcolor{scorebetter}$0.7259\!\pm\!0.0386$ & \cellcolor{scorebetter}$0.7627\!\pm\!0.0623$\\
\bottomrule\end{tabular}
\end{table}

Table~\ref{tab:feature-distribution} shows lower MSE and higher correlation of pairwise distances after distillation for both prediction forms in every dataset average. Their spread changes in opposite directions. Direct variance increases and residual variance decreases, while trained direct predictions remain less variable than teacher means and trained residual predictions remain more variable. Smaller direct spread relative to the teacher does not imply contraction from initialization. The two forms exhibit distinct learning behavior despite sharing a target and unrestricted optimum.

Zero final-layer initialization provides a distinct reference. The initial direct output is identically zero, whereas the initial residual prediction equals the native feature exactly. Training lowers direct MSE in eight dataset means, with higher errors on Private-CRC and PANDA. Residual MSE falls in every dataset mean. Table~\ref{tab:feature-distribution-zero} shows that the lower direct and higher residual variance relative to the teacher persist under this initialization. The distinction prevents improvement over a random initial network from being mistaken for improvement over every reference.

\begin{table}[htbp]
\centering\footnotesize
\setlength{\tabcolsep}{4pt}
\caption{\textbf{Variation and target geometry after training from a zero output layer.} Centered variance is divided by teacher variance, with one indicating equal total variation rather than an optimum. Pearson correlation compares pairwise distances between regions in the output and target spaces. Native \five{} is its reference. The five realizations use independent test fields from the excluded dataset, whereas the feature maps illustrate one slide and one realization.}
\label{tab:feature-distribution-zero}
\begin{tabular}{@{}lccccc@{}}\toprule
 & \multicolumn{2}{c}{Centered variance ratio} & \multicolumn{3}{c}{\shortstack{Pearson correlation\\of pairwise distances $\uparrow$}}\\
\cmidrule(lr){2-3}\cmidrule(l){4-6}
Dataset & Direct & Residual & Native \five{} & Direct & Residual\\\midrule
CAM16 & \shortstack{$0.6997$\\$\pm 0.0473$} & \shortstack{$1.4150$\\$\pm 0.0391$} & \shortstack{$0.7720$\\$\pm 0.0334$} & \cellcolor{scorebetter}\shortstack{$\mathbf{0.8896}$\\$\pm 0.0140$} & \cellcolor{scorebetter}\shortstack{$0.8530$\\$\pm 0.0245$}\\
CAM17 & \shortstack{$0.4978$\\$\pm 0.0326$} & \shortstack{$1.1662$\\$\pm 0.0634$} & \shortstack{$0.5915$\\$\pm 0.0561$} & \cellcolor{scorebetter}\shortstack{$\mathbf{0.7465}$\\$\pm 0.0289$} & \cellcolor{scorebetter}\shortstack{$0.6338$\\$\pm 0.0564$}\\
Private-CRC & \shortstack{$0.6253$\\$\pm 0.0463$} & \shortstack{$2.3830$\\$\pm 0.1137$} & \shortstack{$0.6855$\\$\pm 0.0434$} & \cellcolor{scorebetter}\shortstack{$\mathbf{0.8169}$\\$\pm 0.0282$} & \cellcolor{scorebetter}\shortstack{$0.6917$\\$\pm 0.0423$}\\
PANDA & \shortstack{$0.4742$\\$\pm 0.0388$} & \shortstack{$1.6700$\\$\pm 0.1179$} & \shortstack{$0.5973$\\$\pm 0.0506$} & \cellcolor{scorebetter}\shortstack{$\mathbf{0.7908}$\\$\pm 0.0294$} & \cellcolor{scorebetter}\shortstack{$0.6611$\\$\pm 0.0397$}\\
BRACS & \shortstack{$0.4876$\\$\pm 0.0291$} & \shortstack{$1.3241$\\$\pm 0.0862$} & \shortstack{$0.6206$\\$\pm 0.0424$} & \cellcolor{scorebetter}\shortstack{$\mathbf{0.7543}$\\$\pm 0.0411$} & \cellcolor{scorebetter}\shortstack{$0.6962$\\$\pm 0.0524$}\\
\midrule
KIRC & \shortstack{$0.4748$\\$\pm 0.0331$} & \shortstack{$1.2216$\\$\pm 0.0627$} & \shortstack{$0.7043$\\$\pm 0.0240$} & \cellcolor{scorebetter}\shortstack{$\mathbf{0.8030}$\\$\pm 0.0316$} & \cellcolor{scorebetter}\shortstack{$0.7727$\\$\pm 0.0355$}\\
KIRP & \shortstack{$0.4956$\\$\pm 0.0283$} & \shortstack{$1.1852$\\$\pm 0.0838$} & \shortstack{$0.6957$\\$\pm 0.0312$} & \cellcolor{scorebetter}\shortstack{$\mathbf{0.7772}$\\$\pm 0.0448$} & \cellcolor{scorebetter}\shortstack{$0.7557$\\$\pm 0.0219$}\\
LUAD & \shortstack{$0.5454$\\$\pm 0.0709$} & \shortstack{$1.6263$\\$\pm 0.1022$} & \shortstack{$0.6741$\\$\pm 0.0377$} & \cellcolor{scorebetter}\shortstack{$\mathbf{0.7863}$\\$\pm 0.0389$} & \cellcolor{scorebetter}\shortstack{$0.7233$\\$\pm 0.0286$}\\
STAD & \shortstack{$0.4423$\\$\pm 0.0677$} & \shortstack{$1.2605$\\$\pm 0.1082$} & \shortstack{$0.6911$\\$\pm 0.0369$} & \cellcolor{scorebetter}\shortstack{$0.7205$\\$\pm 0.0439$} & \cellcolor{scorebetter}\shortstack{$\mathbf{0.7589}$\\$\pm 0.0390$}\\
UCEC & \shortstack{$0.4612$\\$\pm 0.0404$} & \shortstack{$1.2071$\\$\pm 0.0569$} & \shortstack{$0.7244$\\$\pm 0.0301$} & \cellcolor{scorebetter}\shortstack{$0.7772$\\$\pm 0.0645$} & \cellcolor{scorebetter}\shortstack{$\mathbf{0.7799}$\\$\pm 0.0497$}\\
\bottomrule\end{tabular}
\end{table}

The native feature provides an unchanged reference in the same coordinates. Trained direct prediction improves MSE and correlation of pairwise distances relative to native input in every dataset average under either initialization. Residual prediction lowers MSE throughout with zero initialization and in all datasets except Private-CRC with random initialization. These comparisons establish improvements relative to native input rather than superiority to every predictor. Table~\ref{tab:raw-test-errors} retains the constant alongside all four trained predictors.

Under Proposition~\ref{prop:conditional-mean}, reduced variation relative to teacher targets is compatible with conditional mean prediction but does not distinguish conditional target variation from estimation error. The direct outputs retain substantial variation and positive correlations of pairwise distances, while residual outputs remain more variable than the targets. Together with the permutation controls in Table~\ref{tab:raw-test-permutation}, these measurements establish dependence on the input rather than a common constant prediction. Theorem~\ref{thm:transfer-risk} asks a different question about information relevant to the task. It does not identify a feature reconstruction error with representation loss or teacher access. Section~\ref{sec:use_results} evaluates the downstream benefit separately.

This feature analysis uses CPU float32 inference. The raw reconstruction evaluation above uses GPU bfloat16 autocast, so its numerical values remain a separate evaluation. All measurements describe the saved predictors on sampled fields. The patch encoders remain frozen.

\subsection{Teacher correspondence in a common feature map}
\label{app:teacher-feature-maps}
The maps illustrate the regional mean prediction study from Section~\ref{sec:experiments}. The teacher target $m_i$ is the arithmetic mean of sixteen aligned \twenty{} features, and the student input is the native \five{} feature together with its slide context. Direct diamonds show the target estimate $f_\theta(B)_i$. Residual diamonds show $x_i+r_\phi(B)_i$, where the correction $r_\phi(B)_i$ learns the teacher-minus-native difference $m_i-x_i$. Thus both plotted outputs estimate the teacher mean. Neither the correction alone nor a concatenation of native and output features is plotted. The networks use trained seed-42 parameters with zero final-layer initialization. This condition matches Table~\ref{tab:regional-mean-reconstruction}. It is not selected because zero initialization performs best on every displayed slide. The patch encoders remain frozen.

\paragraph{Reference construction and sampling.}
Within each dataset, we first sample 1,000 complete test regions by cycling through randomly ordered groups and drawing regions without replacement within each group. The seed combines 20260923, the dataset index and 42. The same identities supply native features, sixteen individual teacher features and the regional teacher means. Each of these three populations is clustered separately in the original 1,536 coordinates using Euclidean K-means. Five initializations, at most 200 iterations and tolerance $10^{-4}$ are used. Candidate cluster counts range from two to eight, and the largest silhouette score on a fixed subset of at most 768 vectors determines the count. These partitions organize the display and do not identify tissue classes.

Each population contributes 800 reference landmarks, distributed across its clusters in proportion to their size and including an actual observation nearest each cluster center. A common 50-component principal component analysis (PCA) transform reduces the 2,400 landmarks without whitening. A t-SNE map then uses Euclidean affinities, perplexity 40, seed 42, PCA initialization, 250 early iterations with exaggeration 12 and 500 subsequent iterations with exaggeration one \citep{vandermaaten2008visualizing}. The remaining vectors are placed against these landmarks with openTSNE \citep{JSSv109i03}, using median initialization from 25 neighbors, perplexity 40, learning rate 0.1, 250 iterations and maximum gradient norm 0.25. Prediction outputs do not fit the clusters, PCA or reference map. Reference and added coordinates share the same centered coordinate frame.

The display expands these reference populations to every cached region with sixteen aligned teacher features across all dataset partitions. This gives 1,007,109 native regions, 16,113,744 individual teacher features and 1,007,109 regional means. The earlier cluster centers and reference map remain unchanged. Contours summarize every mapped reference with a square-bin histogram using 256 bins along the longest coordinate span and Gaussian smoothing of two bins. Each contour encloses approximately 75\% of its cluster's smoothed density. Hollow markers identify actual observations nearest the cluster centers, with size reflecting their population share. Contours describe the reference distribution rather than uncertainty or anatomical boundaries.

For each dataset, one test slide is selected uniformly from sorted eligible slide identities using seed 20260923 combined with the dataset index. Every cached complete region of that slide is displayed for both prediction forms, with no further subsampling. Cached regions can exceed the 128 parents selected for the complete-field task study in Appendix~\ref{app:protocol}, because these illustrations use all eligible cached regions rather than that task selection. Trained predictions use the saved context from the whole native slide bag. Purple and red lines connect each native feature to its direct and residual outputs, respectively, while green lines connect it to its teacher mean. These connections identify the same physical region rather than successive training states. Cluster numbers are omitted because they do not identify tissue classes. All points and their coordinates remain unchanged. Figure~\ref{fig:teacher-correspondence-main} presents CAM16, CAM17, KIRC and KIRP in a compact grid. The examples contrast overlap with teacher neighborhoods, incomplete coverage of a teacher group and differences between direct and residual outputs. These dataset choices are editorial, while slide selection within each remains random. Each main-text panel overlays both prediction forms with one shared legend. The other six datasets appear below with separate direct and residual panels. The saved coordinates, points and connections remain unchanged. Display limits include every visible contour and observation, without reserving space for initial outputs that are not plotted. Equal axis scaling preserves the geometry within each map. Quantitative conclusions use the independent evaluations across all ten datasets.

\paragraph{Reconstruction accuracy and regional variation.}
Table~\ref{tab:slide-feature-errors} computes MSE on precisely the regions shown, averaging equally over regions and feature coordinates. Both trained predictors reduce error relative to native features on each displayed slide, with direct prediction giving the smaller error. On the KIRC slide in Figure~\ref{fig:teacher-correspondence-main}, direct outputs have 0.5284 times the teacher's centered variance, compared with 1.5041 for residual outputs. Their correlations with teacher pairwise distances are 0.8816 and 0.9215, respectively. These calculations use the original 1,536 coordinates and equal region weights, with the definitions in Equation~\ref{eq:feature-variance-ratio}. Thus the wider residual distribution on this slide accompanies better agreement in relative distances but slightly higher average reconstruction error across the corresponding targets. These criteria describe different aspects of prediction.

The independent five-realization evaluation in Table~\ref{tab:feature-distribution-zero} establishes the broader pattern. Every dataset average has direct variation below the teacher reference and residual variation above it. Both forms preserve regional differences, as their positive distance correlations and the permutation results in Table~\ref{tab:raw-test-permutation} show. Direct prediction's smaller spread is not a constant collapse. Nor is it a decrease from initialization, since training expands the initially zero direct output into an input-dependent representation. Distance correlations do not have a universal direct or residual ordering. For example, the KIRC dataset average favors direct prediction even though the displayed slide favors residual prediction. The illustrative slide and the independent evaluation fields answer different sampling questions.

\paragraph{What the remaining maps show.}
The remaining six maps extend the examples in Figure~\ref{fig:teacher-correspondence-main}. PANDA shows a compact direct group while some residual outputs remain near native features. UCEC shows residual predictions occupying teacher neighborhoods that direct outputs cover less extensively. BRACS, Private-CRC, LUAD and STAD show different degrees of overlap between outputs and teacher references. Across all ten displayed slides, direct prediction has smaller variation and lower MSE than residual prediction. Residual variation is not always larger than teacher variation on these individual slides. The different visual patterns do not establish a common number of recovered clusters or a classification-versus-survival distinction.

Dataset contours describe a broader population than the teacher targets of the displayed slide. Their overlap with individual \twenty{} contours does not mean that the displayed means cover every teacher cluster or preserve every individual vector. Moreover, openTSNE positions each added output against reference landmarks without optimizing interactions among the added outputs. Apparent concentration is not a calibrated measure of variance or density in the original space. Connections share the native feature as their origin and identify its corresponding output and teacher mean. Original-coordinate measurements determine reconstruction accuracy and variation. Native and teacher extraction histories also differ, so separation of their reference distributions cannot be attributed solely to magnification from this comparison.

\paragraph{Two kinds of averaging and the predictor bottleneck.}
Lemma~\ref{lem:shared-output-mean} concerns averaging the sixteen observed teacher vectors within one region. Under uniform squared error, their variation around that regional mean contributes no predictor gradient for one shared output. Proposition~\ref{prop:conditional-mean} concerns a different average, the possible regional targets given the low magnification input. If that input does not determine the target, even the optimal deterministic predictor averages over the remaining possibilities. It need not reproduce the full distribution of teacher observations. These distinctions explain why overlap with individual teacher features is not evidence of recovering their within-region detail, and why less variable predictions can still minimize squared error.

The implemented heads also impose a concrete architectural restriction. Each network has a 256-unit GELU hidden layer followed by an affine output in 1,536 coordinates, with no final activation. For a given trained network, all direct outputs lie in an affine subspace of dimension at most 256. The residual network adds its correction to the native vector, allowing variation outside that correction subspace to pass through the skip connection. This distinction is consistent with their different spread despite the common unrestricted optimum in Proposition~\ref{prop:conditional-mean}. It does not establish whether the main reconstruction difficulty comes from architecture, optimization or conditional uncertainty. In particular, the direct predictor's more limited output subspace is not evidence that the missing variation is unavailable in its input.

\paragraph{From reconstruction to task benefit.}
Theorem~\ref{thm:transfer-risk} distinguishes a teacher observation from a deterministic transformation of $B$. Its teacher access term is positive only when the teacher contributes information that reduces the best attainable task loss beyond what $B$ permits. Feature uncertainty need not have that task consequence. A predictor can fail to match teacher features because of its finite approximation, and some unmatched feature variation may be irrelevant to $Y$. Neither a sparse set of mapped neighborhoods nor a nonzero MSE identifies the teacher access term.

The standalone task results in Table~\ref{tab:all-cohort-fidelity} show that lower reconstruction error does not consistently determine which prediction form performs better. KIRC with zero initialization favors direct prediction in mean C-index, despite its smaller spread, and the ordering reverses under random initialization. Wider residual coverage is therefore not a general explanation for better task performance. The native retention study in Appendix~\ref{app:native-retention} addresses a distinct consequence. Retaining the complete $B$ removes representation loss, whereas replacing it by either output can change both representation loss and model excess. A residual sum alone does not guarantee recovery of $B$. Together, the maps and task controls motivate preserving useful regional distinctions while testing their task value separately from average reconstruction accuracy.

\begin{table}[htbp]
\centering\small
\setlength{\tabcolsep}{5pt}
\caption{\textbf{Reconstruction error for the exact regions in the feature maps.} Native features are the reference. Each entry is MSE over one displayed test slide and one trained model with zero final-layer initialization. Errors use the original 1,536 coordinates and equal weight per region. These are individual illustrative cases, without aggregation across training seeds. Lower is better.}
\label{tab:slide-feature-errors}
\begin{tabular}{@{}lrccc@{}}\toprule
Dataset & Regions & Native $x_i$ & Direct $f_\theta(B)_i$ & Residual $x_i+r_\phi(B)_i$\\\midrule
CAM16 & 128 & $0.1571$ & \cellcolor{scorebetter}$\mathbf{0.0516}$ & \cellcolor{scorebetter}$0.0821$\\
CAM17 & 135 & $0.2518$ & \cellcolor{scorebetter}$\mathbf{0.1360}$ & \cellcolor{scorebetter}$0.1641$\\
Private-CRC & 149 & $0.2346$ & \cellcolor{scorebetter}$\mathbf{0.1278}$ & \cellcolor{scorebetter}$0.1978$\\
PANDA & 52 & $0.1463$ & \cellcolor{scorebetter}$\mathbf{0.1044}$ & \cellcolor{scorebetter}$0.1232$\\
BRACS & 134 & $0.1429$ & \cellcolor{scorebetter}$\mathbf{0.0761}$ & \cellcolor{scorebetter}$0.0889$\\
KIRC & 128 & $0.1697$ & \cellcolor{scorebetter}$\mathbf{0.0902}$ & \cellcolor{scorebetter}$0.0942$\\
KIRP & 131 & $0.1732$ & \cellcolor{scorebetter}$\mathbf{0.0881}$ & \cellcolor{scorebetter}$0.1020$\\
LUAD & 128 & $0.2139$ & \cellcolor{scorebetter}$\mathbf{0.0901}$ & \cellcolor{scorebetter}$0.1250$\\
STAD & 154 & $0.1652$ & \cellcolor{scorebetter}$\mathbf{0.0822}$ & \cellcolor{scorebetter}$0.1100$\\
UCEC & 160 & $0.1747$ & \cellcolor{scorebetter}$\mathbf{0.1040}$ & \cellcolor{scorebetter}$0.1144$\\
\bottomrule\end{tabular}
\end{table}

\begin{figure}[p]
\centering
\includegraphics[width=\linewidth]{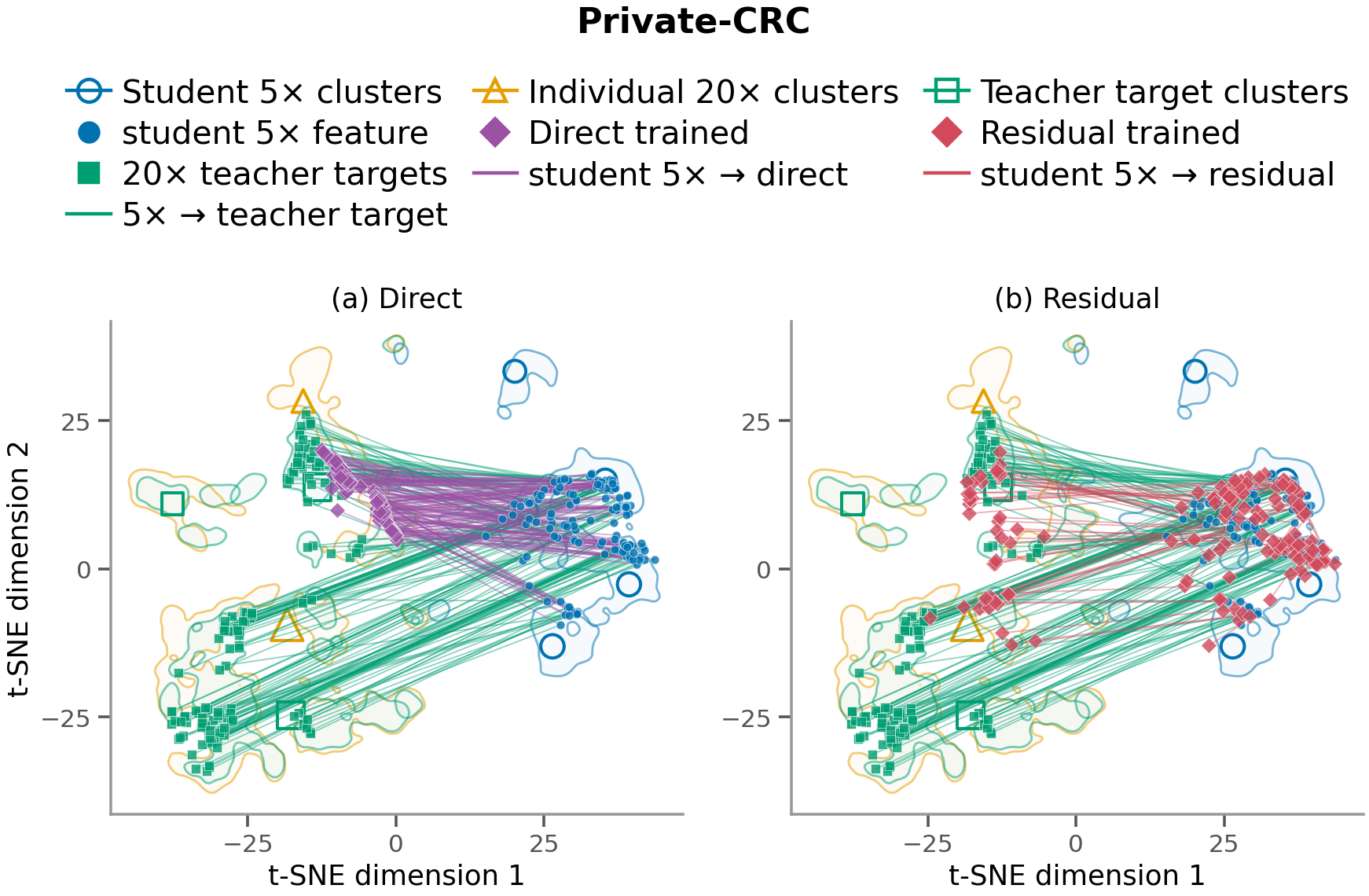}
\caption{\textbf{Teacher correspondence in Private-CRC.} Direct and residual predictions for all 149 cached regions of one test slide. Some residual outputs remain near the native reference distribution while their corresponding teacher targets occupy different neighborhoods. Colors and connections follow Figure~\ref{fig:teacher-correspondence-main}. Table~\ref{tab:slide-feature-errors} reports errors for these exact regions.}
\label{fig:teacher-correspondence-private_crc}
\end{figure}

\begin{figure}[p]
\centering
\includegraphics[width=\linewidth]{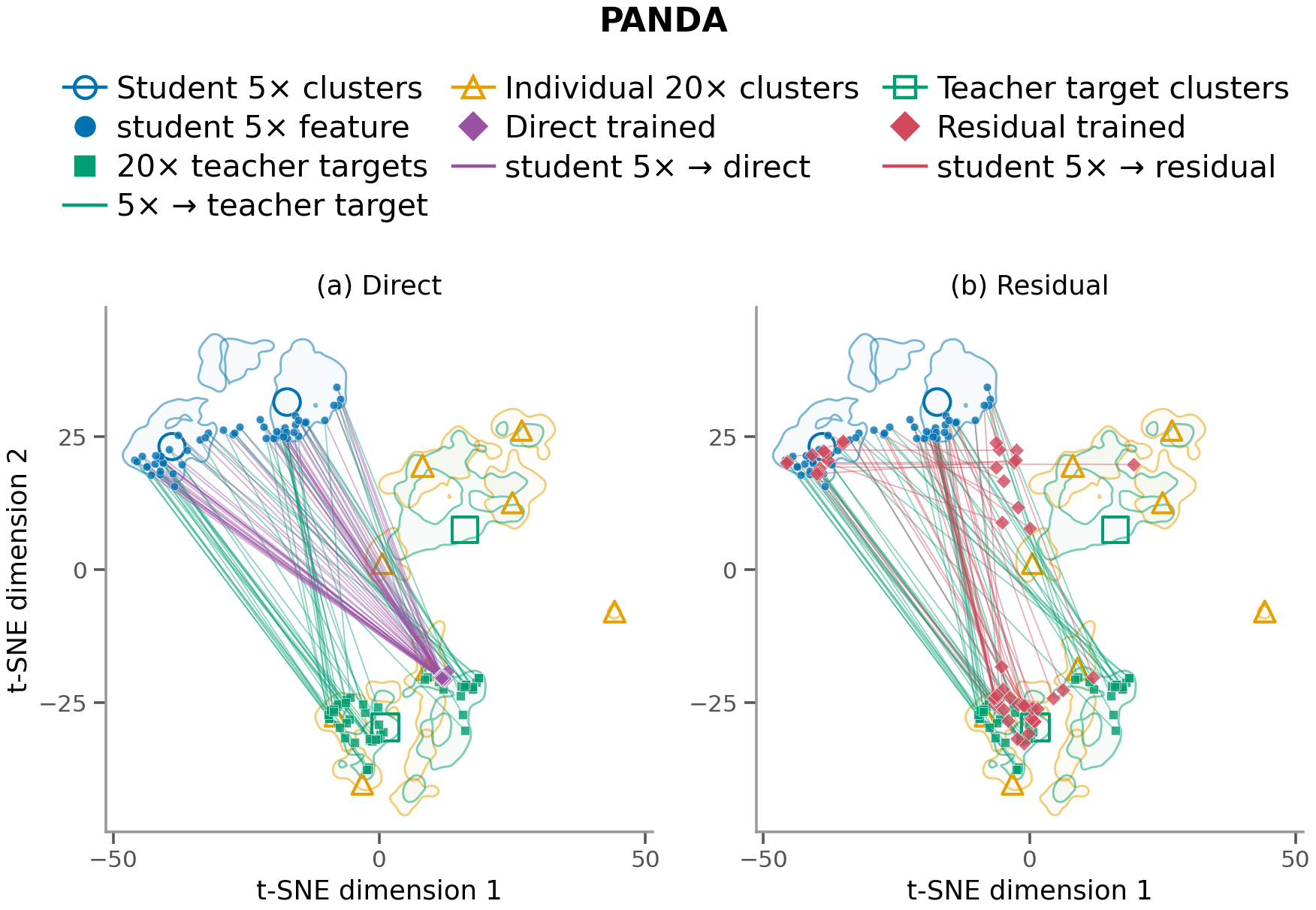}
\caption{\textbf{Teacher correspondence in PANDA.} Direct and residual predictions for all 52 cached regions of one test slide. Direct outputs form a compact group near teacher means, while residual outputs also extend toward native neighborhoods. Colors and connections follow Figure~\ref{fig:teacher-correspondence-main}. Table~\ref{tab:slide-feature-errors} reports errors for these exact regions.}
\label{fig:teacher-correspondence-panda}
\end{figure}

\begin{figure}[p]
\centering
\includegraphics[width=\linewidth]{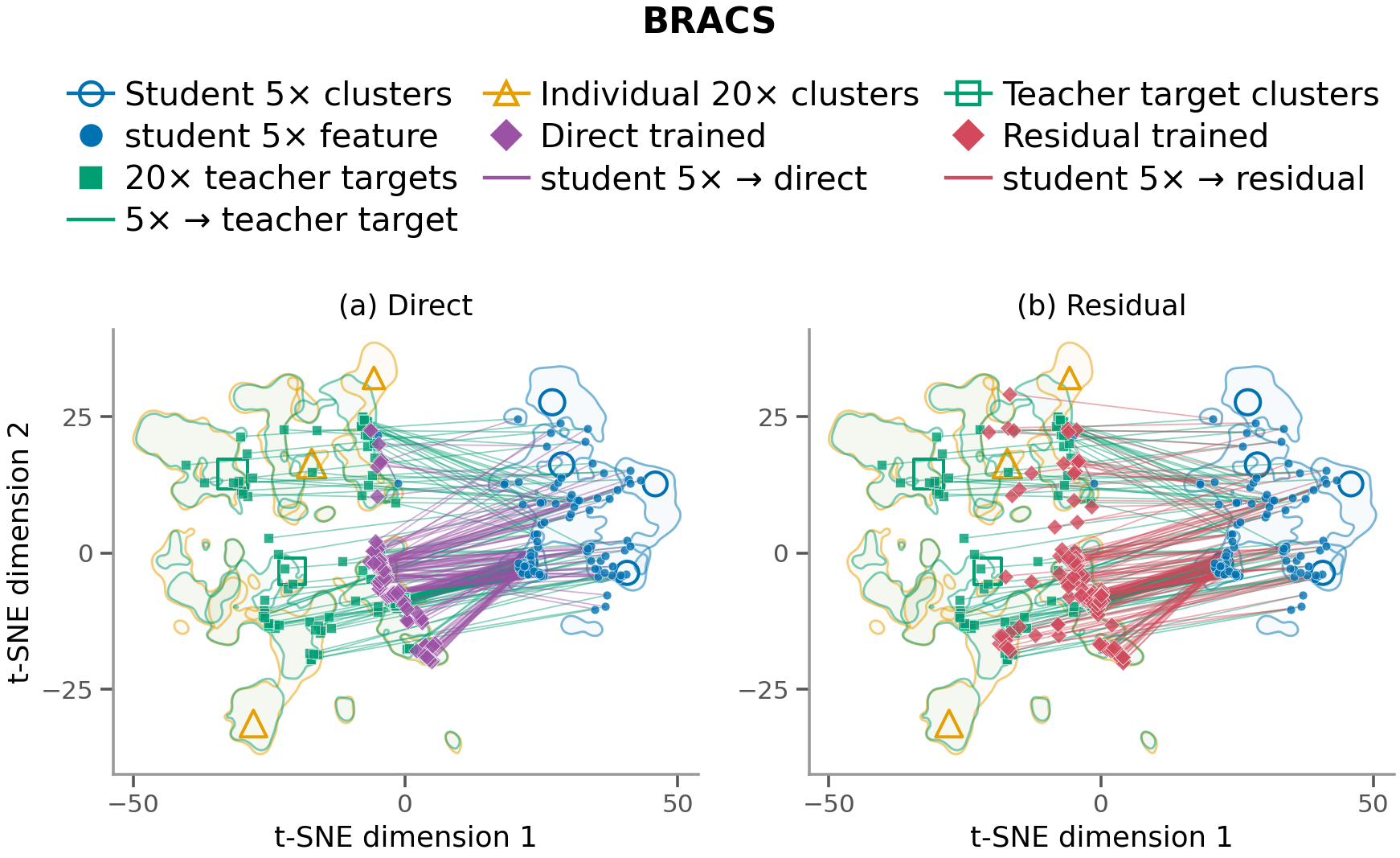}
\caption{\textbf{Teacher correspondence in BRACS.} Direct and residual predictions for all 134 cached regions of one test slide. Direct outputs concentrate near central teacher neighborhoods, while residual outputs extend across more of the displayed target distribution. Colors and connections follow Figure~\ref{fig:teacher-correspondence-main}. Table~\ref{tab:slide-feature-errors} reports errors for these exact regions.}
\label{fig:teacher-correspondence-bracs}
\end{figure}

\begin{figure}[p]
\centering
\includegraphics[width=\linewidth]{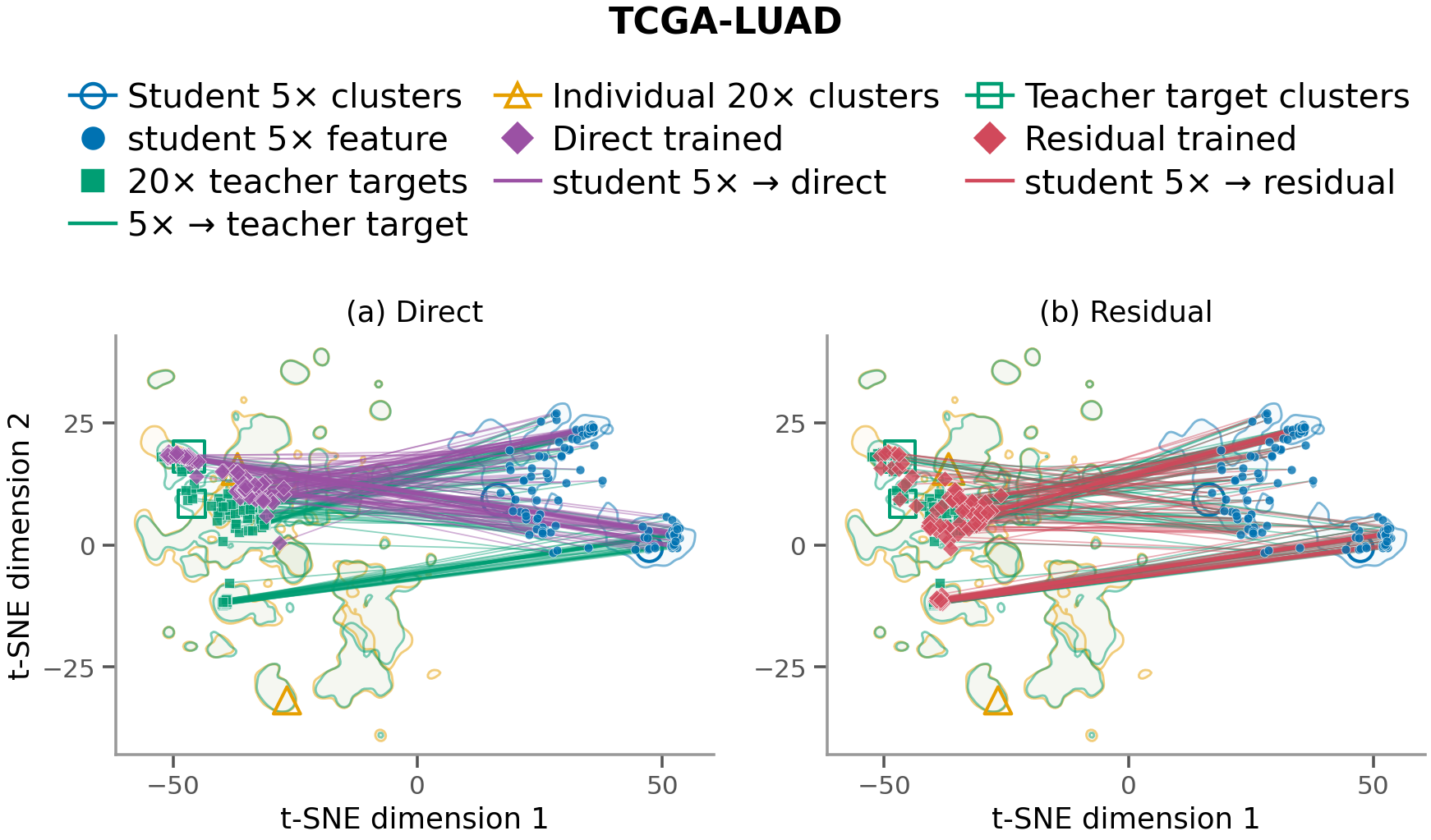}
\caption{\textbf{Teacher correspondence in TCGA-LUAD.} Direct and residual predictions for all 128 cached regions of one test slide. Both prediction forms overlap teacher neighborhoods, but the residual outputs show a broader spread in this view. Colors and connections follow Figure~\ref{fig:teacher-correspondence-main}. Table~\ref{tab:slide-feature-errors} reports errors for these exact regions.}
\label{fig:teacher-correspondence-luad}
\end{figure}

\begin{figure}[p]
\centering
\includegraphics[width=\linewidth]{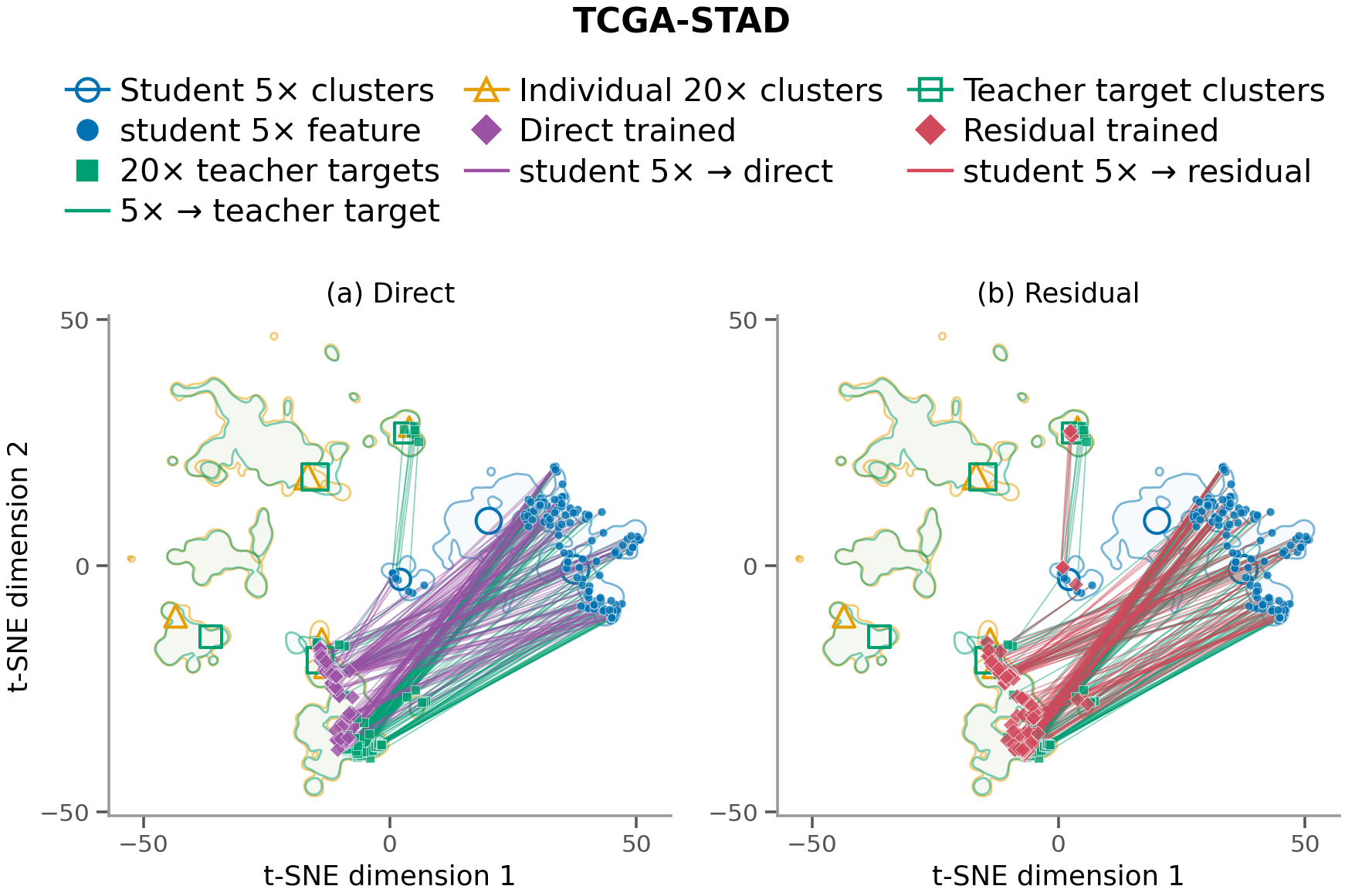}
\caption{\textbf{Teacher correspondence in TCGA-STAD.} Direct and residual predictions for all 154 cached regions of one test slide. Both prediction forms approach the lower teacher neighborhoods, while direct outputs retain less variation in the original coordinates. Colors and connections follow Figure~\ref{fig:teacher-correspondence-main}. Table~\ref{tab:slide-feature-errors} reports errors for these exact regions.}
\label{fig:teacher-correspondence-stad}
\end{figure}

\begin{figure}[p]
\centering
\includegraphics[width=\linewidth]{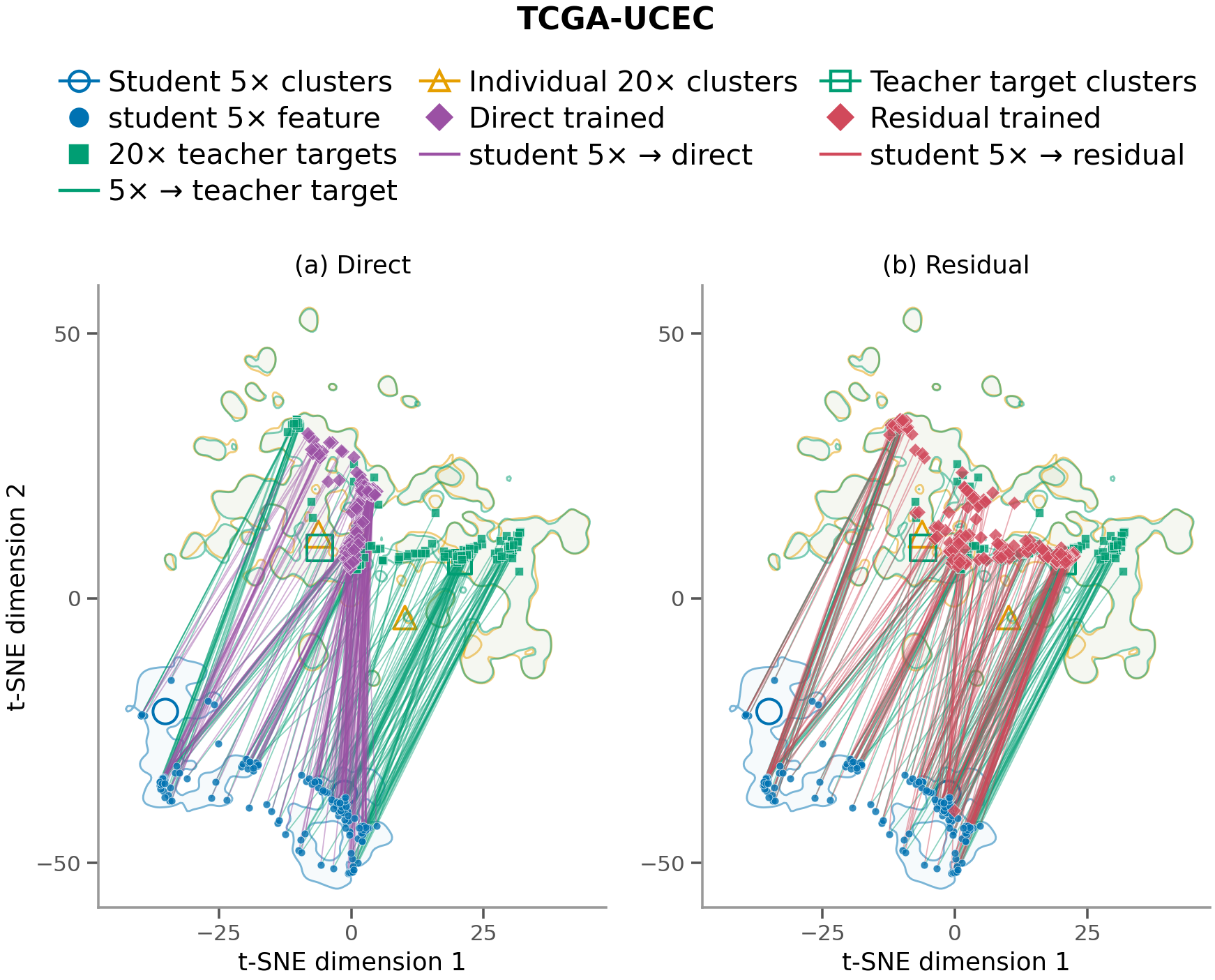}
\caption{\textbf{Teacher correspondence in TCGA-UCEC.} Direct and residual predictions for all 160 cached regions of one test slide. Residual outputs extend across the right teacher neighborhoods more broadly than direct outputs, yet direct prediction has lower target MSE. Colors and connections follow Figure~\ref{fig:teacher-correspondence-main}. Table~\ref{tab:slide-feature-errors} reports errors for these exact regions.}
\label{fig:teacher-correspondence-ucec}
\end{figure}

\FloatBarrier

\subsection{Predicting projected mean and spatial components}
\label{app:source-learning}
The next diagnostic asks which target components account for prediction error. It uses a different distillation partition and head from the raw analysis above. Its native input $b=[x;u]\in\R^{2304}$ concatenates the 1,536-coordinate UNI2-h feature and 768-coordinate DINOv3 feature from one \five{} region. The target has four blocks of 192 coordinates. A common MLP predicts all 768 outputs,
\begin{equation}
f_\theta(b)=W_2\operatorname{GELU}(W_1b+a_1)+a_2,
\qquad W_1\in\R^{512\times2304},\quad W_2\in\R^{768\times512}.
\end{equation}
The parameters $\theta$ comprise both weight matrices and their biases. GELU applies $\operatorname{GELU}(z)=z\Phi(z)$ coordinatewise, where $\Phi$ is the standard normal cumulative distribution function \citep{hendrycks2016gelu}. Squared error across the complete target trains all four components jointly.

Only base training groups determine the transformations. Their moments scale the input, while moments of projected native \five{} UNI2-h features normalize the targets. We center the mean block, use zero offsets for spatial blocks and apply the same 192 scales to all four. This preserves their different variances rather than giving every block unit variance.

Each distillation dataset contributes eight base training groups, eight correction groups, four validation groups for checkpoint selection and four independent evaluation groups. Eight diagnostic groups from the excluded dataset evaluate generalization without influencing fitting or selection. We sample at most 32 aligned fields per group and weight datasets, groups within datasets and fields within groups equally. Mean and spatial errors use identical fields. These samples differ from the task test fields above and from full downstream bags. Table~\ref{tab:neural-prediction-errors} evaluates the base predictor before correction.

Training uses at most 4,096 AdamW updates with learning rate $.001$, weight decay $.0001$, batch size 1,024 and gradient clipping at 1. Total selection MSE chooses update 32 in thirty panels and update 128 in twenty. The reference always returns the target average from base training groups and uses the same coordinates, fields and weights as the network. Since it predicts no regional differences, improving on this constant tests predictive value beyond a common target centre.

\subsection{Reconstruction errors by component}
\label{app:component-errors}
Prediction improves the regional mean much more consistently than the spatial components. In Table~\ref{tab:neural-prediction-errors}, the mean head beats the constant on every independent evaluation within the distillation datasets and in eight excluded-dataset averages. Every nonconstant component instead has higher error than the constant in every realization on both populations. This separates useful mean prediction from a persistent spatial deficit in the fitted network. Under Proposition~\ref{prop:conditional-mean}, these observed errors still combine conditional target variation with estimation error, so the deficit does not identify an irreducible limit.

\begin{table}[htbp]
\centering\small
\setlength{\tabcolsep}{3pt}
\caption{MSE for regional mean and spatial components. The constant is the reference within each component. Panel A evaluates independent groups from the distillation datasets, and Panel B evaluates the excluded dataset. Neither population selects predictors.}
\label{tab:neural-prediction-errors}
\begin{tabular}{@{}lcccc@{}}\toprule
\multicolumn{5}{c}{\textbf{A. Independent evaluation within distillation datasets}}\\[3pt]
& \multicolumn{2}{c}{Regional mean} & \multicolumn{2}{c}{Spatial components}\\
\cmidrule(lr){2-3}\cmidrule(l){4-5}
Excluded dataset & Constant & Network & Constant & Network\\\midrule
CAM16 & $0.6361\!\pm\!0.0305$ & \cellcolor{scorebetter}$\mathbf{0.3648}\!\pm\!0.0173$ & $\mathbf{0.0524}\!\pm\!0.0022$ & \cellcolor{scoreworse}$0.0697\!\pm\!0.0031$\\
CAM17 & $0.6665\!\pm\!0.0335$ & \cellcolor{scorebetter}$\mathbf{0.3702}\!\pm\!0.0219$ & $\mathbf{0.0527}\!\pm\!0.0025$ & \cellcolor{scoreworse}$0.0676\!\pm\!0.0021$\\
Private-CRC & $0.6851\!\pm\!0.0567$ & \cellcolor{scorebetter}$\mathbf{0.4007}\!\pm\!0.0191$ & $\mathbf{0.0527}\!\pm\!0.0010$ & \cellcolor{scoreworse}$0.0698\!\pm\!0.0019$\\
PANDA & $0.6592\!\pm\!0.0310$ & \cellcolor{scorebetter}$\mathbf{0.3775}\!\pm\!0.0188$ & $\mathbf{0.0487}\!\pm\!0.0011$ & \cellcolor{scoreworse}$0.0662\!\pm\!0.0023$\\
BRACS & $0.6295\!\pm\!0.0249$ & \cellcolor{scorebetter}$\mathbf{0.3653}\!\pm\!0.0098$ & $\mathbf{0.0533}\!\pm\!0.0012$ & \cellcolor{scoreworse}$0.0691\!\pm\!0.0014$\\
\addlinespace[3pt]
KIRC & $0.6419\!\pm\!0.0226$ & \cellcolor{scorebetter}$\mathbf{0.3671}\!\pm\!0.0165$ & $\mathbf{0.0552}\!\pm\!0.0016$ & \cellcolor{scoreworse}$0.0718\!\pm\!0.0026$\\
KIRP & $0.6320\!\pm\!0.0533$ & \cellcolor{scorebetter}$\mathbf{0.3584}\!\pm\!0.0220$ & $\mathbf{0.0542}\!\pm\!0.0020$ & \cellcolor{scoreworse}$0.0705\!\pm\!0.0015$\\
LUAD & $0.6612\!\pm\!0.0309$ & \cellcolor{scorebetter}$\mathbf{0.3552}\!\pm\!0.0162$ & $\mathbf{0.0564}\!\pm\!0.0017$ & \cellcolor{scoreworse}$0.0730\!\pm\!0.0013$\\
STAD & $0.6681\!\pm\!0.0490$ & \cellcolor{scorebetter}$\mathbf{0.3671}\!\pm\!0.0204$ & $\mathbf{0.0542}\!\pm\!0.0017$ & \cellcolor{scoreworse}$0.0708\!\pm\!0.0014$\\
UCEC & $0.6269\!\pm\!0.0235$ & \cellcolor{scorebetter}$\mathbf{0.3499}\!\pm\!0.0109$ & $\mathbf{0.0543}\!\pm\!0.0008$ & \cellcolor{scoreworse}$0.0706\!\pm\!0.0019$\\
\midrule
\multicolumn{5}{c}{\textbf{B. Excluded dataset evaluation}}\\[3pt]
& \multicolumn{2}{c}{Regional mean} & \multicolumn{2}{c}{Spatial components}\\
\cmidrule(lr){2-3}\cmidrule(l){4-5}
Excluded dataset & Constant & Network & Constant & Network\\\midrule
CAM16 & $0.6106\!\pm\!0.0182$ & \cellcolor{scorebetter}$\mathbf{0.4501}\!\pm\!0.0418$ & $\mathbf{0.0464}\!\pm\!0.0015$ & \cellcolor{scoreworse}$0.0684\!\pm\!0.0036$\\
CAM17 & $0.5618\!\pm\!0.0415$ & \cellcolor{scorebetter}$\mathbf{0.4599}\!\pm\!0.0539$ & $\mathbf{0.0550}\!\pm\!0.0039$ & \cellcolor{scoreworse}$0.0724\!\pm\!0.0038$\\
Private-CRC & $\mathbf{0.7193}\!\pm\!0.0842$ & \cellcolor{scoreworse}$0.7721\!\pm\!0.0848$ & $\mathbf{0.0845}\!\pm\!0.0096$ & \cellcolor{scoreworse}$0.1162\!\pm\!0.0075$\\
PANDA & $\mathbf{0.5319}\!\pm\!0.0264$ & \cellcolor{scoreworse}$0.5602\!\pm\!0.0312$ & $\mathbf{0.0845}\!\pm\!0.0049$ & \cellcolor{scoreworse}$0.1106\!\pm\!0.0069$\\
BRACS & $0.7437\!\pm\!0.1001$ & \cellcolor{scorebetter}$\mathbf{0.6339}\!\pm\!0.0924$ & $\mathbf{0.0454}\!\pm\!0.0056$ & \cellcolor{scoreworse}$0.0646\!\pm\!0.0081$\\
\addlinespace[3pt]
KIRC & $0.8488\!\pm\!0.0866$ & \cellcolor{scorebetter}$\mathbf{0.6799}\!\pm\!0.1207$ & $\mathbf{0.0421}\!\pm\!0.0022$ & \cellcolor{scoreworse}$0.0657\!\pm\!0.0053$\\
KIRP & $0.7377\!\pm\!0.0775$ & \cellcolor{scorebetter}$\mathbf{0.6091}\!\pm\!0.0801$ & $\mathbf{0.0393}\!\pm\!0.0016$ & \cellcolor{scoreworse}$0.0621\!\pm\!0.0045$\\
LUAD & $0.6055\!\pm\!0.0427$ & \cellcolor{scorebetter}$\mathbf{0.5535}\!\pm\!0.0466$ & $\mathbf{0.0456}\!\pm\!0.0023$ & \cellcolor{scoreworse}$0.0738\!\pm\!0.0065$\\
STAD & $0.5791\!\pm\!0.0191$ & \cellcolor{scorebetter}$\mathbf{0.5166}\!\pm\!0.0253$ & $\mathbf{0.0430}\!\pm\!0.0021$ & \cellcolor{scoreworse}$0.0652\!\pm\!0.0040$\\
UCEC & $0.9264\!\pm\!0.1741$ & \cellcolor{scorebetter}$\mathbf{0.8966}\!\pm\!0.1914$ & $\mathbf{0.0490}\!\pm\!0.0017$ & \cellcolor{scoreworse}$0.0767\!\pm\!0.0096$\\
\bottomrule\end{tabular}
\end{table}

The four target components have equal width, so total MSE averages their four errors. Table~\ref{tab:neural-prediction-errors} shows that the constant makes much more error per coordinate on the mean than on the spatial components. A shared objective can consequently improve through mean prediction while spatial error remains high relative to its own reference. This identifies a reason to report the components separately, without attributing the spatial deficit to loss weighting.

\subsection{Independent neural error correction}
\label{app:neural-correction}
We test whether an additional network can correct errors left by the selected base predictor. Let $f_A$ be the frozen head fitted to distillation training groups A. A second MLP $r_B$ learns from disjoint groups B using the same input, 512 hidden units, GELU and 768 outputs. A zero final layer initializes the corrected prediction $q=f_A+r_B$ at the original function $f_A$. This makes any subsequent change attributable to training the correction.

The same selection groups choose both checkpoints by total MSE. Independent evaluation within the distillation datasets and on excluded-dataset groups selects neither checkpoint. Correction introduces both extra capacity and new training groups, so its effect combines those changes. Table~\ref{tab:neural-component-correction} separates the resulting mean and spatial errors. This diagnostic stops at feature prediction and fits no downstream task model.

\begin{table}[htbp]
\centering\footnotesize
\setlength{\tabcolsep}{2.5pt}
\caption{Absolute component MSE before and after neural error correction. The base predictor is the reference within each component. Spatial MSE averages three equally sized components. Neither evaluation population determines the selected predictors.}
\label{tab:neural-component-correction}
\begin{tabular}{@{}lcccc@{}}\toprule
\multicolumn{5}{c}{\textbf{Independent evaluation within distillation datasets}}\\
 & \multicolumn{2}{c}{Regional mean} & \multicolumn{2}{c}{Spatial components}\\
\cmidrule(lr){2-3}\cmidrule(l){4-5}
Excluded dataset & Base & Corrected & Base & Corrected\\\midrule
CAM16 & $0.3648\!\pm\!0.0173$ & \cellcolor{scorebetter}$\mathbf{0.3401}\!\pm\!0.0199$ & $0.0697\!\pm\!0.0031$ & \cellcolor{scorebetter}$\mathbf{0.0696}\!\pm\!0.0034$\\
CAM17 & $0.3702\!\pm\!0.0219$ & \cellcolor{scorebetter}$\mathbf{0.3361}\!\pm\!0.0156$ & $0.0676\!\pm\!0.0021$ & \cellcolor{scorebetter}$\mathbf{0.0665}\!\pm\!0.0019$\\
Private-CRC & $0.4007\!\pm\!0.0191$ & \cellcolor{scorebetter}$\mathbf{0.3748}\!\pm\!0.0152$ & $0.0698\!\pm\!0.0019$ & \cellcolor{scorebetter}$\mathbf{0.0694}\!\pm\!0.0023$\\
PANDA & $0.3775\!\pm\!0.0188$ & \cellcolor{scorebetter}$\mathbf{0.3532}\!\pm\!0.0097$ & $0.0662\!\pm\!0.0023$ & \cellcolor{scorebetter}$\mathbf{0.0659}\!\pm\!0.0031$\\
BRACS & $0.3653\!\pm\!0.0098$ & \cellcolor{scorebetter}$\mathbf{0.3455}\!\pm\!0.0167$ & $0.0691\!\pm\!0.0014$ & \cellcolor{scorebetter}$\mathbf{0.0686}\!\pm\!0.0016$\\
KIRC & $0.3671\!\pm\!0.0165$ & \cellcolor{scorebetter}$\mathbf{0.3526}\!\pm\!0.0244$ & $0.0718\!\pm\!0.0026$ & \cellcolor{scorebetter}$\mathbf{0.0716}\!\pm\!0.0028$\\
KIRP & $0.3584\!\pm\!0.0220$ & \cellcolor{scorebetter}$\mathbf{0.3328}\!\pm\!0.0158$ & $0.0705\!\pm\!0.0015$ & \cellcolor{scorebetter}$\mathbf{0.0704}\!\pm\!0.0016$\\
LUAD & $0.3552\!\pm\!0.0162$ & \cellcolor{scorebetter}$\mathbf{0.3342}\!\pm\!0.0241$ & $0.0730\!\pm\!0.0013$ & \cellcolor{scorebetter}$\mathbf{0.0727}\!\pm\!0.0021$\\
STAD & $0.3671\!\pm\!0.0204$ & \cellcolor{scorebetter}$\mathbf{0.3463}\!\pm\!0.0124$ & $\mathbf{0.0708}\!\pm\!0.0014$ & \cellcolor{scoreworse}$0.0709\!\pm\!0.0017$\\
UCEC & $0.3499\!\pm\!0.0109$ & \cellcolor{scorebetter}$\mathbf{0.3255}\!\pm\!0.0136$ & $0.0706\!\pm\!0.0019$ & \cellcolor{scorebetter}$\mathbf{0.0699}\!\pm\!0.0025$\\
\midrule
\multicolumn{5}{c}{\textbf{Excluded dataset evaluation}}\\
 & \multicolumn{2}{c}{Regional mean} & \multicolumn{2}{c}{Spatial components}\\
\cmidrule(lr){2-3}\cmidrule(l){4-5}
Excluded dataset & Base & Corrected & Base & Corrected\\\midrule
CAM16 & $\mathbf{0.4501}\!\pm\!0.0418$ & \cellcolor{scoreworse}$0.4592\!\pm\!0.0420$ & $\mathbf{0.0684}\!\pm\!0.0036$ & \cellcolor{scoreworse}$0.0690\!\pm\!0.0036$\\
CAM17 & $0.4599\!\pm\!0.0539$ & \cellcolor{scorebetter}$\mathbf{0.4528}\!\pm\!0.0478$ & $\mathbf{0.0724}\!\pm\!0.0038$ & \cellcolor{scoreworse}$0.0726\!\pm\!0.0038$\\
Private-CRC & $\mathbf{0.7721}\!\pm\!0.0848$ & \cellcolor{scoreworse}$0.7773\!\pm\!0.0735$ & $\mathbf{0.1162}\!\pm\!0.0075$ & \cellcolor{scoreworse}$0.1168\!\pm\!0.0069$\\
PANDA & $0.5602\!\pm\!0.0312$ & \cellcolor{scorebetter}$\mathbf{0.5584}\!\pm\!0.0301$ & $\mathbf{0.1106}\!\pm\!0.0069$ & \cellcolor{scoreworse}$0.1112\!\pm\!0.0069$\\
BRACS & $0.6339\!\pm\!0.0924$ & \cellcolor{scorebetter}$\mathbf{0.5908}\!\pm\!0.0751$ & $\mathbf{0.0646}\!\pm\!0.0081$ & \cellcolor{scoreworse}$0.0652\!\pm\!0.0079$\\
KIRC & $0.6799\!\pm\!0.1207$ & \cellcolor{scorebetter}$\mathbf{0.6558}\!\pm\!0.1257$ & $\mathbf{0.0657}\!\pm\!0.0053$ & \cellcolor{scoreworse}$0.0669\!\pm\!0.0053$\\
KIRP & $0.6091\!\pm\!0.0801$ & \cellcolor{scorebetter}$\mathbf{0.5872}\!\pm\!0.0958$ & $\mathbf{0.0621}\!\pm\!0.0045$ & \cellcolor{scoreworse}$0.0629\!\pm\!0.0041$\\
LUAD & $0.5535\!\pm\!0.0466$ & \cellcolor{scorebetter}$\mathbf{0.5405}\!\pm\!0.0434$ & $\mathbf{0.0738}\!\pm\!0.0065$ & \cellcolor{scoreworse}$0.0746\!\pm\!0.0064$\\
STAD & $0.5166\!\pm\!0.0253$ & \cellcolor{scorebetter}$\mathbf{0.4971}\!\pm\!0.0185$ & $\mathbf{0.0652}\!\pm\!0.0040$ & \cellcolor{scoreworse}$0.0655\!\pm\!0.0045$\\
UCEC & $0.8966\!\pm\!0.1914$ & \cellcolor{scorebetter}$\mathbf{0.8963}\!\pm\!0.2074$ & $\mathbf{0.0767}\!\pm\!0.0096$ & \cellcolor{scoreworse}$0.0776\!\pm\!0.0090$\\
\bottomrule\end{tabular}
\end{table}

For each excluded dataset, correction lowers average total MSE on independent groups from the distillation datasets, mainly by improving the regional mean (Table~\ref{tab:neural-component-correction}). The mean accounts for 95.7\% of the pooled error reduction. We compute that fraction by averaging each block's reduction over the fifty realizations, then dividing the mean reduction by the sum across all four blocks. It is not an average of realization-level fractions. Correction demonstrates remaining learnable error in the mean while producing only a small reduction in spatial error.

Generalization to excluded datasets is less consistent. Correction lowers total error in six dataset averages and improves the mean on average, with adverse mean effects on CAM16 and Private-CRC. Spatial error rises in every dataset average. Each nonconstant block remains worse than its constant reference in every realization on both evaluation populations, before and after correction. Correction improves regional mean prediction while leaving the spatial deficit unresolved.

\subsection{Spatial access to the low magnification image}
\label{app:input-access-all}
To test whether local image features improve prediction, we give a predictor either repeated global features or a spatial grid. Both inputs contain sixteen 768-coordinate DINOv3 vectors with two position coordinates each. The global version repeats the normalized class token. The spatial version averages the same encoder's normalized $16\times16$ patch-token map into a $4\times4$ grid. Both come from one image field. Flattening produces 12,320 inputs for a $12320\to512\to768$ GELU network. The target is a 768-coordinate projection of the regional teacher mean. It contains neither individual teacher vectors nor nonconstant spatial components. The study comprises 100 distillation fits and 200 TransMIL fits across ten datasets and five realizations.

The predictors train separately with shared architecture, initialization, distillation training and validation groups, and normalization. Each uses 4,096 updates, batch size 1,024, learning rate $.001$, weight decay $.0001$ and gradient clipping at 1. Distillation validation selects checkpoints, including early evaluations at updates 0, 32, 128 and 256. For task evaluation, the 2,304-coordinate native base receives zeros, the teacher mean or a prediction, producing 3,072 coordinates. Every model uses the same complete fields, task initialization, splits and training procedure within a realization, with at most 128 regions per slide.

\begin{table}[t]
\centering\small
\setlength{\tabcolsep}{4pt}
\caption{Task performance when a common native \five{} base receives zeros, the teacher mean or a prediction from global or spatial input. Zero augmentation is the reference. The teacher mean requires \twenty{} features during evaluation.}
\label{tab:input_access-all}
\begin{tabular}{@{}llcccc@{}}\toprule
Dataset & Metric & $[B;0]$ & $[B;M]$ & Global input & Spatial input\\\midrule
CAM16 & AUC & $0.8455\!\pm\!0.0485$ & \cellcolor{scorebetter}$\mathbf{0.8935}\!\pm\!0.0635$ & \cellcolor{scorebetter}$0.8729\!\pm\!0.0384$ & \cellcolor{scorebetter}$0.8516\!\pm\!0.0515$\\
CAM17 & $F_1$ & $0.5197\!\pm\!0.0336$ & \cellcolor{scoreworse}$0.5158\!\pm\!0.0700$ & \cellcolor{scoreworse}$0.5144\!\pm\!0.0381$ & \cellcolor{scorebetter}$\mathbf{0.5275}\!\pm\!0.0310$\\
Private-CRC & $F_1$ & $0.8305\!\pm\!0.0276$ & \cellcolor{scorebetter}$\mathbf{0.8378}\!\pm\!0.0267$ & \cellcolor{scoreworse}$0.8277\!\pm\!0.0228$ & \cellcolor{scoreworse}$0.8276\!\pm\!0.0367$\\
PANDA & $\kappa$ & $0.8386\!\pm\!0.0067$ & \cellcolor{scorebetter}$\mathbf{0.8601}\!\pm\!0.0137$ & \cellcolor{scorebetter}$0.8387\!\pm\!0.0124$ & \cellcolor{scoreworse}$0.8330\!\pm\!0.0062$\\
BRACS & $F_1$ & $\mathbf{0.4211}\!\pm\!0.0710$ & \cellcolor{scoreworse}$0.4168\!\pm\!0.0542$ & \cellcolor{scoreworse}$0.4178\!\pm\!0.0575$ & \cellcolor{scoreworse}$0.4125\!\pm\!0.0856$\\
KIRC & C-index & $0.6823\!\pm\!0.0548$ & \cellcolor{scorebetter}$0.6879\!\pm\!0.0586$ & \cellcolor{scorebetter}$0.6876\!\pm\!0.0718$ & \cellcolor{scorebetter}$\mathbf{0.7186}\!\pm\!0.0538$\\
KIRP & C-index & $0.7352\!\pm\!0.1486$ & \cellcolor{scorebetter}$0.7850\!\pm\!0.1042$ & \cellcolor{scorebetter}$\mathbf{0.8278}\!\pm\!0.0529$ & \cellcolor{scorebetter}$0.8155\!\pm\!0.0828$\\
LUAD & C-index & $0.5381\!\pm\!0.0614$ & \cellcolor{scorebetter}$0.5551\!\pm\!0.0459$ & \cellcolor{scorebetter}$\mathbf{0.5678}\!\pm\!0.0224$ & \cellcolor{scoreworse}$0.5343\!\pm\!0.0527$\\
STAD & C-index & $0.5104\!\pm\!0.0188$ & \cellcolor{scorebetter}$\mathbf{0.5553}\!\pm\!0.0669$ & \cellcolor{scorebetter}$0.5335\!\pm\!0.0561$ & \cellcolor{scoreworse}$0.4940\!\pm\!0.0619$\\
UCEC & C-index & $0.6499\!\pm\!0.1086$ & \cellcolor{scorebetter}$0.7039\!\pm\!0.1145$ & \cellcolor{scorebetter}$\mathbf{0.7105}\!\pm\!0.0840$ & \cellcolor{scorebetter}$0.7048\!\pm\!0.0618$\\
\bottomrule\end{tabular}
\end{table}

\begin{figure}[htbp]
\centering
\includegraphics[width=\linewidth]{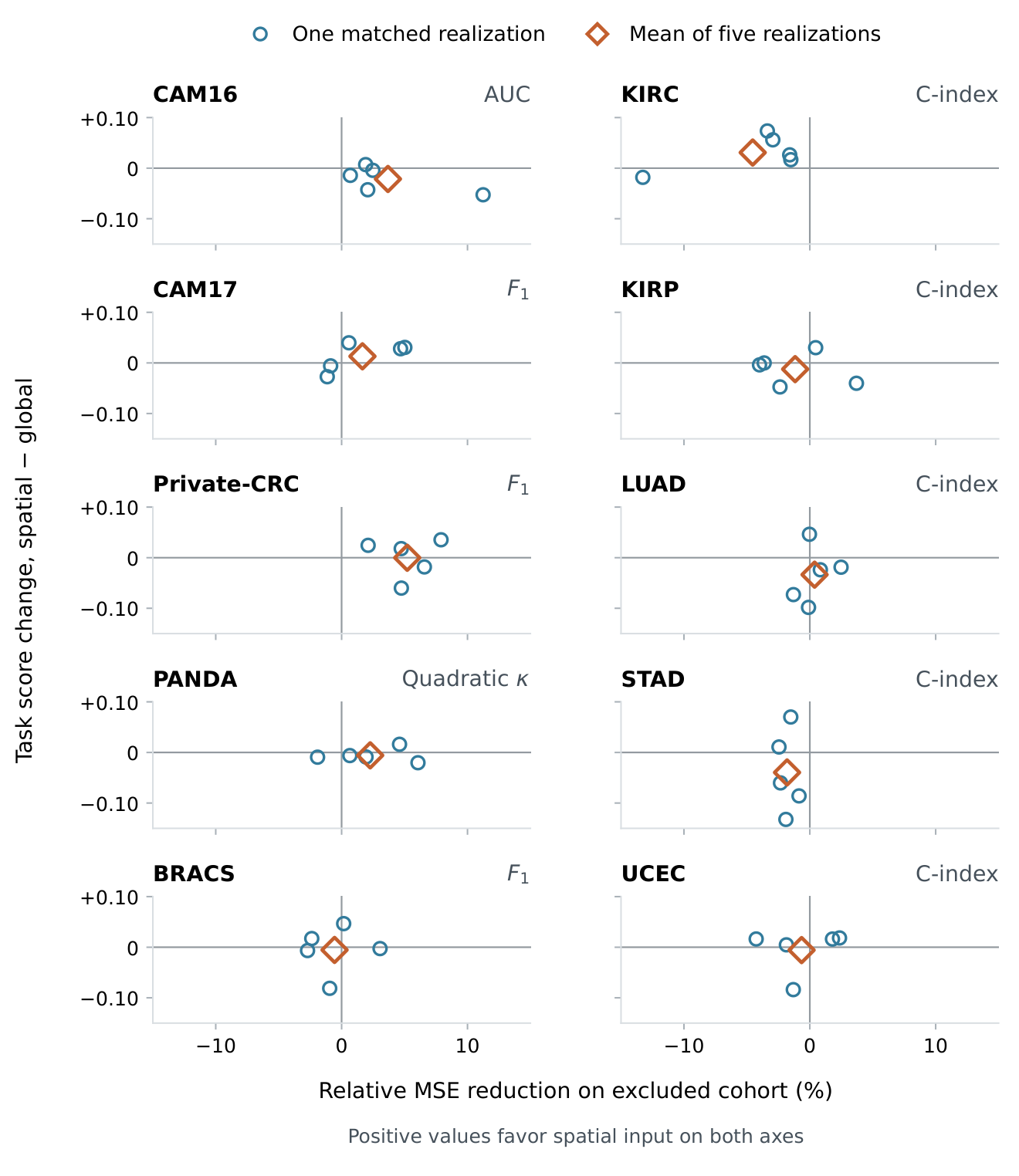}
\caption{Spatial and global input for a common regional mean target. Horizontal position is $100(\mathrm{MSE}_{\rm global}-\mathrm{MSE}_{\rm spatial})/\mathrm{MSE}_{\rm global}$. Vertical position is the spatial minus global task score. Circles show individual realizations and diamonds show their mean. Positive values favor spatial input on each axis.}
\label{fig:input-access-all}
\end{figure}

Local input changes reconstruction and task performance in different ways, as Figure~\ref{fig:input-access-all} shows. It lowers average feature error on CAM16, CAM17, Private-CRC, PANDA and LUAD, whereas only CAM17 and KIRC gain on the average task score. KIRC gains despite worse reconstruction. These descriptive orderings motivate separate feature and task evaluations without establishing a causal relationship between them.

The task model retains global UNI2-h and DINOv3 features, but these do not necessarily recover the local grid seen by the spatial predictor. Complete-input recovery, required to cancel representation loss in Theorem~\ref{thm:transfer-risk}, is therefore not established for this comparison. Appendix~\ref{app:native-retention} tests delivery separately by holding the prediction and its native input unchanged.

\clearpage
\section{Downstream Use and Attribution to Distillation}
\label{app:downstream}
A useful teacher target does not ensure a useful prediction, and a useful prediction does not by itself establish a benefit from cross-resolution distillation. This appendix separates those questions with initial predictors, output removal and native retention. Each study compares task models on common tissue support and input width. The final study tests a different route to transfer by adapting the image encoder and deploying its features.

\subsection{Learning the nonconstant spatial component}
\label{app:spatial-learning}
We ask whether Stage~1 training improves the spatial head after the regional mean is already learned. The study retaining all feature channels in Appendix~\ref{app:n3-joint} compares the trained nonconstant head with its exact initialization while retaining the same trained mean. It also omits the spatial output while retaining that mean. Table~\ref{tab:spatial-complete-arms} reports the absolute NLL for these three choices, using the same trained mean with zero spatial output as its reference. Table~\ref{tab:spatial-complete-arms} adds their primary scores, native input with zero augmentation and observed teacher summaries. Zero spatial output retains the trained mean, whereas zero augmentation removes the entire summary. Every comparison uses the same complete fields and 8,448-coordinate input.

Learning the spatial head lowers average NLL on PANDA, BRACS, KIRC, KIRP, LUAD and UCEC, while CAM16, CAM17, Private-CRC and STAD have higher averages than the initial head. These effects vary across realizations. BRACS has lower NLL after learning in three realizations, whereas PANDA, KIRC and LUAD have lower NLL in all five. Classification NLL uses class probabilities, while survival NLL retains slide weights and training-derived bins. Each loss is interpreted on its own scale.

The comparison with the trained mean alone answers a different delivery question. CAM17 has lower mean NLL when spatial predictions are supplied despite higher NLL than the initial spatial head. STAD has higher mean NLL than both references. On PANDA, the mean NLL improvement over the trained mean coexists with improvement in only two realizations. KIRC and LUAD still have higher mean NLL than native input with zero augmentation. Learning a head, benefiting from its output beyond a predicted mean and improving beyond native input are therefore distinct results. The primary scores likewise show no common ordering across datasets.

\subsection{General augmentation controls}
\label{app:augmentation-controls}
To examine model excess, we test whether distillation makes an unchanged native input easier for the task model to use. The input $B$ from Equation~\ref{eq:teacher_targets} is the bag $B=(b_i)_i$ of native vectors defined in Section~\ref{sec:experiments}. Every task input retains $b_i$. We append zeros, trained predictions $q_i$, teacher summaries $t_i$ or the network's exact initial predictions $q_i^I$,
\begin{equation}
A_0=([b_i;0])_i,\quad A_P=([b_i;q_i])_i,\quad
A_M=([b_i;t_i])_i,\quad A_I=([b_i;q_i^I])_i.
\label{eq:summary_inputs_appendix}
\end{equation}
Zero padding keeps the task input width constant. The initial predictor controls for adding the network transformation before Stage~1 training. Both initial and trained predictions are functions of the retained input, so Equation~\ref{eq:retained-input-risk} expresses their population loss difference as a difference in model excess. The observed teacher supplies a separate reference with access to $(B,T)$.

\subsection{Removing spatial predictions while retaining the same mean}
\label{app:spatial-removal}
To test delivery on unrestricted native bags, we reuse one saved predictor of the projected regional mean and three spatial components. Its full output has four blocks of 192 coordinates. The removal version keeps the mean unchanged and zeros the spatial blocks. Both append 768 coordinates to the native UNI2-h and DINOv3 base, giving a 3,072-coordinate input. A new task model trains for each version using common splits, initialization and settings. Every original \five{} row is retained, distinguishing this support from complete teacher fields and from the separate full-channel study.

Table~\ref{tab:spatial-delivery} compares the complete trained output with the same trained mean alone, exact initial prediction and native input with zeros. Fifty additional task fits supply the mean-only condition across ten datasets and five realizations. Zeroing spatial components also lowers the added feature energy, so the intervention combines omission with that change.

Omitting the spatial blocks has mixed effects. PANDA improves over the full prediction in all five realizations, and its mean score also exceeds the zero and initial references. In contrast, CAM16 and LUAD remain below the native reference after omission. Removing a harmful component can improve an augmented model without establishing a general benefit from distillation.

\begin{table}[htbp]
\centering\footnotesize
\setlength{\tabcolsep}{2.5pt}
\caption{Task scores after omitting projected spatial predictions on unrestricted native bags. Every input retains the native UNI2-h and DINOv3 base and has 3,072 coordinates. Native input with zero augmentation is the reference. Higher is better.}
\label{tab:spatial-delivery}
\begin{tabular}{@{}lcccc@{}}\toprule
Dataset / metric & \shortstack{Native \five{}\\Zero augmentation} & \shortstack{Native \five{}\\Initial prediction} & \shortstack{Native \five{}\\Trained mean + spatial} & \shortstack{Native \five{}\\Same trained mean}\\\midrule
CAM16 / AUC & $\mathbf{0.9770}\!\pm\!0.0129$ & \cellcolor{scoreworse}$0.9625\!\pm\!0.0308$ & \cellcolor{scoreworse}$0.9589\!\pm\!0.0293$ & \cellcolor{scoreworse}$0.9487\!\pm\!0.0315$\\
CAM17 / $F_1$ & $0.6300\!\pm\!0.0294$ & \cellcolor{scorebetter}$0.6397\!\pm\!0.0262$ & \cellcolor{scorebetter}$\mathbf{0.6462}\!\pm\!0.0445$ & \cellcolor{scoreworse}$0.6079\!\pm\!0.0719$\\
Private-CRC / $F_1$ & $0.8244\!\pm\!0.0482$ & \cellcolor{scorebetter}$\mathbf{0.8462}\!\pm\!0.0154$ & \cellcolor{scorebetter}$0.8426\!\pm\!0.0312$ & \cellcolor{scoreworse}$0.8238\!\pm\!0.0383$\\
PANDA / $\kappa$ & $0.8425\!\pm\!0.0121$ & \cellcolor{scorebetter}$0.8489\!\pm\!0.0149$ & \cellcolor{scorebetter}$0.8430\!\pm\!0.0103$ & \cellcolor{scorebetter}$\mathbf{0.8508}\!\pm\!0.0132$\\
BRACS / $F_1$ & $0.4141\!\pm\!0.0340$ & \cellcolor{scorebetter}$0.4336\!\pm\!0.0714$ & \cellcolor{scorebetter}$0.4321\!\pm\!0.0780$ & \cellcolor{scorebetter}$\mathbf{0.4570}\!\pm\!0.0587$\\
\midrule
KIRC / C-index & $0.7017\!\pm\!0.0456$ & \cellcolor{scorebetter}$0.7104\!\pm\!0.0127$ & \cellcolor{scorebetter}$\mathbf{0.7151}\!\pm\!0.0277$ & \cellcolor{scorebetter}$0.7026\!\pm\!0.0395$\\
KIRP / C-index & $0.8262\!\pm\!0.0836$ & \cellcolor{scoreworse}$0.7984\!\pm\!0.0342$ & \cellcolor{scoreworse}$0.7203\!\pm\!0.1028$ & \cellcolor{scorebetter}$\mathbf{0.8267}\!\pm\!0.0343$\\
LUAD / C-index & $\mathbf{0.5457}\!\pm\!0.0297$ & \cellcolor{scoreworse}$0.5099\!\pm\!0.0663$ & \cellcolor{scoreworse}$0.4866\!\pm\!0.0679$ & \cellcolor{scoreworse}$0.5004\!\pm\!0.0636$\\
STAD / C-index & $\mathbf{0.5198}\!\pm\!0.0395$ & \cellcolor{scoreworse}$0.5004\!\pm\!0.0703$ & \cellcolor{scoreworse}$0.4975\!\pm\!0.0275$ & \cellcolor{scoreworse}$0.5030\!\pm\!0.0342$\\
UCEC / C-index & $\mathbf{0.7027}\!\pm\!0.0652$ & \cellcolor{scoreworse}$0.6575\!\pm\!0.0832$ & \cellcolor{scoreworse}$0.6772\!\pm\!0.0691$ & \cellcolor{scoreworse}$0.6988\!\pm\!0.1115$\\
\bottomrule\end{tabular}
\end{table}

\subsection{Direct and residual prediction across ten datasets}
\label{app:all-cohort-fidelity}
To test whether reconstruction error ranks task representations, we compare direct and residual predictions of the same regional mean. This experiment trains task models separately for the two representations. A $3072\rightarrow256\rightarrow1536$ GELU predictor receives the standardized native \five{} vector $x_i$ and the standardized mean $c$ of native vectors across its slide. Direct output $q_i=f_i$ and residual output $q_i=x_i+f_i$ use common target scales $s_{m,d}$ to predict $m_i$,
\begin{equation}
\mathcal L=\mathbb E_{\rm train}\frac{1}{1536}
\sum_{d=1}^{1536}\left(\frac{q_{id}-m_{id}}{s_{m,d}}\right)^2.
\end{equation}
The expectation $\mathbb E_{\rm train}$ averages over the distillation training population. Scales estimated from that population balance the coordinates during training. We reverse those transformations for reconstruction error in the original UNI2-h space. Figure~\ref{fig:held_fidelity} uses the historical evaluation on downstream training fields, while Table~\ref{tab:regional-mean-reconstruction} uses independent test fields from Appendix~\ref{app:raw-test-prediction}. For population $e$ with its recorded group and field weights, error is
\begin{equation}
\operatorname{MSE}_e(q,m)=\mathbb E_e\!\left[\frac1{1536}\sum_{d=1}^{1536}(q_{id}-m_{id})^2\right].
\label{eq:raw-reconstruction-error}
\end{equation}
The historical diagnostic samples eight groups and 234 to 256 downstream training fields from each excluded dataset. Task performance is evaluated on separate test cases. The native reference sets $q_i=x_i$ on the same fields, measuring the discrepancy between native features and the teacher target. It equals the residual predictor at zero initialization before training, as confirmed by direct calculation of $\|m_i-x_i\|^2$. We count this reference once per dataset and realization.

Appendix~\ref{app:raw-test-prediction} uses the training target mean, estimated from the original distillation training sample, as a constant reference for the same selected heads on independent test fields. The native vector tests reduction of the discrepancy between magnifications, whereas the constant tests improvement over a common target centre. Because the populations differ, the two analyses retain separate errors and realization counts.

Each distillation dataset contributes at most sixteen training groups and eight validation groups, with 32 fields per group. The objective contains no cosine or native-retention term. AdamW uses batch size 1,024, learning rate $.001$, weight decay $.0001$, gradient clipping at 1 and 4,096 updates. Distillation validation chooses among checkpoints saved every 512 updates. All 200 predictors select the first candidate at update 512, so these results describe that selection procedure rather than an experiment on training duration.

We cross direct or residual output with random or zero initialization of the final layer. A zero layer starts direct output at zero and residual output at $x_i$. The native skip also changes the finite function class around the 256-coordinate bottleneck. Proposition~\ref{prop:conditional-mean} equates unrestricted optima but permits these differences in finite outputs and optimization. Across ten datasets and five realizations, each task model receives only the 1,536 predicted coordinates. Native features are not supplied as a separate block.

Direct prediction has lower mean reconstruction error in every dataset under both initializations, yet residual prediction often performs better downstream (Figure~\ref{fig:held_fidelity}). Their preferences disagree in 35 of fifty realizations with zero initialization and 28 with random initialization. These results demonstrate that reconstruction error alone cannot rank the resulting task representations. The original-coordinate analysis in Appendix~\ref{app:feature-distribution} shows that the same prediction forms also differ in variation across regions. That difference is a candidate explanation to investigate, not an isolated cause of the task outcomes. Under Theorem~\ref{thm:transfer-risk}, the standalone outputs can differ in both representation loss and model excess. Their reconstruction errors cannot separate these contributions.

\begin{figure}[t]
\centering
\includegraphics[width=\linewidth]{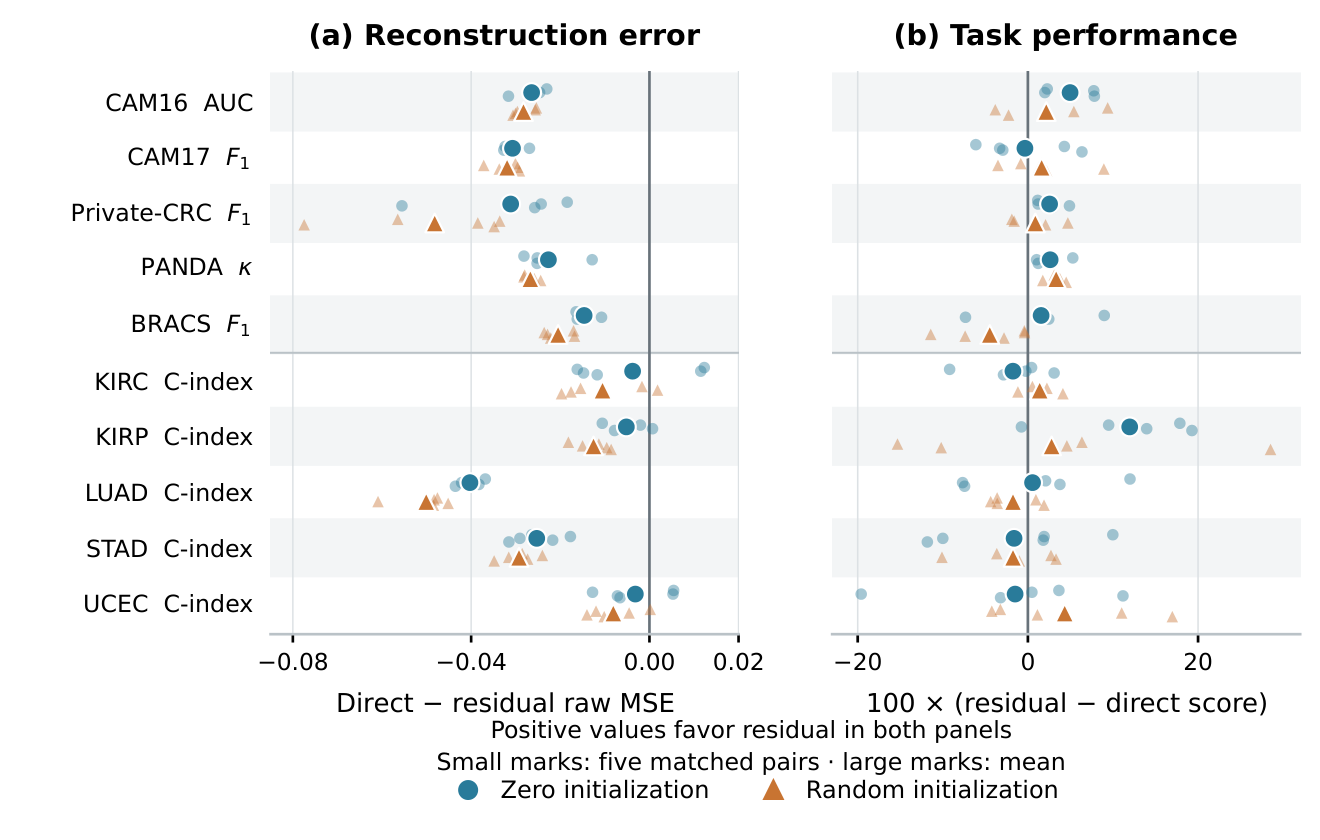}
\caption{\textbf{Reconstruction error need not rank task performance.} Small marks show five realizations and large marks their means. Positive differences favor residual prediction in both panels. Reconstruction uses direct minus residual MSE on historical downstream training fields from datasets excluded from distillation. Task performance uses residual minus direct test score multiplied by 100 in each named metric. Table~\ref{tab:all-cohort-fidelity} reports absolute task scores.}
\label{fig:held_fidelity}
\end{figure}

Table~\ref{tab:all-cohort-fidelity} places both trained prediction forms beside native input. The best form varies across datasets and initialization. A higher mean than the other predictor does not ensure a higher mean than native input. This comparison uses standalone 1,536-coordinate predictions, while the following experiment changes the separate native block at a common width of 3,072.

\begin{table}[htbp]
\centering\footnotesize
\setlength{\tabcolsep}{1.5pt}
\caption{Task scores for native \five{} features and standalone predictions of the regional \twenty{} mean. Each input has 1,536 coordinates. Zero and random specify final-layer initialization before distillation. Native \five{} input is the reference. Higher is better.}
\label{tab:all-cohort-fidelity}
\label{tab:all-cohort-fidelity-random}
\begin{tabular}{@{}lccccc@{}}\toprule
Dataset / metric & Native \five{} & \shortstack{Direct\\Zero init.} & \shortstack{Residual\\Zero init.} & \shortstack{Direct\\Random init.} & \shortstack{Residual\\Random init.}\\\midrule
CAM16 / AUC & $0.9494\!\pm\!0.0198$ & \cellcolor{scoreworse}$0.9142\!\pm\!0.0422$ & \cellcolor{scorebetter}$\mathbf{0.9638}\!\pm\!0.0244$ & \cellcolor{scoreworse}$0.9224\!\pm\!0.0464$ & \cellcolor{scoreworse}$0.9440\!\pm\!0.0417$\\
CAM17 / $F_1$ & $0.6093\!\pm\!0.0556$ & \cellcolor{scorebetter}$0.6116\!\pm\!0.0857$ & \cellcolor{scoreworse}$0.6081\!\pm\!0.0385$ & \cellcolor{scoreworse}$0.6069\!\pm\!0.0683$ & \cellcolor{scorebetter}$\mathbf{0.6231}\!\pm\!0.0494$\\
Private-CRC / $F_1$ & $0.8432\!\pm\!0.0198$ & \cellcolor{scoreworse}$0.8249\!\pm\!0.0179$ & \cellcolor{scorebetter}$\mathbf{0.8504}\!\pm\!0.0142$ & \cellcolor{scoreworse}$0.8412\!\pm\!0.0258$ & \cellcolor{scorebetter}$0.8499\!\pm\!0.0239$\\
PANDA / $\kappa$ & $0.8496\!\pm\!0.0089$ & \cellcolor{scoreworse}$0.8281\!\pm\!0.0203$ & \cellcolor{scorebetter}$\mathbf{0.8542}\!\pm\!0.0116$ & \cellcolor{scoreworse}$0.8154\!\pm\!0.0175$ & \cellcolor{scoreworse}$0.8486\!\pm\!0.0104$\\
BRACS / $F_1$ & $0.4170\!\pm\!0.0527$ & \cellcolor{scorebetter}$0.4337\!\pm\!0.0907$ & \cellcolor{scorebetter}$\mathbf{0.4493}\!\pm\!0.0443$ & \cellcolor{scoreworse}$0.4068\!\pm\!0.0790$ & \cellcolor{scoreworse}$0.3620\!\pm\!0.0708$\\
\midrule
KIRC / C-index & $0.7007\!\pm\!0.0279$ & \cellcolor{scorebetter}$\mathbf{0.7082}\!\pm\!0.0250$ & \cellcolor{scoreworse}$0.6906\!\pm\!0.0407$ & \cellcolor{scoreworse}$0.6909\!\pm\!0.0458$ & \cellcolor{scorebetter}$0.7047\!\pm\!0.0423$\\
KIRP / C-index & $0.7707\!\pm\!0.1060$ & \cellcolor{scoreworse}$0.6881\!\pm\!0.0308$ & \cellcolor{scorebetter}$\mathbf{0.8078}\!\pm\!0.1011$ & \cellcolor{scoreworse}$0.7430\!\pm\!0.1636$ & \cellcolor{scorebetter}$0.7709\!\pm\!0.0964$\\
LUAD / C-index & $0.5327\!\pm\!0.0682$ & \cellcolor{scoreworse}$0.5296\!\pm\!0.0692$ & \cellcolor{scorebetter}$0.5351\!\pm\!0.0300$ & \cellcolor{scorebetter}$\mathbf{0.5636}\!\pm\!0.0464$ & \cellcolor{scorebetter}$0.5461\!\pm\!0.0373$\\
STAD / C-index & $0.5164\!\pm\!0.0790$ & \cellcolor{scorebetter}$0.5192\!\pm\!0.0732$ & \cellcolor{scoreworse}$0.5029\!\pm\!0.0812$ & \cellcolor{scorebetter}$\mathbf{0.5321}\!\pm\!0.0551$ & \cellcolor{scoreworse}$0.5146\!\pm\!0.0770$\\
UCEC / C-index & $\mathbf{0.6756}\!\pm\!0.0830$ & \cellcolor{scoreworse}$0.6576\!\pm\!0.0973$ & \cellcolor{scoreworse}$0.6425\!\pm\!0.1134$ & \cellcolor{scoreworse}$0.6193\!\pm\!0.1241$ & \cellcolor{scoreworse}$0.6626\!\pm\!0.1135$\\
\bottomrule\end{tabular}
\end{table}

\subsection{Native retention with an unchanged teacher prediction}
\label{app:native-retention}
To isolate delivery from reconstruction, we reuse the direct and residual predictions above and change only whether the native vector is supplied alongside them. For each saved $q_i$, the retained version gives the task model $[x_i;q_i]$, while the zero native block version gives $[0_d;q_i]$. Both contain 3,072 coordinates in the same region order. The two task fits share splits, initial weights, training and selection. Their predictions and reconstruction errors are identical. Write $q_i^D=f_\theta(B)_i$ for direct prediction and $q_i^R=x_i+r_\phi(B)_i$ for residual prediction.

The study crosses prediction form, predictor final-layer initialization and native retention, producing eight conditions and 400 task fits across ten datasets and five realizations. Every predictor has undergone Stage~1 training. Zero and random describe how that training began. Fifty additional native-plus-zero references are reused across initialization panels at the same task width. They compare augmentation with the native base, while an exact initial predictor would be needed to attribute its effect to distillation.

\paragraph{Complete native input and the risk comparison.}
For this experiment, the deployment input $B$ in Equation~\ref{eq:teacher_targets} is the complete bag of native UNI2-h vectors. The predictor uses each vector and the bag's arithmetic mean, without DINOv3. Taking the first block of $[x_i;q_i]$ recovers every native vector and hence the slide context. All 14,623 recorded bags have at most 6,488 rows, below the cap of 49,152, so no native row is lost through subsampling. Projection and attention occur later inside task model $g$.

Complete recovery makes representation loss zero for native retention in Theorem~\ref{thm:transfer-risk}. Replacing the separate native block with zeros can leave some native information in $q_i$. A residual prediction, for example, contains $x_i$ in the sum $x_i+r_i(B)$ even when that sum is not invertible. The loss difference between retained and zero native block conditions therefore combines any representation loss with their difference in model excess. Keeping $q_i$ unchanged lets the experiment measure the practical effect of delivery without attributing it to better teacher reconstruction.

\begin{figure}[htbp]
\centering
\includegraphics[width=\linewidth]{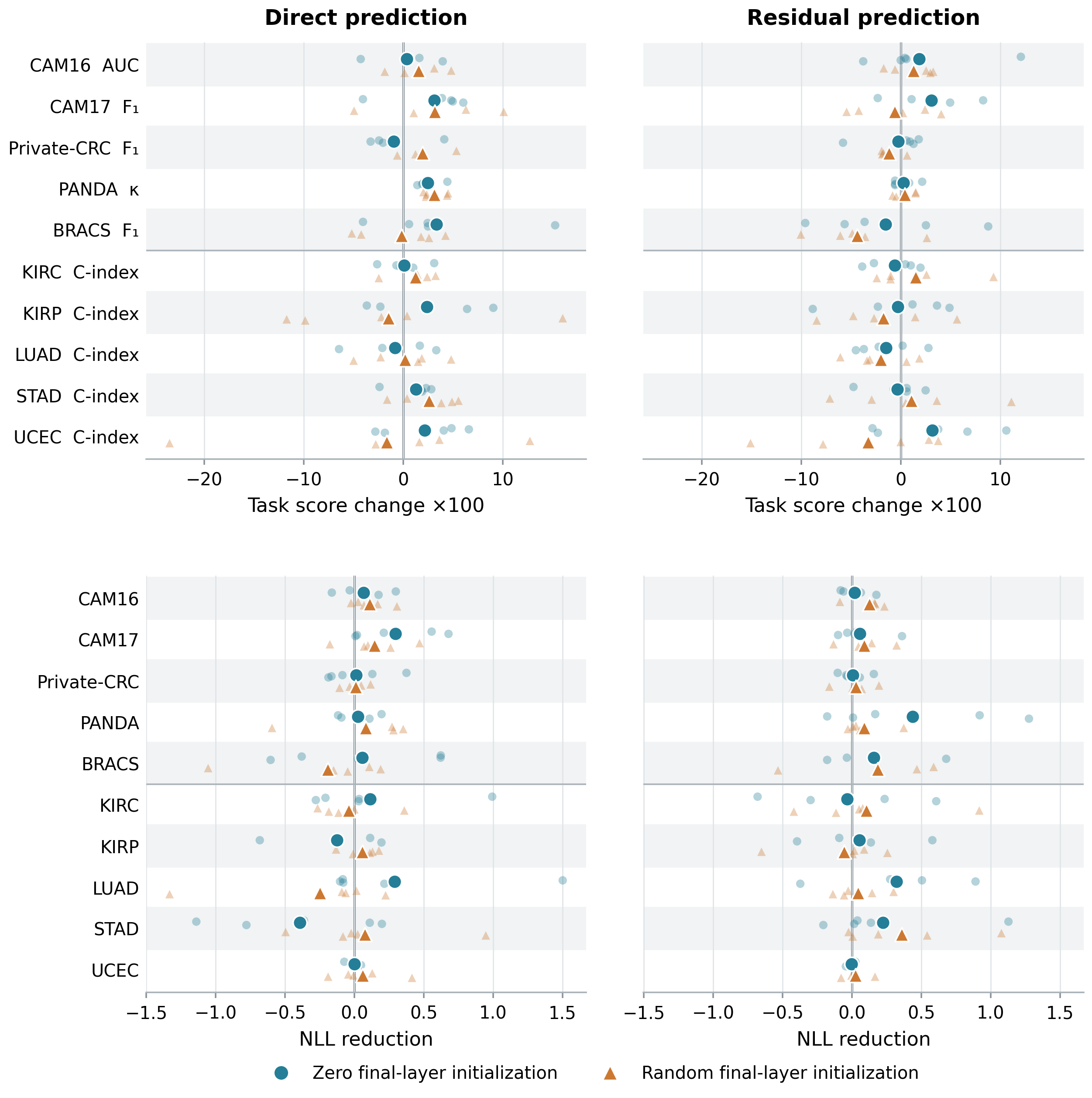}
\caption{\textbf{Native retention changes task results for unchanged predictions.} Upper panels show 100 times the primary score for $[x_i;q_i]$ minus $[0;q_i]$. Lower panels show NLL for $[0;q_i]$ minus $[x_i;q_i]$. Positive values favor retention. Columns distinguish prediction forms, while circles and triangles distinguish zero and random final-layer initialization. All predictors are trained. Small marks show five corresponding task differences and large marks their means. Appendix~\ref{app:native-retention} reports absolute scores.}
\label{fig:native-retention-realizations}
\end{figure}

\paragraph{Task effects of retaining native features.}
Figure~\ref{fig:native-retention-realizations} reports changes within each realization for both primary scores and NLL, extending the NLL comparison in Figure~\ref{fig:native-retention}. Tables~\ref{tab:native-retention-absolute-primary} and~\ref{tab:native-retention-absolute-nll} report absolute results. Native retention improves average primary scores more often for direct predictions than for residual predictions. Residual outputs already include a native skip, which makes the additional block a different intervention for the two forms. NLL favors retention across most conditions of both forms, but neither retention nor teacher prediction consistently improves on the matched native reference. The consequences of delivery depend on the task criterion and vary across the five joint distillation and task realizations.

NLL is the saved arithmetic mean of slide classification cross entropy or discrete survival loss. Survival uses the shared training-derived bins before case aggregation. Classification retains the trainer's probability convention, which differs from the saved case-analysis convention in Appendix~\ref{app:joint-teacher}. We keep these loss magnitudes separate. The reference $[x_i;0_d]$ shares support, input width, initial task weights, splits and training with all eight prediction conditions. It is a different task fit from the 1,536-coordinate native reference in Table~\ref{tab:all-cohort-fidelity}.

\begin{table}[htbp]
\centering\footnotesize
\setlength{\tabcolsep}{1.5pt}
\caption{Task scores with the same trained prediction and either a zero or retained native block. All inputs have 3,072 coordinates. The matched native \five{} reference $[x_i;0_d]$ is reused across panels. Higher is better.}
\label{tab:native-retention-absolute-primary}
\label{tab:native-retention-primary}
\begin{tabular}{@{}lccccc@{}}\toprule
Dataset / metric & \shortstack{Native \five{}\\$[x_i;0_d]$} & \shortstack{Direct\\$[0_d;q_i^D]$} & \shortstack{Direct + native\\$[x_i;q_i^D]$} & \shortstack{Residual\\$[0_d;q_i^R]$} & \shortstack{Residual + native\\$[x_i;q_i^R]$}\\\midrule
\multicolumn{6}{c}{Zero final-layer initialization before distillation}\\
CAM16 / AUC & $0.9669\!\pm\!0.0183$ & \cellcolor{scoreworse}$0.9415\!\pm\!0.0306$ & \cellcolor{scoreworse}$0.9450\!\pm\!0.0290$ & \cellcolor{scoreworse}$0.9486\!\pm\!0.0556$ & \cellcolor{scorebetter}$\mathbf{0.9672}\!\pm\!0.0146$\\
CAM17 / $F_1$ & $0.6080\!\pm\!0.0327$ & \cellcolor{scoreworse}$0.5934\!\pm\!0.0464$ & \cellcolor{scorebetter}$\mathbf{0.6248}\!\pm\!0.0339$ & \cellcolor{scoreworse}$0.5852\!\pm\!0.0608$ & \cellcolor{scorebetter}$0.6161\!\pm\!0.0505$\\
Private-CRC / $F_1$ & $\mathbf{0.8391}\!\pm\!0.0337$ & \cellcolor{scoreworse}$0.8159\!\pm\!0.0438$ & \cellcolor{scoreworse}$0.8063\!\pm\!0.0511$ & \cellcolor{scoreworse}$0.8370\!\pm\!0.0227$ & \cellcolor{scoreworse}$0.8346\!\pm\!0.0326$\\
PANDA / $\kappa$ & $\mathbf{0.8517}\!\pm\!0.0221$ & \cellcolor{scoreworse}$0.8251\!\pm\!0.0155$ & \cellcolor{scoreworse}$0.8496\!\pm\!0.0122$ & \cellcolor{scoreworse}$0.8445\!\pm\!0.0122$ & \cellcolor{scoreworse}$0.8471\!\pm\!0.0184$\\
BRACS / $F_1$ & $0.3809\!\pm\!0.1116$ & \cellcolor{scorebetter}$0.3987\!\pm\!0.0517$ & \cellcolor{scorebetter}$\mathbf{0.4323}\!\pm\!0.1115$ & \cellcolor{scorebetter}$0.4268\!\pm\!0.0741$ & \cellcolor{scorebetter}$0.4118\!\pm\!0.0515$\\
\midrule
KIRC / C-index & $0.6984\!\pm\!0.0353$ & \cellcolor{scorebetter}$0.7095\!\pm\!0.0278$ & \cellcolor{scorebetter}$0.7106\!\pm\!0.0286$ & \cellcolor{scorebetter}$\mathbf{0.7248}\!\pm\!0.0329$ & \cellcolor{scorebetter}$0.7187\!\pm\!0.0369$\\
KIRP / C-index & $\mathbf{0.8363}\!\pm\!0.0911$ & \cellcolor{scoreworse}$0.8012\!\pm\!0.0756$ & \cellcolor{scoreworse}$0.8250\!\pm\!0.1022$ & \cellcolor{scoreworse}$0.8007\!\pm\!0.0782$ & \cellcolor{scoreworse}$0.7979\!\pm\!0.0849$\\
LUAD / C-index & $0.5453\!\pm\!0.0578$ & \cellcolor{scoreworse}$0.5343\!\pm\!0.0668$ & \cellcolor{scoreworse}$0.5262\!\pm\!0.0546$ & \cellcolor{scorebetter}$\mathbf{0.5602}\!\pm\!0.0526$ & $0.5453\!\pm\!0.0686$\\
STAD / C-index & $0.5495\!\pm\!0.0290$ & \cellcolor{scoreworse}$0.5487\!\pm\!0.0463$ & \cellcolor{scorebetter}$0.5617\!\pm\!0.0502$ & \cellcolor{scorebetter}$\mathbf{0.5802}\!\pm\!0.0497$ & \cellcolor{scorebetter}$0.5768\!\pm\!0.0381$\\
UCEC / C-index & $0.6693\!\pm\!0.0576$ & \cellcolor{scoreworse}$0.6446\!\pm\!0.1337$ & \cellcolor{scoreworse}$0.6664\!\pm\!0.1133$ & \cellcolor{scoreworse}$0.6424\!\pm\!0.0568$ & \cellcolor{scorebetter}$\mathbf{0.6743}\!\pm\!0.0797$\\
\midrule
\multicolumn{6}{c}{Random final-layer initialization before distillation}\\
CAM16 / AUC & $0.9669\!\pm\!0.0183$ & \cellcolor{scoreworse}$0.9455\!\pm\!0.0336$ & \cellcolor{scoreworse}$0.9611\!\pm\!0.0285$ & \cellcolor{scoreworse}$0.9563\!\pm\!0.0160$ & \cellcolor{scorebetter}$\mathbf{0.9693}\!\pm\!0.0234$\\
CAM17 / $F_1$ & $0.6080\!\pm\!0.0327$ & \cellcolor{scoreworse}$0.6068\!\pm\!0.0334$ & \cellcolor{scorebetter}$\mathbf{0.6384}\!\pm\!0.0582$ & \cellcolor{scorebetter}$0.6164\!\pm\!0.0298$ & \cellcolor{scorebetter}$0.6105\!\pm\!0.0345$\\
Private-CRC / $F_1$ & $0.8391\!\pm\!0.0337$ & \cellcolor{scoreworse}$0.8147\!\pm\!0.0467$ & \cellcolor{scoreworse}$0.8339\!\pm\!0.0330$ & \cellcolor{scorebetter}$\mathbf{0.8494}\!\pm\!0.0298$ & \cellcolor{scoreworse}$0.8376\!\pm\!0.0404$\\
PANDA / $\kappa$ & $0.8517\!\pm\!0.0221$ & \cellcolor{scoreworse}$0.8208\!\pm\!0.0118$ & \cellcolor{scorebetter}$\mathbf{0.8521}\!\pm\!0.0109$ & \cellcolor{scoreworse}$0.8479\!\pm\!0.0169$ & \cellcolor{scorebetter}$0.8518\!\pm\!0.0100$\\
BRACS / $F_1$ & $0.3809\!\pm\!0.1116$ & \cellcolor{scorebetter}$0.4313\!\pm\!0.0258$ & \cellcolor{scorebetter}$0.4299\!\pm\!0.0373$ & \cellcolor{scorebetter}$\mathbf{0.4514}\!\pm\!0.0665$ & \cellcolor{scorebetter}$0.4077\!\pm\!0.0483$\\
\midrule
KIRC / C-index & $0.6984\!\pm\!0.0353$ & \cellcolor{scoreworse}$0.6953\!\pm\!0.0591$ & \cellcolor{scorebetter}$0.7076\!\pm\!0.0471$ & \cellcolor{scorebetter}$0.7027\!\pm\!0.0484$ & \cellcolor{scorebetter}$\mathbf{0.7177}\!\pm\!0.0470$\\
KIRP / C-index & $\mathbf{0.8363}\!\pm\!0.0911$ & \cellcolor{scoreworse}$0.7936\!\pm\!0.0832$ & \cellcolor{scoreworse}$0.7790\!\pm\!0.0789$ & \cellcolor{scoreworse}$0.8044\!\pm\!0.0917$ & \cellcolor{scoreworse}$0.7869\!\pm\!0.0934$\\
LUAD / C-index & $0.5453\!\pm\!0.0578$ & \cellcolor{scorebetter}$0.5644\!\pm\!0.0759$ & \cellcolor{scorebetter}$\mathbf{0.5663}\!\pm\!0.0585$ & \cellcolor{scoreworse}$0.5436\!\pm\!0.0525$ & \cellcolor{scoreworse}$0.5234\!\pm\!0.0400$\\
STAD / C-index & $0.5495\!\pm\!0.0290$ & \cellcolor{scoreworse}$0.5389\!\pm\!0.0356$ & \cellcolor{scorebetter}$\mathbf{0.5651}\!\pm\!0.0396$ & \cellcolor{scorebetter}$0.5524\!\pm\!0.0721$ & \cellcolor{scorebetter}$0.5629\!\pm\!0.0669$\\
UCEC / C-index & $0.6693\!\pm\!0.0576$ & \cellcolor{scorebetter}$0.6764\!\pm\!0.1530$ & \cellcolor{scoreworse}$0.6600\!\pm\!0.0783$ & \cellcolor{scorebetter}$\mathbf{0.7051}\!\pm\!0.1016$ & \cellcolor{scorebetter}$0.6726\!\pm\!0.0675$\\
\bottomrule\end{tabular}
\end{table}

\begin{table}[htbp]
\centering\footnotesize
\setlength{\tabcolsep}{1.5pt}
\caption{Task NLL with the same trained prediction and either a zero or retained native block. All inputs have 3,072 coordinates. The matched native \five{} reference $[x_i;0_d]$ is reused across panels. Lower is better.}
\label{tab:native-retention-absolute-nll}
\label{tab:native-retention-nll}
\begin{tabular}{@{}lccccc@{}}\toprule
Dataset & \shortstack{Native \five{}\\$[x_i;0_d]$} & \shortstack{Direct\\$[0_d;q_i^D]$} & \shortstack{Direct + native\\$[x_i;q_i^D]$} & \shortstack{Residual\\$[0_d;q_i^R]$} & \shortstack{Residual + native\\$[x_i;q_i^R]$}\\\midrule
\multicolumn{6}{c}{Zero final-layer initialization before distillation}\\
CAM16 & $\mathbf{0.2610}\!\pm\!0.0953$ & \cellcolor{scoreworse}$0.3580\!\pm\!0.1922$ & \cellcolor{scoreworse}$0.2924\!\pm\!0.1396$ & \cellcolor{scoreworse}$0.3253\!\pm\!0.1347$ & \cellcolor{scoreworse}$0.3048\!\pm\!0.1512$\\
CAM17 & $0.6529\!\pm\!0.2769$ & \cellcolor{scoreworse}$0.9189\!\pm\!0.3316$ & \cellcolor{scorebetter}$\mathbf{0.6241}\!\pm\!0.3428$ & \cellcolor{scoreworse}$0.7350\!\pm\!0.4200$ & \cellcolor{scoreworse}$0.6791\!\pm\!0.2495$\\
Private-CRC & $0.6192\!\pm\!0.0378$ & \cellcolor{scorebetter}$0.5823\!\pm\!0.1906$ & \cellcolor{scorebetter}$0.5690\!\pm\!0.0727$ & \cellcolor{scorebetter}$0.5371\!\pm\!0.0869$ & \cellcolor{scorebetter}$\mathbf{0.5297}\!\pm\!0.1282$\\
PANDA & $1.2043\!\pm\!0.4426$ & \cellcolor{scorebetter}$1.0014\!\pm\!0.0692$ & \cellcolor{scorebetter}$\mathbf{0.9764}\!\pm\!0.0898$ & \cellcolor{scoreworse}$1.5970\!\pm\!0.7405$ & \cellcolor{scorebetter}$1.1580\!\pm\!0.2178$\\
BRACS & $1.5287\!\pm\!0.6716$ & \cellcolor{scoreworse}$1.5428\!\pm\!0.3129$ & \cellcolor{scorebetter}$1.4855\!\pm\!0.5011$ & \cellcolor{scoreworse}$1.5966\!\pm\!0.4032$ & \cellcolor{scorebetter}$\mathbf{1.4386}\!\pm\!0.3485$\\
\midrule
KIRC & $1.8947\!\pm\!0.5592$ & \cellcolor{scorebetter}$1.7613\!\pm\!0.5343$ & \cellcolor{scorebetter}$\mathbf{1.6468}\!\pm\!0.2609$ & \cellcolor{scorebetter}$1.7733\!\pm\!0.2047$ & \cellcolor{scorebetter}$1.8066\!\pm\!0.5783$\\
KIRP & $1.1656\!\pm\!0.4045$ & \cellcolor{scorebetter}$\mathbf{1.1319}\!\pm\!0.1381$ & \cellcolor{scoreworse}$1.2582\!\pm\!0.2996$ & \cellcolor{scoreworse}$1.2343\!\pm\!0.2811$ & \cellcolor{scoreworse}$1.1800\!\pm\!0.2694$\\
LUAD & $\mathbf{2.1217}\!\pm\!0.1648$ & \cellcolor{scoreworse}$2.6486\!\pm\!0.7593$ & \cellcolor{scoreworse}$2.3585\!\pm\!0.7398$ & \cellcolor{scoreworse}$2.7231\!\pm\!0.7288$ & \cellcolor{scoreworse}$2.4006\!\pm\!0.4799$\\
STAD & $2.6422\!\pm\!0.7560$ & \cellcolor{scorebetter}$\mathbf{2.2103}\!\pm\!0.3926$ & \cellcolor{scorebetter}$2.6048\!\pm\!0.7348$ & \cellcolor{scoreworse}$2.7300\!\pm\!0.6291$ & \cellcolor{scorebetter}$2.5068\!\pm\!0.4820$\\
UCEC & $0.9924\!\pm\!0.1125$ & \cellcolor{scorebetter}$0.9382\!\pm\!0.0483$ & \cellcolor{scorebetter}$0.9375\!\pm\!0.0790$ & \cellcolor{scorebetter}$\mathbf{0.9268}\!\pm\!0.0816$ & \cellcolor{scorebetter}$0.9309\!\pm\!0.0756$\\
\midrule
\multicolumn{6}{c}{Random final-layer initialization before distillation}\\
CAM16 & $0.2610\!\pm\!0.0953$ & \cellcolor{scoreworse}$0.3762\!\pm\!0.1650$ & \cellcolor{scoreworse}$0.2669\!\pm\!0.1596$ & \cellcolor{scoreworse}$0.3478\!\pm\!0.1310$ & \cellcolor{scorebetter}$\mathbf{0.2221}\!\pm\!0.1084$\\
CAM17 & $\mathbf{0.6529}\!\pm\!0.2769$ & \cellcolor{scoreworse}$0.8497\!\pm\!0.2992$ & \cellcolor{scoreworse}$0.7054\!\pm\!0.3072$ & \cellcolor{scoreworse}$0.7946\!\pm\!0.3287$ & \cellcolor{scoreworse}$0.7055\!\pm\!0.3602$\\
Private-CRC & $0.6192\!\pm\!0.0378$ & \cellcolor{scorebetter}$0.5807\!\pm\!0.0714$ & \cellcolor{scorebetter}$0.5703\!\pm\!0.0761$ & \cellcolor{scorebetter}$0.5354\!\pm\!0.0870$ & \cellcolor{scorebetter}$\mathbf{0.5078}\!\pm\!0.0814$\\
PANDA & $1.2043\!\pm\!0.4426$ & \cellcolor{scorebetter}$1.1324\!\pm\!0.1417$ & \cellcolor{scorebetter}$1.0520\!\pm\!0.2602$ & \cellcolor{scorebetter}$1.1098\!\pm\!0.4763$ & \cellcolor{scorebetter}$\mathbf{1.0221}\!\pm\!0.3193$\\
BRACS & $\mathbf{1.5287}\!\pm\!0.6716$ & \cellcolor{scoreworse}$1.5451\!\pm\!0.4189$ & \cellcolor{scoreworse}$1.7362\!\pm\!0.5892$ & \cellcolor{scoreworse}$1.8601\!\pm\!0.5966$ & \cellcolor{scoreworse}$1.6750\!\pm\!0.4275$\\
\midrule
KIRC & $1.8947\!\pm\!0.5592$ & \cellcolor{scorebetter}$1.8774\!\pm\!0.4209$ & \cellcolor{scoreworse}$1.9181\!\pm\!0.4312$ & \cellcolor{scoreworse}$1.9248\!\pm\!0.4497$ & \cellcolor{scorebetter}$\mathbf{1.8219}\!\pm\!0.1382$\\
KIRP & $1.1656\!\pm\!0.4045$ & \cellcolor{scoreworse}$1.1808\!\pm\!0.2036$ & \cellcolor{scorebetter}$1.1251\!\pm\!0.1656$ & \cellcolor{scorebetter}$\mathbf{1.0906}\!\pm\!0.1580$ & \cellcolor{scorebetter}$1.1478\!\pm\!0.2383$\\
LUAD & $2.1217\!\pm\!0.1648$ & \cellcolor{scorebetter}$\mathbf{2.1058}\!\pm\!0.1759$ & \cellcolor{scoreworse}$2.3550\!\pm\!0.6127$ & \cellcolor{scoreworse}$2.4996\!\pm\!0.5208$ & \cellcolor{scoreworse}$2.4547\!\pm\!0.6061$\\
STAD & $2.6422\!\pm\!0.7560$ & \cellcolor{scorebetter}$2.4557\!\pm\!0.3826$ & \cellcolor{scorebetter}$\mathbf{2.3816}\!\pm\!0.6888$ & \cellcolor{scoreworse}$2.8569\!\pm\!0.4871$ & \cellcolor{scorebetter}$2.4986\!\pm\!0.3642$\\
UCEC & $0.9924\!\pm\!0.1125$ & \cellcolor{scoreworse}$1.0223\!\pm\!0.2452$ & \cellcolor{scorebetter}$0.9615\!\pm\!0.0645$ & \cellcolor{scorebetter}$0.9622\!\pm\!0.0719$ & \cellcolor{scorebetter}$\mathbf{0.9363}\!\pm\!0.0998$\\
\bottomrule\end{tabular}
\end{table}

\paragraph{A coordinate change that preserves information.}
An invertible coordinate change provides a separate check of information preservation. Starting from $x$ and residual correction $r$, we form $u=(x+r)/\sqrt2$ and $v=(x-r)/\sqrt2$, with inverse $x=(u+v)/\sqrt2$ and $r=(u-v)/\sqrt2$. Rewriting a saved task model's first affine layer expresses the same function in these coordinates without retraining. All fifty dataset and realization replays agree within the recorded tolerance of $.001$, with maximum absolute logit difference $1.91\times10^{-5}$. They cover 14,685 slide appearances, including repeated slides across realizations. This validates the conversion for saved functions and illustrates complete recovery in the theorem. It does not equate fresh optimization outcomes. Both coordinates remain available in this replay.

\subsection{Individual features or their mean as auxiliary supervision}
\label{app:encoder-target-law}
The final study tests transfer through the image encoder itself. We adapt UNI2-h by distillation on the training images, freeze it and train downstream models on its features. The references are the published encoder, an encoder trained by mean KD and that model's mean-prediction output. Two further encoders combine mean KD with a flow objective whose target is either the regional mean or a sampled individual teacher feature. Their comparison changes the auxiliary target while retaining mean supervision and the same downstream readout.

\paragraph{Architecture and objectives.}
Let $I_i$ be the raw \five{} patch and $c_i=E_\theta(I_i)\in\mathbb R^{1536}$ its adapted encoder feature. UNI2-h uses rank-8 LoRA updates on query and value projections in all 24 blocks, with pretrained weights and key projections frozen \citep{UNI,hu2022lora}. A GELU head $r_\psi$ with widths $1536\to1536\to1536$ predicts $m_i=K_i^{-1}\sum_{j=1}^{K_i}h_{ij}$ from $c_i$. The $1\le K_i\le16$ teacher vectors are the observed aligned \twenty{} patches assigned to the region. The distillation data do not require all sixteen positions. With coordinate scales $s_{m,d}$ estimated from the distillation training targets, the common mean loss is
\begin{equation}
\mathcal L_{\rm mean}
=\mathbb E_i\left[\frac1{1536}\sum_{d=1}^{1536}
\operatorname{SmoothL1}_{1}\left(\frac{r_\psi(c_i)_d-m_{i,d}}{s_{m,d}}\right)
+.1\left(1-\cos\bigl(r_\psi(c_i),m_i\bigr)\right)\right].
\label{eq:encoder-mean-kd}
\end{equation}
SmoothL1 uses unit threshold, with value $e^2/2$ for $|e|<1$ and $|e|-1/2$ otherwise. Cosine similarity uses tolerance $10^{-8}$. This objective accompanies encoder adaptation and differs from frozen-feature MSE prediction, so results across the two studies cannot isolate the loss function.

The auxiliary flow target is $a_i=m_i$ for mean supervision or $a_i=h_{iJ}$ for individual supervision, with $J$ uniform over the $K_i$ observed patches. Both use the centre $\mu$ and scale $s_h$ estimated from distillation training data to form $A_i=(a_i-\mu)\oslash s_h$, where $\oslash$ is coordinatewise division. With Gaussian noise $\epsilon\sim\mathcal N(0,I)$ and time $s\sim\operatorname{Unif}[0,1)$, the field $v_\omega$ learns the velocity along a straight path \citep{lipman2023flow},
\begin{align}
z_s&=(1-s)\epsilon+sA_i,\\
\mathcal L_{\rm flow}&=\mathbb E_{i,s,\epsilon,J}\left[\frac1{1536}
\|v_\omega(z_s,s,c_i)-(A_i-\epsilon)\|_2^2\right],\\
\mathcal L&=\mathcal L_{\rm mean}+.1\mathcal L_{\rm flow}.
\label{eq:encoder-flow-kd}
\end{align}
Only individual-feature supervision uses the draw $J$. Neither its index nor its spatial position enters the encoder or velocity field. Teacher features, path state and target velocity are detached, while the conditioning feature $c_i$ carries gradients to the encoder. The velocity field has three modulated residual MLP blocks of width 512 and a 128-dimensional sinusoidal time encoding. Changing the sampled target also changes the path input $z_s$. The shared-output premise of Lemma~\ref{lem:shared-output-mean} therefore does not hold, so replacing individual targets by their mean can change this auxiliary objective.

For both flow-trained encoders, downstream models receive only $E_\theta(I_i)$. We discard the mean head and velocity field, solve no differential equation and generate no teacher features at deployment. The separate head readout uses $r_\psi(E_\theta(I_i))$ from the mean-KD model. All task tokens have 1,536 coordinates. In Theorem~\ref{thm:transfer-risk}, $B$ now denotes raw images, which the adapted features need not recover. Both representation loss and model excess can consequently change, even though the auxiliary-target comparison holds task width constant.

\paragraph{Distillation populations and optimization.}
Three distillation partitions provide excluded-dataset evaluation across all ten datasets. The first trains on Private-CRC, PANDA, KIRC, KIRP, LUAD and STAD, then evaluates CAM16, CAM17, BRACS and UCEC. Partition A trains on CAM16, CAM17, BRACS, UCEC, KIRP and LUAD, then evaluates Private-CRC, PANDA and KIRC. Partition B trains on CAM16, CAM17, BRACS, UCEC, Private-CRC and PANDA, then evaluates KIRP, LUAD and STAD. Each partition shares its distillation sample across the three objectives and five realizations. Evaluation datasets enter neither distillation training nor validation. These experiments use six distillation datasets, while the frozen-feature studies use nine.

Each distillation dataset contributes 1,024 fitting and 256 selection occurrences. Sampling first chooses a group, then a slide and an eligible region, retaining repeated occurrences in the objective. The six datasets contribute equal mass. Means and scales are estimated from fitting occurrences only. Individual-feature moments average within each region before averaging across occurrences, so regions with more observed teacher patches receive no extra weight. The mean loss uses coordinate SD of regional targets. The flow loss uses the common centre and coordinate SD of individual features, balanced by region. Both scales are floored at $10^{-4}$.

Regions used for distillation can contain incomplete teacher grids. The first partition has 4,142 complete grids among 6,144 fitting occurrences, with two to fifteen observed patches in the remainder. Partitions A and B have 5,011 and 4,085 complete grids and include some regions with one patch. Uniformly sampling an observed teacher vector therefore weights the available patches within each region, rather than every physical position of a complete field.

Training uses 2,000 updates, effective batch size 96 and microbatches of 16. Each update includes sixteen fitting occurrences from each distillation dataset. AdamW uses learning rates $10^{-5}$ for encoder updates and $10^{-4}$ for the prediction head and field, weight decay .05 on decay-eligible parameters, ten percent warmup, cosine decay and gradient clipping at 1. BF16 forward passes retain FP32 parameters and losses, with TF32 disabled. Paired distillation realizations share their initial encoder, mean head, velocity field, occurrence schedule, noise and time streams across the models with flow supervision. Teacher indices are sampled independently of the noise and time draws.

Distillation validation compares checkpoints from updates 200 through 2,000 using each model's combined objective. Flow evaluation uses four fixed noise and time draws per occurrence, with one time from each quarter of the interval. The initial model is recorded but cannot be selected. All thirty flow fits choose update 2,000, holding selected training duration constant in the target comparison. Three objectives across three partitions and five realizations require 45 distillation fits.

\paragraph{Task comparison.}
All five task representations use the same selection of at most 128 raw \five{} patches per slide, with shared splits, task initialization and training within a realization. CAM16, CAM17, Private-CRC and BRACS train for fifteen epochs, and PANDA and survival for twenty. Coupled distillation and task seeds give 250 fits across five realizations. Test membership changes across realizations but matches across models within each. Because distillation partitions differ across evaluation datasets, variation in effects cannot be assigned solely to task category.

Table~\ref{tab:encoder-target-absolute} compares all five readouts, including the mean-KD head. The frozen reference changes the interpretation of adapted models. Individual-feature flow has a higher mean than regional-mean flow on KIRP, yet both remain below the frozen encoder. Both flow encoders exceed the frozen reference on LUAD and STAD. The prediction head can also change the ordering relative to its own encoder, so auxiliary target choice and downstream readout require separate comparisons.

\begin{table}[htbp]
\centering\footnotesize
\setlength{\tabcolsep}{1.5pt}
\caption{Task scores from frozen UNI2-h, the mean-KD encoder and its prediction head, and encoders trained with mean KD plus flow supervision. Every readout has 1,536 coordinates on the same sampled \five{} patches. Published frozen UNI2-h is the reference. Higher is better.}
\label{tab:encoder-target-absolute}
\begin{tabular}{@{}lccccc@{}}\toprule
Dataset / metric & \shortstack{Frozen\\UNI2-h} & \shortstack{Mean KD\\Encoder} & \shortstack{Mean KD\\Mean head} & \shortstack{Mean KD + flow\\Regional mean} & \shortstack{Mean KD + flow\\Individual feature}\\\midrule
CAM16 / AUC & $0.8808\!\pm\!0.0568$ & \cellcolor{scoreworse}$0.8734\!\pm\!0.0787$ & \cellcolor{scorebetter}$\mathbf{0.8982}\!\pm\!0.0548$ & \cellcolor{scorebetter}$0.8827\!\pm\!0.0775$ & \cellcolor{scoreworse}$0.8749\!\pm\!0.0768$\\
CAM17 / $F_1$ & $0.5025\!\pm\!0.0435$ & \cellcolor{scorebetter}$0.5115\!\pm\!0.0539$ & \cellcolor{scorebetter}$\mathbf{0.5403}\!\pm\!0.0355$ & \cellcolor{scorebetter}$0.5077\!\pm\!0.0423$ & \cellcolor{scorebetter}$0.5211\!\pm\!0.0606$\\
Private-CRC / $F_1$ & $0.8411\!\pm\!0.0262$ & \cellcolor{scoreworse}$0.8353\!\pm\!0.0509$ & \cellcolor{scoreworse}$0.8353\!\pm\!0.0408$ & \cellcolor{scoreworse}$0.8225\!\pm\!0.0434$ & \cellcolor{scorebetter}$\mathbf{0.8492}\!\pm\!0.0282$\\
PANDA / $\kappa$ & $0.8499\!\pm\!0.0155$ & \cellcolor{scorebetter}$\mathbf{0.8519}\!\pm\!0.0131$ & \cellcolor{scoreworse}$0.8458\!\pm\!0.0117$ & \cellcolor{scoreworse}$0.8496\!\pm\!0.0122$ & \cellcolor{scoreworse}$0.8495\!\pm\!0.0067$\\
BRACS / $F_1$ & $0.4589\!\pm\!0.0541$ & \cellcolor{scoreworse}$0.4331\!\pm\!0.0477$ & \cellcolor{scorebetter}$\mathbf{0.4608}\!\pm\!0.0606$ & \cellcolor{scoreworse}$0.4204\!\pm\!0.0339$ & \cellcolor{scoreworse}$0.4347\!\pm\!0.0363$\\
\midrule
KIRC / C-index & $0.7321\!\pm\!0.0407$ & \cellcolor{scoreworse}$0.7206\!\pm\!0.0176$ & \cellcolor{scoreworse}$0.7186\!\pm\!0.0451$ & \cellcolor{scorebetter}$\mathbf{0.7378}\!\pm\!0.0176$ & \cellcolor{scoreworse}$0.7069\!\pm\!0.0332$\\
KIRP / C-index & $\mathbf{0.7921}\!\pm\!0.0721$ & \cellcolor{scoreworse}$0.7519\!\pm\!0.0963$ & \cellcolor{scoreworse}$0.7504\!\pm\!0.1031$ & \cellcolor{scoreworse}$0.6861\!\pm\!0.1161$ & \cellcolor{scoreworse}$0.7744\!\pm\!0.1057$\\
LUAD / C-index & $0.5307\!\pm\!0.0481$ & \cellcolor{scorebetter}$0.5674\!\pm\!0.0424$ & \cellcolor{scorebetter}$0.5453\!\pm\!0.0380$ & \cellcolor{scorebetter}$0.5537\!\pm\!0.0268$ & \cellcolor{scorebetter}$\mathbf{0.5783}\!\pm\!0.0396$\\
STAD / C-index & $0.5532\!\pm\!0.0235$ & \cellcolor{scorebetter}$0.5944\!\pm\!0.0312$ & \cellcolor{scorebetter}$0.5628\!\pm\!0.0607$ & \cellcolor{scorebetter}$0.5912\!\pm\!0.0316$ & \cellcolor{scorebetter}$\mathbf{0.5953}\!\pm\!0.0331$\\
UCEC / C-index & $\mathbf{0.6561}\!\pm\!0.0591$ & \cellcolor{scoreworse}$0.5980\!\pm\!0.0711$ & \cellcolor{scoreworse}$0.6348\!\pm\!0.1202$ & \cellcolor{scoreworse}$0.5924\!\pm\!0.0718$ & \cellcolor{scoreworse}$0.5976\!\pm\!0.0671$\\
\bottomrule\end{tabular}
\end{table}

\begin{table}[htbp]
\centering\footnotesize
\setlength{\tabcolsep}{2.5pt}
\caption{Validation task scores at the selected downstream checkpoints. Both encoders retain mean KD. Regional-mean flow is the reference for individual-feature flow. Higher is better.}
\label{tab:encoder-target-details}
\begin{tabular}{@{}lcc@{}}\toprule
Dataset / metric & Flow toward regional mean & Flow toward individual feature\\\midrule
CAM16 / AUC & $\mathbf{0.8598}\!\pm\!0.0900$ & \cellcolor{scoreworse}$0.8538\!\pm\!0.0921$\\
CAM17 / $F_1$ & $0.5468\!\pm\!0.0535$ & \cellcolor{scorebetter}$\mathbf{0.5549}\!\pm\!0.1129$\\
Private-CRC / $F_1$ & $\mathbf{0.8889}\!\pm\!0.0303$ & \cellcolor{scoreworse}$0.8800\!\pm\!0.0367$\\
PANDA / $\kappa$ & $\mathbf{0.8709}\!\pm\!0.0114$ & \cellcolor{scoreworse}$0.8691\!\pm\!0.0136$\\
BRACS / $F_1$ & $0.4201\!\pm\!0.0934$ & \cellcolor{scorebetter}$\mathbf{0.4333}\!\pm\!0.0992$\\
\midrule
KIRC / C-index & $0.6957\!\pm\!0.0499$ & \cellcolor{scorebetter}$\mathbf{0.6965}\!\pm\!0.0566$\\
KIRP / C-index & $\mathbf{0.8618}\!\pm\!0.0556$ & \cellcolor{scoreworse}$0.8587\!\pm\!0.0366$\\
LUAD / C-index & $\mathbf{0.6386}\!\pm\!0.0513$ & \cellcolor{scoreworse}$0.6357\!\pm\!0.0734$\\
STAD / C-index & $\mathbf{0.6428}\!\pm\!0.0863$ & \cellcolor{scoreworse}$0.6390\!\pm\!0.0849$\\
UCEC / C-index & $\mathbf{0.8026}\!\pm\!0.0777$ & $\mathbf{0.8026}\!\pm\!0.0777$\\
\bottomrule\end{tabular}
\end{table}

The preferred flow target varies across datasets. Individual-feature flow has higher test macro $F_1$ than regional-mean flow on CAM17, Private-CRC and BRACS, while CAM16 and PANDA favor regional-mean flow. Survival cohorts show both directions. Private-CRC favors individual-feature flow in all five test realizations, whereas KIRC favors regional-mean flow in all five. Their preferred targets reverse on average validation scores in Table~\ref{tab:encoder-target-details}. KIRP's positive test average for individual-feature flow combines two gains, two losses and one tie. These patterns support sensitivity to the auxiliary target without a rule for choosing it by task category or establishing recovery of the conditional distribution of teacher patches.

\end{document}